\documentclass[twoside,11pt]{article}

\usepackage{jmlr2e}
\usepackage{epsfig}
\usepackage{graphicx}
\usepackage{algorithm}
\usepackage{algorithmic,amssymb}
\usepackage{amsmath}
\allowdisplaybreaks
\usepackage{natbib}
\usepackage{color}
\usepackage{times}
\usepackage[hyphens]{url}
\usepackage{soul}
\usepackage{url}
\usepackage{caption}
\usepackage{multirow}
\usepackage{makecell}
\usepackage[T1]{fontenc}
\usepackage[utf8]{inputenc}
\usepackage{placeins}
\usepackage{amsfonts,amstext,mathrsfs}
\usepackage{subcaption}

\usepackage{booktabs}       
\usepackage{nicefrac}       
\usepackage{microtype}      
\usepackage{multirow}
\usepackage{makecell}

\newtheorem{innercustomgeneric}{\customgenericname}
\providecommand{\customgenericname}{}

\usepackage{float}
\allowdisplaybreaks

\newcommand{\ie}{\emph{i.e.}}
\newcommand{\bO}{\mathcal{O}}
\newcommand{\bG}{\mathcal{G}}
\newcommand{\be}{\begin{equation}}
	\newcommand{\ee}{\end{equation}}
\newcommand{\I}{\mathcal{I}}

\usepackage{geometry}
\begin{document}

\title{AsySQN: Faster Vertical Federated Learning Algorithms with Better Computation Resource Utilization}

\author{\name Qingsong Zhang \email qszhang1995@gmail.com      \\
\addr Xidian University, Xi’an, China, and also with JD Tech.
\\
\name Bin Gu \email {jsgubin@gmail.com} \\
\addr MBZUAI and also with JD Finance America Corporation
\\
\name Cheng Deng	\email {chdeng.xd@gmail.com} \\
 \addr School of Electronic Engineering, Xidian University, Xi'an, China
 \\
 \name Songxiang Gu	\email {songxiang.gu@jd.com} \\
 \addr JD Tech, Beijing, China
 \\
 \name Liefeng Bo	\email {boliefeng@jd.com} \\
 \addr JD Finance America Corporation, USA
 \\
 \name Jian Pei \email{jian\_pei@sfu.ca} \\
 \addr Simon Fraser University, Canada
 \\
 \name Heng Huang \email {heng.huang@pitt.edu} \\
\addr JD Finance America Corporation and also with University of Pittsburgh, USA\\
}
\editor{}

\maketitle

\setcounter{equation}{11}

	\begin{abstract}
		Vertical federated learning (VFL) is an effective paradigm of training the emerging cross-organizational (\emph{e.g.}, different corporations, companies and organizations) collaborative learning with privacy preserving.
		Stochastic gradient descent (SGD) methods are the popular choices for training VFL models because of the low per-iteration computation.
		However, existing SGD-based VFL algorithms are communication-expensive due to a large number of communication rounds.
		Meanwhile, most existing VFL algorithms use synchronous computation which seriously hamper the computation resource utilization  in real-world applications. To address the challenges of communication and computation resource utilization, we propose an asynchronous stochastic quasi-Newton (AsySQN) framework for VFL, under which three algorithms, {\ie} AsySQN-SGD, -SVRG and -SAGA, are proposed. The proposed AsySQN-type algorithms making descent steps scaled by approximate (without calculating the inverse Hessian matrix explicitly) Hessian information convergence much faster than SGD-based methods in practice and thus can dramatically reduce the number of communication rounds. Moreover, the adopted asynchronous computation can make better use of the computation resource.
		We theoretically prove the convergence rates of our proposed algorithms for strongly convex problems. Extensive numerical experiments on real-word datasets demonstrate the lower communication costs and  better computation resource utilization of our algorithms compared with state-of-the-art VFL algorithms.
	\end{abstract}

	\section{Introduction}
	Federated learning attracts much attention from both  academic and industry \cite{mcmahan2017communication,yang2020federated,kairouz2019advances,gong2016private,zhang2021secure} because it meets the emerging demands of collaboratively-modeling with privacy-preserving. Currently, existing federated learning frameworks can be categorized into two main classes, \ie, horizon federated  learning (HFL) and vertical federated learning (VFL). In HFL, samples sharing the same features are distributed over different parties. While, as for VFL, data owned by different parties have the same sample IDs but disjoint subsets of features. Such scenario is common in the industry applications of collaborative learning, such as medical study, financial risk, and  targeted marketing \cite{zhang2021secure,gong2016private,yang2019federated,cheng2019secureboost,liu2019federated,hu2019fdml}. For example, E-commerce companies owning the online shopping information could collaboratively train joint-models with banks and digital finance companies that own other information of the same people such as the average monthly deposit and online consumption, respectively, to achieve a precise customer profiling.
	In this paper, we focus on VFL.
	
	In the real VFL system, different parties always represent different companies or organizations across different networks, or even in a wireless environment with limited bandwidth \cite{mcmahan2017communication}. The frequent communications with large per-round communication overhead (PRCO) between different parties are thus much expensive, making the communication expense being one of the main bottlenecks for efficiently training VFL models. On the other hand, for VFL applications in industry, different parties always own unbalanced computation resources (CR). For example, it is common that large corporations and small companies collaboratively optimize the joint-model, where the former have better CR while the later have the poorer. In this case, synchronous computation has a poor computation resource utilization (CRU). Because corporations owning better CR have to waste its CR to wait for the stragglers for synchronization, leading to another bottleneck  for efficiently  training VFL models. Thus, it is desired to develop algorithms with lower communication cost (CC) and better CRU to efficiently train VFL models in the real-world applications.
	
	Currently, there are extensive works focusing on VFL. Some works focus on designing different machine learning models for VFL such as linear regression \cite{gascon2016secure}, logistic regression \cite{hardy2017private} and tree model \cite{cheng2019secureboost}. Some works also aim at developing secure optimization algorithms for training VFL models, such as the SGD-based methods \cite{zhang2021secure,liu2019communication,gu2020Privacy}, which are popular due to the per-iteration computation efficiency but are communication-expensive due to the large number of communication rounds (NCR). Especially, there have been several works focusing on addressing the CC and CRU challenges of VFL. The quasi-Newton (QN) based framework \cite{yang2019quasi} is designed to reduce the number of communication rounds (NCR) of SGD-based methods, and the bilevel asynchronous VFL framework (VF$\text{B}^2$) \cite{zhang2021secure} and AFVP algorithms \cite{gu2020Privacy} are developed to achieve better CRU.
	
	However, in the QN-based framework \cite{yang2019quasi}, 1) the gradient differences are transmitted to globally compute the approximate Hessian information, which has expensive PRCO, 2) the synchronous computation is adopted. Thus, QN-based framework still has the large CC and dramatically sacrifices the CRU. Moreover, VF$\text{B}^2$ \cite{zhang2021secure} and AFVP algorithms \cite{gu2020Privacy} are communication-expensive owing to large NCR. Thus, it is still challenging to design VFL algorithms with lower CC and better CRU for real-world  scenarios.
	
	To address this challenge, we propose an asynchronous stochastic quasi-Newton based framework for VFL, \ie, AsySQN. Specifically, AsySQN-type algorithms significantly improve the practical convergence speed by utilizing approximate Hessian information to obtain a better descent direction, and thus reduce the NCR. Especially, the approximate Hessian information is locally computed and only scalars are necessary to be transmitted, which thus has low PRCO. Meanwhile, the AsySQN framework enables all parties update the model asynchronously, and thus keeps the CR being utilized all the time. Moreover,  we consider adopting the vanilla SGD and its variance reduction variants, \ie, SVRG and SAGA, as the stochastic gradient estimator,  due to their  promising performance in practice.
	We summarize the contributions of this paper as follows.
	\begin{itemize}
		\item
		We propose a novel asynchronous stochastic quasi-Newton  framework (AsySQN) for VFL, which has the lower CC and better CRU.
		\item
		Three AsySQN-type algorithms, including AsySQN-SGD and its variance reduction variants AsySQN-SVRG and -SAGA, are proposed under AsySQN. Moreover, we theoretically prove the convergence rates of these three algorithms for strongly convex problems.
	\end{itemize}
	\begin{figure}[!t]
		\centering
		\includegraphics[width=0.8\linewidth]{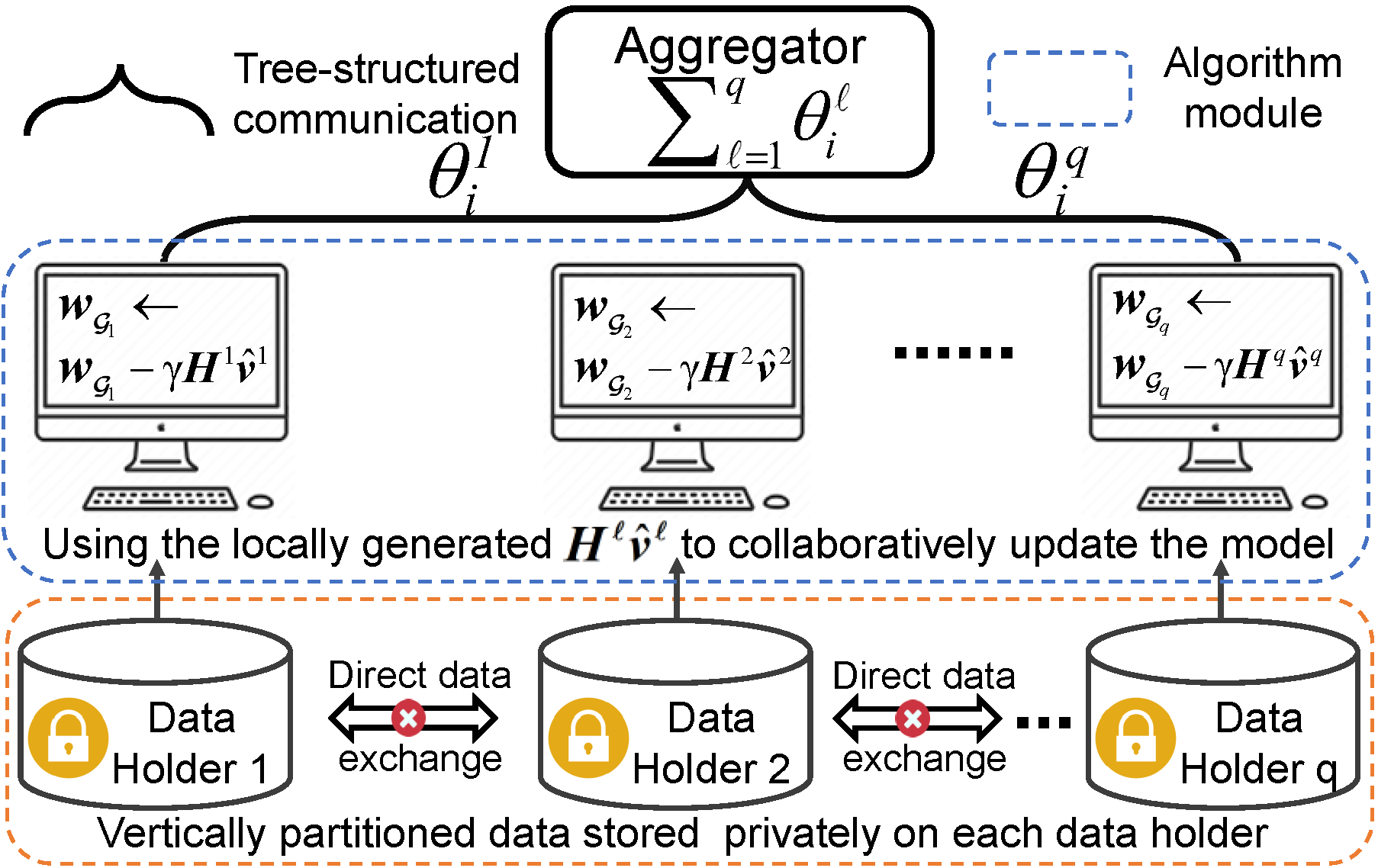}
		\caption{System structure of AsySQN framework.}
		\label{struc}
	\end{figure}
	\section{Methodology}
	In this section, we formulate the problem studied in this paper and propose the AsySQN framework for VFL, which has the lower CC and the better CRU when applied to the industry.
	\subsection{Problem Formulation}
	In this paper, we consider a VFL system with $q$ parties, where each party holds different features of the same samples. Given a training set $\{x_i,y_i\}_{i=1}^n$, where  $x_i \in \mathbb{R}^d $, $y_i \in \{-1, +1\}$ for binary classification task or $y_i \in \mathbb{R}$ for regression problem.  Then $x$  can be represented as a concatenation of all features, \ie, $x=[x_{1}, \cdots, x_{q}]$, where $x_{\ell} \in \mathbb{R}^{d_\ell}$ is stored privately on party $\ell$, and $\sum_{\ell=1}^{q}d_{\ell} = d$.  Similar to previous works \cite{gong2016private,hu2019fdml}, we assume the labels are held by all parties. In this paper, we consider the model in the form of $w^{\mathrm{T}} x$, where $w \in \mathbb{R}^d $ corresponds to the model parameters. Particularly, we focus on the regularized empirical risk minimization problem with following form
	\begin{align}\label{P}
		\min_{w\in \mathbb{R}^{d}} f(w) := \frac{1}{n} \sum_{i=1}^{n} \underbrace{\mathcal{L}\left(\theta_i, y_{i}\right)+\lambda g(w)}_{f_{i}(w)}, \tag{P}
	\end{align}
	where $\theta_i = \sum_{\ell=1}^{q}\theta_i^{\ell}=\sum_{\ell=1}^{q}w_{{\ell }}^{\mathrm{T}} {(x_i)}_{{\ell }}$, $\mathcal{L}$ denotes the loss function, $g(w)=\sum_{\ell=1}^{q} g_\ell(w_{\ell})$ is the regularization term,  $f_i: \mathbb{R}^d \to \mathbb{R}$ is strongly convex. Problem~\ref{P} capsules many machine learning problems such as  the widely used $\ell_2$-regularized logistic regression \cite{conroy2012fast}, least squares support vector machine \cite{suykens1999least} and  ridge regression \cite{shen2013novel}.

	\subsection{AsySQN Framework}
	To address the CC and CRU challenges for VFL application to industry, we propose the AsySQN framework shown in Fig.~\ref{struc}. In AsySQN, the data are vertically distributed over all parties and each party can not directly exchange the data due to privacy concerning. As shown in algorithm module, each party uses the descent direction $d^\ell=H^\ell v^\ell$ generated locally (refer to Algorithms \ref{SdLBFGS}) to update its model and the $\theta_i^\ell$ used for calculating the local $H^\ell$ are aggregated from the other parties through Algorithm \ref{safer_tree}.
	In the following, we present that it is not easy to design AsySQN.

	First, we consider using locally stored information, \ie, the block-coordinate gradient difference and local parameter difference, to calculate the approximate local Hessian information implicitly by SLBFGS \cite{zhang2020faster} to circumvent the large PRCO of transmitting the gradient difference \cite{yang2019quasi}. Then, to improve the CRU of SQN methods when applied to real-world VFL systems with unbalanced computation resource, we consider asynchronously parallelizing the SQN algorithms for VFL. However, such parallelization can be difficult because those $\theta_i^\ell$'s contributed by the other parties for computing the $k$-th ($k$ is current local iteration number) local $H^\ell_k$ are always stale, which may lead to an unstable SLBFGS process and make the generated local $H_k$ (for notation abbreviation, we omit the superscript $\ell$, so do other notations in this section) even not positive semidefinite for (strongly) convex problem. To address this challenge, motivated by damped LBFGS \cite{nocedal2006numerical} for nonconvex problem, we turn to designing the stochastic damped L-BFGS (SdLBFGS) for VFL with painstaking to ensure the generated $H_k$ be positive semidefinite (please refer to the reXiv version for the proof). The corresponding algorithm is shown in Algorithm \ref{SdLBFGS}.
	\begin{algorithm}[!t]
		\caption{{ Stochastic damped L-BFGS on Party $\ell$}.}\label{SdLBFGS}
		\begin{algorithmic}[1]
			\REQUIRE {Let $k$ be current local iteration number, $m$ is the memory size, stochastic gradient estimator $v_{k-1}^{\ell}$, index set $\I_k$ at local iteration $k$ and vector pairs $\{s_i^\ell, \bar{y}_i^\ell, \rho_{i}^\ell\}$ $i=k-m,\ldots,k-2$ stored on party $\ell$, and let $u_0=v_k^\ell$}
			\STATE Calculate $s_{k-1}^\ell $, $\bar{y}_{k-1}^\ell$ and $\gamma_k^\ell$
			\STATE  Calculate $\hat{y}_{k-1}^\ell$ through Eq.~\ref{2.1} and $\rho_{k-1}^\ell=((s_{k-1}^\ell)^\top {\hat{y}_{k-1}^\ell})^{-1}$
			\FOR {$i=0,\ldots,\min\{m,k-1\}-1$}
			\STATE  Calculate $\mu_i^\ell=\rho_{k-i-1}^\ell(u_i^\ell)^\top s_{k-i-1}^\ell$
			\STATE  Calculate $u_{i+1}^\ell = u_i^\ell - \mu_i^\ell {\hat{y}_{k-i-1}^\ell}$
			\ENDFOR
			\STATE  Calculate $v_0^\ell=(\gamma_k^\ell)^{-1}u_p^\ell$
			\FOR {$i=0,\ldots,\min\{m,k-1\}-1$}
			\STATE  Calculate $\nu_i^\ell=\rho_{k-m+i}^\ell(v_i^\ell)^\top {\hat{y}_{k-m+i}^\ell}$
			\STATE  Calculate $
			\bar{v}_{i+1} = \bar{v}_i^\ell + (\mu_{m-i-1}^\ell-\nu_i^\ell)s_{k-m+i}^\ell$.
			\ENDFOR
			\ENSURE {$d^\ell=H_k^\ell v_k^\ell=\bar{v}_p^\ell$.}
		\end{algorithmic}
	\end{algorithm}
	
	Instead of using gradient difference transmitted from the other parties, Algorithm \ref{SdLBFGS} uses the history information stored locally on party $\ell$ to generate a local descent direction $d_k = H_kv_k$ without calculating inverse matrix $H_k$ explicitly. In Algorithm \ref{SdLBFGS},  $s_{k-1} = x_{k}-x_{k-1}$ and $\bar{y}_{k-1} = v_{k} - v_{k-1}$, and $\gamma_{k} = {\mathrm{max}}\{\frac{\bar{y}_{k-1}^\top \bar{y}_{k-1}}{s_{k-1}^\top \bar{y}_{k-1}}, \delta \}$, where $\delta$ is a positive constant. Different from traditional SLBFGS for convex problem \cite{byrd2016stochastic}, algorithm \ref{SdLBFGS} introduces a new vector {$\hat{y}_{k-1}$}.
	\be {\label{2.1}}
	\hat{y}_{k-1} = \theta_{k-1} \bar{y}_{k-1} + (1-\theta_{k-1})H_{{k-1},0}^{-1}s_{k-1}, k\geq1,
	\ee
	where $H_{k,0} = \gamma_{k}^{-1}I_{d_\ell \times d_\ell} $, $k\geq0$, and $\theta_{k-1}$ is defined~as
	\be \label{2.2}
	\theta_{k-1}=\left\{\begin{array}{ll}{\frac{0.7 \sigma_{k-1}}{\sigma_{k-1}-s_{k-1}^{\mathrm{T}} \bar{y}_{k-1}},} & {\text { if } s_{k-1}^{\mathrm{T}} \bar{y}_{k-1}<0.3 \sigma_{k-1}} \\ {1,} & {\text { otherwise }}\end{array}\right.,
	\ee
	where $\sigma_{k-1} = s_{k-1}^{\mathrm{T}} H_{k, 0}^{-1} s_{k-1}$.  Since vector pairs $\{s_{i}, \hat{y}_{i}\}_{i=0}^{k-1}$ are obtained, and then $H_kv_k$ can be approximated through the two-loop recursion \ie, steps 3 to 10. Importantly,  the local $H_k$ implicitly generated is positive semidefinite despite that the history information is stale (Please refer to arXiv version for the proof).
	
	In our AsySQN framework, each party needs to compute the corresponding stochastic (block-coordinate) gradient estimator $v_k$ for generating the local $H_k$. Given $f_i(w)$ defined in Problem \ref{P}, the block-coordinate gradient can be represented as
	\begin{equation}\label{gradient}
		\nabla_{\bG_\ell} f_i(w) =H(\theta_i,y_i)(x_i)_{\ell} + \lambda \nabla g_\ell(w_{\bG_\ell}),
	\end{equation}
	where $H(\theta_i,y_i)=\frac{\partial \mathcal{L} (\theta_i,y_i)}{\partial \theta_i}$. Thus, a party need obtain $\theta_i=\sum_{\ell=1}^{q} w_{{\ell}}^{\mathrm{T}} {(x_i)}_{{\ell}}$ for computing the block-coordinate gradient. To avoid the large communication overhead of directly transmitting $ w_{{\ell}}^{\mathrm{T}}$ and ${(x_i)}_{{\ell}}$ and prevent the directly leaking of them, we consider transmitting the  computational results of $\theta_i^\ell$.  Many recent works achieved this in different manners \cite{zhang2021secure,liu2019communication,hu2019fdml,gu2020federated}. In this paper, we use the efficient tree-structured communication scheme \cite{zhang2018feature}.
	
	{\noindent{\bf{Aggregating $\theta_i$ with Privacy-Preserving:}}} The details are summarized in Algorithm~~\ref{safer_tree}. Specifically, at step~2, $\theta_i^\ell=w_{\ell}^{\mathrm{T}} {(x_i)}_{\ell}$ is computed locally on the $\ell$-th party to prevent the direct leakage of $w_{\ell}$ and ${(x_i)}_{\ell}$. Especially, a random number $\delta_{\ell}$ is added to $\theta_i^\ell$ to mask the value of $\theta_i^\ell$, which can enhance the security during aggregation process. At steps~4 and 5, $\theta_1$ and $\theta_2$ are aggregated through tree structures $T_1$ and $T_2$, respectively. Note that $T_2$ is totally different (refer to \cite{gu2020federated} for definition) from $T_1$  that can prevent the random value being removed under threat model 1 (defined in Section \ref{sec_security}).
	Finally, value of $\theta_i$ is recovered by removing term $\sum_{\ell =1}^{q} \delta_{\ell }$ from $\sum_{\ell =1}^{q} (\theta_i^\ell +\delta_{\ell})$ at the output step. Using such aggregation strategy, ${(x_i)}_{{\ell }}$ and $w_{{\ell }} $ are prevented from leaking during the aggregation, the data and model securities are thus guaranteed
	\begin{algorithm}[!t]
		\caption{Safe algorithm of obtaining $\theta_i$ }\label{safer_tree}
		\begin{algorithmic}[1]
			\REQUIRE { $w$, index $i$ } \\
			{ \bf{Do this in parallel}}
			\FOR {$\ell '=1, \cdots, q$}
			\STATE Generate a  random number $\delta_{\ell '}$ and calculate $\theta_i^{\ell '}+ \delta_{\ell '}$,
			\ENDFOR
			\STATE Obtain $\varphi_1 = \sum_{\ell '=1}^{q} (\theta_i^{\ell '} +\delta_{\ell '})$ based on tree structure $T_1$.
			\STATE Obatin $\varphi_2 = \sum_{\ell '=1}^{q} \delta_{\ell '} $ based on significally different tree structure (please refer to \cite{gu2020federated}) $T_2\neq T_1$.
			\ENSURE  {$\theta_i =\varphi_1-\varphi_2$}
		\end{algorithmic}
	\end{algorithm}
	\subsection{AsySQN-Type Algorithms for VFL}
Vanilia SGD \cite{bottou2010large}  and its variance reduction variants \cite{huang2019faster,huang2020accelerated,huang2019nonconvex} are  popular methods for learning machine learning (ML) models \cite{dang2020large,yang2020adversarial}. In this paper, we thus propose  three SGD-based AsySQN algorithms.

	\noindent{\bf AsySQN-SGD:}
	First, we propose the vanilla AsySQN-SGD summarized in Algorithm~\ref{AsySQNSGD}. At each iteration,  AsySQN-SGD randomly samples a batch of samples $\I$ with replacement, and then obtain the vector $\theta_{\I}$ asynchronously based on tree-structured communication. Based on the received  vector $\theta_{\I}$,  $\nabla_{\ell}f_{{i}}(\widehat{w})$ is computed as Eq.~\ref{gradient} and the stochastic gradient estimator used for generating $d_k$ is computed as $\widehat{v}^{\ell}= \nabla_{\ell}f_{{\I}}(\widehat{w}) =\frac{1}{\left|\I\right|} \sum_{i \in \I} \nabla_{\ell} f_{i}(\widehat{w})$.\\
	\begin{algorithm}[!t]
		\caption{AsySQN-SGD algorithm on the $\ell$-th  worker}\label{AsySQNSGD}
		\begin{algorithmic}[1]
			\REQUIRE {Data $\{D^{\ell '}\}_{\ell '=1}^{q}$ and learning rate $\gamma$}
			\STATE Initialize $w_{\ell}\in \mathbb{R}^d$.\\
			{ \bf{Keep doing in parallel}}
			\STATE \quad  Sample $\I \overset{\text{Unif}}{\sim} \{1,\ldots,n\}$.
			\STATE \quad Compute $\theta_{\I}=\{\theta_i\}_{i\in\I}$  asynchronously based on\\
			\quad Algorithm~\ref{safer_tree}.
			\STATE  \quad Compute $\widehat{v}^{\ell}=\nabla_{{\ell}} f_{\I}(\widehat{w})$.
			\STATE \quad Compute $d^\ell = {H^{\ell}}\widehat{v}^{\ell}$ through Algorithm \ref{SdLBFGS},
			\STATE \quad Update $w_{{\ell}} \leftarrow w_{{\ell}}-\gamma d^\ell$. \\
			{ \bf{End parallel }}
			\ENSURE {$w_{\ell}$}
		\end{algorithmic}
	\end{algorithm}
	\noindent{\bf AsySQN-SVRG:}
	The proposed AsySQN-SVRG with an improved convergence rate than AsySQN-SGD is shown in Algorithm~\ref{AsySQNSVRG}. Different from AsySQN-SGD directly using the stochastic gradient for updating, AsySQN-SVRG adopts the variance reduction technique to control the intrinsic variance of the stochastic gradient estimator. In this case,
	$\widehat{v}^{\ell}=\nabla_{{\ell}} f_{\I}(\widehat{w})-\nabla_{{\ell}} f_{\I}\left(w^{s}\right)+\nabla_{{\ell}} f\left(w^{s}\right)$.
	\begin{algorithm}[!t]
		\caption{AsySQN-SVRG algorithm on the $\ell$-th worker}\label{AsySQNSVRG}
		\begin{algorithmic}[1]
			\REQUIRE {Data $\{D^{\ell '}\}_{\ell '=1}^{q}$  and learning rate $\gamma$}
			\STATE Initialize $w_{\ell}^0\in \mathbb{R}^{d_\ell}$.
			\FOR {$s=0, 1, \ldots, S-1$}
			\STATE Compute the full local gradient $\nabla_{{\ell}} f\left(w^{s}\right)=\frac{1}{n} \sum_{i=1}^{n} \nabla_{{\ell}} f_{i}\left( {w}^{s}\right)$ through tree-structured communication.
			\STATE $w_{{\ell}}=w_{{\ell}}^{s}$.\\
			\textbf{Keep doing in parallel}
			\STATE \  Sample $\I \overset{\text{Unif}}{\sim} \{1,\ldots,n\}$.
			\STATE \ Compute $\theta_{\I}=\{\theta_i\}_{i\in\I}$
			and $\theta(w^s)_{\I}=\{\theta_i(w^s)\}_{i\in\I}$ \\
			\ based on Algorithm~\ref{safer_tree}.
			\STATE \  Compute $\widehat{v}^{\ell}=\nabla_{{\ell}} f_{\I}(\widehat{w})-\nabla_{{\ell}} f_{\I}\left(w^{s}\right)+\nabla_{{\ell}} f\left(w^{s}\right)$
			\STATE \  Compute $d^\ell = {H^{\ell}}\widehat{v}^{\ell}$ through Algorithm \ref{SdLBFGS},
			\STATE \  Update $w_{{\ell}} \leftarrow w_{{\ell}}-\gamma d^\ell$.\\
			\textbf{End parallel loop}
			\STATE $w_{{\ell}}^{s+1}=w_{{\ell}}$.
			\ENDFOR
			\ENSURE {$w_{\ell}$}
		\end{algorithmic}
	\end{algorithm}
	
	\noindent{\bf AsySQN-SAGA:}
	AsySQN-SAGA enjoying the same convergence rate with AsySQN-SVRG is shown in Algorithm~\ref{AsySQNSAGA}. Different from Algorithm~\ref{AsySQNSVRG} using ${w}^s$ as the reference gradient, AsySQN-SAGA uses the average of history gradients stored in a table. The corresponding $\widehat{v}^{\ell}$ is  computed as $\widehat{v}^{\ell}=\nabla_{{\ell}} f_{\I}(\widehat{w})-\widehat{\alpha}_{\I}^{\ell}+\frac{1}{n} \sum_{i=1}^{n} \widehat{\alpha}_{i}^{\ell}$.
		\begin{algorithm}[!t]
			\caption{AsySQN-SAGA algorithm on the $\ell$-th  worker}\label{AsySQNSAGA}
			\begin{algorithmic}[1]
				\REQUIRE {Data $\{D^{\ell '}\}_{\ell '=1}^{q}$ and learning rate $\gamma$}
				\STATE Initialize $w_{\ell}\in \mathbb{R}^{d_\ell}$.
				\STATE Compute the local gradient $\widehat{\alpha}_{i}^{\ell}=\nabla_{{\ell}} f_{i}(\widehat{w})$, for $\forall i \in\{1, \ldots, n\}$ through tree-structured communication.\\
				\textbf{Keep doing in parallel}
				\STATE \quad  Sample $\I \overset{\text{Unif}}{\sim} \{1,\ldots,n\}$.
				\STATE \quad Compute $\theta_{\I}=\{\theta_i\}_{i\in\I}$  asynchronously based on \\
				\quad Algorithm~\ref{safer_tree}.
				\STATE \quad Compute $\widehat{v}^{\ell}=\nabla_{{\ell}} f_{\I}(\widehat{w})-\widehat{\alpha}_{\I}^{\ell}+\frac{1}{n} \sum_{i=1}^{n} \widehat{\alpha}_{i}^{\ell}$.
				\STATE \quad  Compute $d^\ell = {H^{\ell}}\widehat{v}^{\ell}$ through Algorithm \ref{SdLBFGS},
				\STATE \quad Update $w_{{\ell}} \leftarrow w_{{\ell}}-\gamma d^\ell$.
				\STATE \quad Update $\widehat{\alpha}_{\I}^{\ell} \leftarrow \nabla_{{\ell}} f_{\I}(\widehat{w})$.\\
				\textbf{End parallel loop}
				\ENSURE {$w_{\ell}$}
			\end{algorithmic}
		\end{algorithm}
	\section{Convergence Analysis}
	In this section, the convergence analysis is presented. Please refer to the arXiv verison for details.
	\subsection{Preliminaries}
	\begin{assumption}\label{assum1}
		Each  function $f_i$, $i=1,\ldots,n$, is $\mu$-strongly convex, i.e., $\forall \ w,\  w'\in \mathbb{R}^d$ there exists a $\mu>0$ such that
		\begin{equation}
			f_i(w)\geq f_i(w') +  \langle \nabla f_i(w'), w- w' \rangle + \frac{\mu}{2}\|w-w'\|^2.
		\end{equation}
	\end{assumption}
	\begin{assumption}\label{assum2}
		For each  function $f_i$, $i=1,\ldots,n$, we assume the following conditions hold:\\
		\noindent{\bf 2.1 Lipschitz Gradient:}
		There exists  $L>0$ such that
		\begin{equation}
			\|\nabla f_i (w) - \nabla f_i (w')\| \le L\|w-w'\|, \quad \forall \ w,w'\in \mathbb{R}^d
		\end{equation}
		\noindent{\bf 2.2 Block-Coordinate Lipschitz Gradient:}
		There exists an $L_\ell>0$ for the $\ell$-th block, where $\ell\in[q]$ such that
		\begin{equation}
			\|\nabla_{\ell} f_i (w+U_\ell\Delta_\ell) - \nabla_{\ell} f_i (w)\| \le L_\ell\|\Delta_\ell\|,
		\end{equation}
		where $\Delta_\ell \in \mathbb{R}^{d_\ell}$, $U_\ell \in \mathbb{R}^{d\times d_\ell}$ and $[U_1, \cdots, U_q] = I_d$.
		
		\noindent{\bf 2.3 Bounded Block-Coordinate Gradient:}
		There exists a constant $G$ such that $\|\nabla_{\ell} f_i(w)\|^2\leq G $ for $\ell$, $\ell =1,\cdots,q$.
	\end{assumption}
	Assumptions \ref{assum2}.1 to \ref{assum2}.3  are standard for convergence analysis in previous works \cite{gu2020Privacy,zhang2021secure}.
	\begin{assumption}\label{assum3}
		We introduce the following assumptions necessary for the analysis of stochastic quasi-Newton methods.\\
		\noindent{\bf 3.1}
		For $ i\in [n]$, function $f_i(w)$ is twice continuously differentiable w.r.t. $w$ and for the $\ell$-th block, where $\ell\in [q]$,
		there exists two positive constant $\kappa_1$ and $\kappa_2$ such that for $\forall$ $w\in\mathbb{R}^{d_\ell}$ there is
		\begin{equation}\label{hess}
			\kappa_1 I \preceq \nabla^2_{\ell} f_i(w) \preceq  \kappa_2I
		\end{equation}
		where notation $A \preceq B $ with $A, B\in \mathbb{R}^{d_\ell \times d_\ell}$ means that $A-B$ is positive semidefinite.\\
		\noindent{\bf 3.2}
		There exist two positive constants $\sigma_{1}$, and $\sigma_{2}$ such that
		\begin{align}
			\sigma_{1}I \preceq H_k^\ell \preceq \sigma_{2}I, \quad \ell \in[q]
		\end{align}
		where {$H_k^\ell$ is the inverse Hessian approximation matrix} on party $\ell$.\\
		\noindent{\bf 3.3}
		For any local iteration $k\ge 2$, the random variable $H_k^\ell$ ($k\ge 2$) depends only on $v_{k-1}$ and $\mathcal{I}_{k}$
		\begin{align}
			\mathbb{E}[H_kv_k|\mathcal{I}_{k}, v_{k-1}] = H_kv_k,
		\end{align}
		where the expectation is taken with respect to $|\mathcal{I}_{k}|$ samples generated for calculation of $\nabla f_{\mathcal{I}_k}$.
	\end{assumption}
	
	Assumptions \ref{assum3}.1 to \ref{assum3}.3 are standard assumptions for SQN methods \cite{zhang2020faster}. Specifically, Assumption \ref{assum3}.2 shows that the matrix norm of $H_k$ is bounded. Assumption \ref{assum3}.3 means that given $v_{k-1}$ and $\I_{k}$ the $H_kv_k$ is determined. Similar to previous asynchronous work \cite{zhang2021secure,gu2020Privacy}, we introduce the following definition and assumption.
	\begin{definition}\label{defin1}
		$D(t)$ is defined as a set of  iterations, such that:
		\begin{equation}
			\widehat{w}_t-w_t = \gamma \sum_{u\in D(t)}U_{\psi(u)}\widehat{v}_u^{\psi(u)},
		\end{equation}
		where $u \le t$ for $\forall u \in D(t)$.
	\end{definition}
	\begin{assumption}\label{assum4} {\bf{Bounded Overlap:}}
		We assume that there exists a constant $\tau_1$ which bounds the maximum number of iterations that can overlap together, i.e., for $\forall t$ there is  $\tau_1 \geq t-{\text{min}}\{u|u\in D(t)\}$.
	\end{assumption}
	
	To track the behavior of the global model in the convergence analysis, it is necessary to introduce $K(t)$.
	\begin{definition}\label{defin2} {\bf{$K(t)$:}}
		The minimum set of successive iterations that fully visit all coordinates from  global iteration number~$t$.
	\end{definition}
	\begin{assumption}\label{assum5}
		We assume that the size of $K(t)$ is upper bounded by $\eta_1$, \ie, $|K(t)|\leq \eta_1$.
	\end{assumption}
	Based on Definition \ref{defin2} and Assumption \ref{assum5}, we introduce the epoch number $v(t)$, which our convergence analyses are built on.
	\begin{definition}\label{defin3}
		Let $P(t)$ be a partition of $\{0,1,\cdots, t-\sigma'\}$, where $\sigma'\geq0$. For any $\kappa\subseteq P(t)$ we have that there exists $t'\leq t$ such that $K(t')=\kappa$, and $\kappa_1 \subseteq P(t)$ such that $K(0)=\kappa_1$. The epoch number for the global $t$-th iteration, i.e., $v(t)$ is defined as the maximum cardinality of $P(t)$.
	\end{definition}
	\subsection{Convergence Analyses}
	\begin{theorem}\label{thm-sgdconvex}
		Under Assumptions \ref{assum1}-\ref{assum5}, to achieve the accuracy $\epsilon$ of (\ref{P}) for AsySQN-SGD, \emph{i.e.}, $\mathbb{E} f (w_{t}) -f(w^*) \leq \epsilon$, we set $\gamma =  \frac{-   L_{\max}   + \sqrt{  L_{\max}^2  + \frac{{2 \mu \sigma_1 b \epsilon  (\sigma_1L^2 \eta_1^2  + \sigma_2\tau_1  L^2 \tau )}}{G\eta_1 \sigma_2^2}} }{2 L^2 (\sigma_1\eta_1^2  + \sigma_2\tau_1^2 )}$ and the epoch  number $\upsilon(t)$ should satisfy the following condition.
		\begin{align}\label{equ_theorem2_o.1}
			\upsilon(t) \geq \frac{2}{ \mu \gamma\sigma_1}  \log\left ( \frac{2 \left (  f(w_0)-f(w^*) \right )}{\epsilon} \right )
		\end{align}
	\end{theorem}
	\begin{theorem} \label{thm-svrgconvex}
		Under Assumptions \ref{assum1}-\ref{assum5}, to achieve the accuracy $\epsilon$ of (\ref{P}) for AsySQN-SVRG,  let $C=\left (\sigma_1\gamma L^2 \eta_1^2 + L_{\max}  \right )\frac{\gamma^2\sigma_2^2}{2}$ and $\rho = \frac{\gamma \mu\sigma_1}{2} -  \frac{16 L^2 \eta_1 C}{\mu} $, we  choose $\gamma$ such that
		\begin{align}
			&&1) \ \rho < 0; \quad 2) \  \frac{8 L^2 \eta_1 C}{\rho \mu} \leq 0.5
			\\ && 3) \ \gamma^3 \left (  \left ( \frac{ \sigma_1  }{2} +     \frac{2C}{\gamma} \right )   \tau_1^2    + 4 \frac{C}{\gamma}     \eta_1^2 \right )   \frac{9\eta_1  L^2 \sigma_2^2 G}{b\rho} \leq \frac{\epsilon}{8}
		\end{align}
		the inner epoch  number should satisfy $\upsilon(t) \geq \frac{\log 0.25}{\log (1 - \rho)}$, and the outer loop number  should satisfy $S \geq \frac{\log \frac{2 ( f (w_{0}) -f(w^*) )}{\epsilon }}{\log \frac{4}{3}} $.
	\end{theorem}
	\begin{theorem}\label{thm-sagaconvex}
		Under Assumptions \ref{assum1}-\ref{assum5}, to achieve the accuracy $\epsilon$ of (\ref{P}) for AsySQN-SAGA, \emph{i.e.}, $\mathbb{E} f (w_{t}) -f(w^*) \leq \epsilon$, let ${c_0} =  \left ( \frac{\tau_1^2 }{2}   + (9\tau_1^2+8\eta_1^2)\sigma_2\gamma\left ( \gamma L^2 \sigma_1\eta_1^2 + L_{\max} \right )  \right ) \sigma_2^3\gamma^3 L^2 \eta_1\frac{G}{b}$, \\
		$c_1 = \left ( \gamma L^2 \sigma_1\eta_1^2 + L_{\max} \right ) \sigma_2^2\gamma^2\eta_1\frac{L^2}{b}$,  ${c_2}  = 4\left ( \gamma L^2 \sigma_1\eta_1^2 + L_{\max} \right ) *$\\
$   \frac{L^2 \eta_1^2\sigma_2^2\gamma^2}{nb } $, and let $\rho \in (1 -\frac{1}{n},1)$, we  choose $\gamma$ such that
		\begin{eqnarray}
			\nonumber
			&1)	\frac{4 c_0}{ \gamma \sigma_1 \mu (1-\rho) \left ( \frac{\gamma \sigma_1 \mu^2}{4} -2 c_1-c_2 \right )} \leq \frac{\epsilon}{2}, \quad 2)
			0<1 -  \frac{\gamma \sigma_1 \mu}{4}  <1
			\\ \nonumber
			&3)-  \frac{\gamma \sigma_1 \mu^2}{4} +2 c_1+c_2 \left ( 1+   \frac{1}{1- \frac{ 1 -\frac{1}{n}}{\rho}}  \right ) \leq 0
			\\
			&4)-  \frac{\gamma \sigma_1 \mu^2}{4} +c_2+  c_1 \left ( 2+ \frac{1}{1- \frac{ 1 -\frac{1}{n}}{\rho}} \right )   \leq 0
		\end{eqnarray}
		the epoch  number should satisfy {$ \upsilon(t)  \geq \frac{\log \frac{2 \left ({{2\rho- 1 +  \frac{\gamma \sigma_1 \mu}{4} }} \right ) ( f (w_{0}) -f(w^*) ) }{\epsilon {(\rho- 1 +  \frac{\gamma\sigma_1 \mu}{4} )\left ( \frac{\gamma \sigma_1 \mu^2}{4} -2 c_1-c_2 \right )} }}{\log \frac{1}{\rho}}$}.
	\end{theorem}
	\begin{remark}
		For strongly convex problems, given the assumptions and  parameters in corresponding theorems, the convergence rate of AsySQN-SGD is $\mathcal{O} (\frac{1}{{\epsilon}}\text{log}(\frac{1}{\epsilon}))$, and those of AsySQN-SVRG and AsySQN-SAGA are $\mathcal{O} (\text{log}(\frac{1}{\epsilon}))$.
	\end{remark}
	\section{Complexity Analyses}
	In this section, we present the computation and communication complexity analyses of our framework.
	\subsection{Computation Complexity Analysis}
	SQN methods incorporated with approximate Hessian  information indeed converge faster than SGD methods in practice, however, from the perspective of its applications to industry, there is a concern that it may introduce much extra computation cost. In the following, we will show that the extra computation cost introduced by the approximate second-order information is negligible.
	
	First, we analyze the computation complexity of Algorithm~\ref{SdLBFGS}.  At step 1,  two inner products take $2d_\ell$  multiplications. At step 2, two inner products and one scalar-vector product take $3d_\ell$ multiplications. The first recursive loop (\ie, Steps 3 to 5) involves $2m$ scalar-vector multiplications and $m$ vector inner products, which takes $3md_\ell$ multiplications. So does the second loop (\ie, Steps 8 to 10). At step 7, the scalar-vector product takes $d_\ell$ multiplications. Therefore, Algorithm \ref{SdLBFGS} takes $(6m+d)d_\ell$ multiplications totally. As for Algorithm~\ref{safer_tree}, all multiplications are performed at steps 1-3, which takes $\sum_{\ell=1}^{q}d_\ell=d$ multiplications.
	
	Then we turn to Algorithm~\ref{AsySQNSGD}. At step 3, it takes $bd$ multiplications to compute the vector $\theta_i$. Step 4 takes $bd_\ell$ multiplications to compute the mini-batch stochastic gradient. Step 5 calls Algorithm~\ref{SdLBFGS} and thus takes $(6m+d)d_\ell$ multiplications. The scalar-vector product at step 6 takes $d_\ell$ multiplications. Compared with Algorithm~\ref{AsySQNSGD}, Algorithm~\ref{AsySQNSVRG} need compute the full gradient at the beginning of each epoch, which takes $nd$ multiplications. Moreover, step 6 takes $2bd$ multiplications and step 7 takes $2bd_\ell$ multiplications. As for other steps, the analyses are similar to Algorithm~\ref{AsySQNSGD}. In terms of Algorithm~\ref{AsySQNSAGA}, the initialization of the gradient matrix takes $nd$ multiplications and the analyses of other steps are similar to Algorithm~\ref{AsySQNSGD}.
	
	We summarize the detailed computation complexity of Algorithms~\ref{SdLBFGS} to \ref{AsySQNSAGA} in Table~\ref{complexity table}. Based on the results in Table~\ref{complexity table},  it is obvious that, for Algorithm \ref{AsySQNSGD}, the extra computation cost of computing the approximate second-order information takes up $r=\frac{md_\ell}{bd+md_\ell}< \frac{md_\ell}{bd}$ in the whole procedure. The extra computation cost is negligible because 1) as in previous work \cite{ghadimi2016mini} we can choose $b=n^{-1/2}$ and $n$ is sufficiently large in big data situation, 2) $m$ ranges from 5 to 20 as suggested in \cite{nocedal2006numerical}, 3) generally, $d_\ell$ is much smaller than $d$.  As for Algorithms \ref{AsySQNSVRG} and \ref{AsySQNSAGA}, there are $r=\frac{TSmd_\ell}{ndS+TS(bd+md_\ell)}$ and $r=\frac{Tmd_\ell}{nd+T(bd+md_\ell)}$, respectively, which are also negligible.
	\begin{table*}[!t]
		\centering
		\caption{Total computational complexities (TCC) of going through Algorithms~\ref{SdLBFGS} to \ref{AsySQNSAGA}, where $t$ denotes the total iteration number for Algorithms \ref{AsySQNSGD} and \ref{AsySQNSAGA}, and the number of iterations in an epoch for Algorithm \ref{AsySQNSVRG}, and $S$ is the epoch number for Algorithm~\ref{AsySQNSVRG}.}
		\label{complexity table}
		\setlength{\tabcolsep}{1.4mm}{{\small\begin{tabular}{llllllllll}
				\toprule
				\multicolumn{2}{c}{Algorithm~\ref{SdLBFGS} }& \multicolumn{2}{c}{Algorithm~\ref{safer_tree}}  & \multicolumn{2}{c}{Algorithm~\ref{AsySQNSGD}} & \multicolumn{2}{c}{Algorithm~\ref{AsySQNSVRG}} & \multicolumn{2}{c}{Algorithm~\ref{AsySQNSAGA}} \\ \midrule
				steps          & TCC                 & steps         & TCC                     & steps         & TCC                      & steps        & TCC                             & steps        & TCC                              \\
				1  & $\bO(d_\ell)$          & 1-3   & $\bO(d)$                   &3   & $\bO(bd )$
				& 3    & $\bO(nd)$      & 2            & $\bO(nd)$
				\\
				2  & $\bO(d_\ell)$           & -   & -                  &4   & $\bO(bd_\ell)$
				& 6  & $\bO(bd)$   & 4            & $\bO(bd)$
				\\
				3-6    & $\bO(md_\ell)$           & -   & -                  &5   & $\bO(md_\ell)$
				& 7   & $\bO(bd_\ell)$   & 5            & $\bO(bd_\ell)$
				\\
				7       & $\bO(d_\ell)$      & -   & -                    &6   & $\bO(d_\ell)$
				& 8     & $\bO(md_\ell)$               & 6    & $\bO(d_\ell)$
				\\
				8-11      & $\bO(md_\ell)$        & -   & -                   &4   & $\bO(d_\ell)$
				& 9        & $\bO(d_\ell)$   & -            & -
				\\ \hline
				\textbf{total}          & $\bO(md_\ell)$
				& \textbf{total}           & $\bO(d)$
				& \textbf{total}            &   {$\bO[T(bd +md_\ell)]$}
				& \textbf{total}          &   {$\bO[ndS+TS(bd+md_\ell)]$}
				& \textbf{total}          & {$\bO[nd +T(bd+md_\ell)]$}
				\\ \bottomrule
		\end{tabular}}}
	\end{table*}
	\subsection{Communication Complexity Analysis}
	The communication  of Algorithm \ref{safer_tree} is $\bO(q)$. For Algorithm \ref{AsySQNSGD}, step 3  has a communication complexity of $\bO(bq)$ due to calling Algorithm \ref{safer_tree} $b$ times. Similarly, for Algorithm \ref{AsySQNSVRG}, communication complexity of steps 3 and 6 are $\bO(nq)$ and $\bO(bq)$, respectively. For Algorithm \ref{AsySQNSAGA}, communication complexity of steps 2 and 4 are $\bO(nq)$ and $\bO(bq)$, respectively. Given $T$ and $S$ defined in Table \ref{complexity table}, the total communication complexities of Algorithms \ref{AsySQNSGD} to \ref{AsySQNSAGA} are $\bO(Tbq)$, $\bO(S(n+T)bq)$ and $\bO((n+T)bq)$, respectively.
	\section{Privacy Security Analysis}\label{sec_security}
	In this section, we analyze the data security and model security of AsySQN framework under two semi-honest threat models commonly used in previous works \cite{cheng2019secureboost,xu2019hybridalpha,gu2020federated}. Especially, adversaries under threat model 2 can collude with each other, thus they have the stronger attack ability than those under threat model 1.
	
	\noindent{\bf{Honest-but-Curious}} (Threat Model 1): All workers will follow the federated learning protocol to perform the correct computations. However, they may use their own retained records of the intermediate computation result to infer other worker's data and model.
	
	\noindent {\bf{Honest-but-Colluding}} (Threat Model 2): All workers will follow the federated learning protocol to perform the correct computations. However, some workers may collude to infer other worker's data and model by sharing their retained records of the intermediate computation result.
	
	Similar to \cite{gu2020federated,gu2020Privacy}, we analyze the security of AsySQN by analyzing its ability to prevent following inference attacks.
	\begin{definition}[{\bf{Exact Inference Attack}}]\label{definatt1}
		The adversary perform the inference attack by inferring $(x_i)_{\ell}$ (or $w_{\ell}$) belonging to other parties without directly accessing them.
	\end{definition}
	\begin{definition}[{\bf{$\epsilon$-Approximate Exact Inference Attack}}]\label{definatt2}
		The adversary perform the $\epsilon$-approximate exact inference attack by inferring  $(x_i)_{\ell}$ (or $w_{\ell}$) belonging to other parties  as $(\widetilde{{x}}_i)_{\ell}$ (or $\widetilde{w}_{\ell}$) with accuracy of $\epsilon$ (\emph{i.e.}, $\|({x_i})_{\ell}-(\widetilde{{x}}_i)_{\ell}\|_{\infty}\leq \epsilon$, $\|{w}_{\ell}-\widetilde{w}_{\ell}\|_{\infty}\leq \epsilon$, or $\|y_i-\widetilde{y_i}\|_{\infty}\leq \epsilon$ ) without directly accessing them.
	\end{definition}
	\begin{lemma}\label{infinite}
		Given equations $\theta_i^\ell = w_{{\ell }}^{\mathrm{T}} {(x_i)}_{{\ell }}$   with only observing $\theta_i^\ell$, there are infinite different solutions to both equations.
	\end{lemma}
	\begin{proof}
		First, we consider the equation $\theta_i^\ell = w_{{\ell }}^{\mathrm{T}} {(x_i)}_{{\ell }}$ with two cases, \ie, $d_{\ell}\geq 2$ and $d_{\ell}=1$. For $\forall d_{\ell}\geq 2$, given an arbitrary non-identity orthogonal matrix $U\in \mathbb{R}^{d_\ell\times d_\ell}$, we have
		\begin{equation}\label{proofinfinite1}
			(w_{{\ell }}^{\mathrm{T}}U^{\mathrm{T}})(U {(x_i)}_{{\ell }}) = w_{{\ell }}^{\mathrm{T}}(U^{\mathrm{T}}U) {(x_i)}_{{\ell }} = w_{{\ell }}^{\mathrm{T}} {(x_i)}_{{\ell }}=\theta_i^\ell
		\end{equation}
		From Eq.~\ref{proofinfinite1}, we have that given an equation $\theta_i^\ell = w_{{\ell }}^{\mathrm{T}} {(x_i)}_{{\ell }}$ with only $\theta_i^\ell$ being known, the solutions corresponding to  $w_{{\ell }}$ and  ${(x_i)}_{{\ell }}$ can be represented as $w_{{\ell}}^{\mathrm{T}}U^{\mathrm{T}}$ and  $U {(x_i)}_{{\ell}}$, respectively. $U$ can be arbitrary different non-identity orthogonal matrices, the solutions are thus infinite. If $d_\ell=1$, give an arbitrary real number $u\neq1$, we have
		\begin{equation}\label{proofinfinite2}
			(w_{\ell}^{\mathrm{T}}u)(\frac{1}{u} {(x_i)}_{\ell}) = w_{\ell}^{\mathrm{T}}(u\frac{1}{u}) {(x_i)}_{\ell} =  w_{\ell}^{\mathrm{T}} {(x_i)}_{\ell}=\theta_i^\ell
		\end{equation}
		Similarly, we have that the solutions of equation  $\theta_i^\ell = w_{\ell}^{\mathrm{T}} {(x_i)}_{\ell}$ are infinite when $d_\ell=1$.
		This completes the proof.
	\end{proof}
	Based on lemma~\ref{infinite}, we obtain the following theorem.
	\begin{theorem}\label{security}
		AsySQN can prevent the exact and the $\epsilon$-approximate exact inference attacks under semi-honest threat models.
	\end{theorem}
	\begin{proof}
		We prove above Theorem~\ref{security} under following two threat models.\\
		\noindent{\bf Threats Model 1:}
		During the aggregation, the value of $\theta_i^\ell = w_{\ell}^{\mathrm{T}} {(x_i)}_{\ell}$ is masked by $\delta_{\ell }$ and  just the value of  $\theta_i^\ell + \delta_{\ell }$ is transmitted. In this case, one cannot even access the true value of $\theta_i^\ell$, let alone using relation $\theta_i^\ell = w_{\ell}^{\mathrm{T}} {(x_i)}_{\ell}$ to refer $w_{\ell}^{\mathrm{T}}$ and ${(x_i)}_{\ell}$. Thus, AsySQN can prevent both exact inference attack and the $\epsilon$-approximate exact inference attack under Threat Model 1.
		
		\noindent{\bf Threats Model 2:}
		Under threat model 2, the adversary can remove the random value $\delta_{\ell}$ from  term $\theta_i^\ell+ \delta_{\ell }$ by colluding with other adversaries. Applying Lemma~\ref{infinite} to this circumstance, and we have that even if the random value is removed  it is still impossible to exactly refer $w_{\ell}^{\mathrm{T}}$ and ${(x_i)}_{\ell}$.  However, it is possible to approximately infer $w_{\ell}$ when $d_\ell =1$.
		Specifically, if one knows the region of ${(x_i)}_{\ell}$ as $\mathcal{R}$ (e.g., applying z-score
		normalization to $\{(x_i)_{\ell}\}_{i=1}^{n}$), one has $\theta_i^\ell/w_{\ell}^{\mathrm{T}} \in \mathcal{R}$, thus can infer $w_{\ell}^{\mathrm{T}}$ approximately, so does infer ${(x_i)}_{\ell}$ approximately.  Importantly, one can avoid this attack easily by zero-padding to make $d_\ell\geq2$.  Thus, under this threat model, AsySQN is secure in practice.
	\end{proof}
	\section{Experiments}
	\begin{table*}[!t]
		\centering
		\caption{Datasets used in the experiments.}
		\label{dataset}
		\begin{tabular}{@{}ccccccccc@{}}
			\toprule
			& $D_1$ & $D_2$ & $D_3$ & $D_4$&  $D_5$  & $D_6$ & $D_7$ &$D_8$ \\
			\midrule
			\#Samples & 24,000 & 96,257 & 17,996 & 49,749 & 677,399 & 32,561 &  400,000  & 60,000 \\
			\#Features & 90 & 92 & 1,355,191  & 300 & 47,236 & 127 & 2,000 & 780 \\ \bottomrule
		\end{tabular}
	\end{table*}
	In this section, we implement extensive  experiments on real-world datasets to demonstrate the lower communication cost and better CRU of AsySQN. The results concerning convergence speed also consistent to the corresponding theoretical results.
	\subsection{Experiment Settings}
	All experiments are performed on a machine with four sockets, and each sockets has 12 cores. We use OpenMPI to implement the communication scheme. In the experiments, there are $q=8$ parties and each party owns nearly equal number of features.
	$\delta$ is fixed for specific SGD-type of algorithms. As for the learning rate $\gamma$, we choose a suitable one from $\{ \cdots, 1e^{-1}, 5e^{-2}, 2e^{-2}, 1e^{-2},\cdots\}$ for all experiments. Moreover, to simulate the industry application scenarios, we set a synthetic party which only has 30\% to 60\% computation resource compared with the faster party. This means that, as for synchronous algorithms, the faster party has an poor CRU around  only 30\% to 60\% due to waiting for the straggler for synchronization.
	\begin{figure*}[!t]
		\centering
		\ \begin{subfigure}{0.22\linewidth}
			\includegraphics[width=\linewidth]{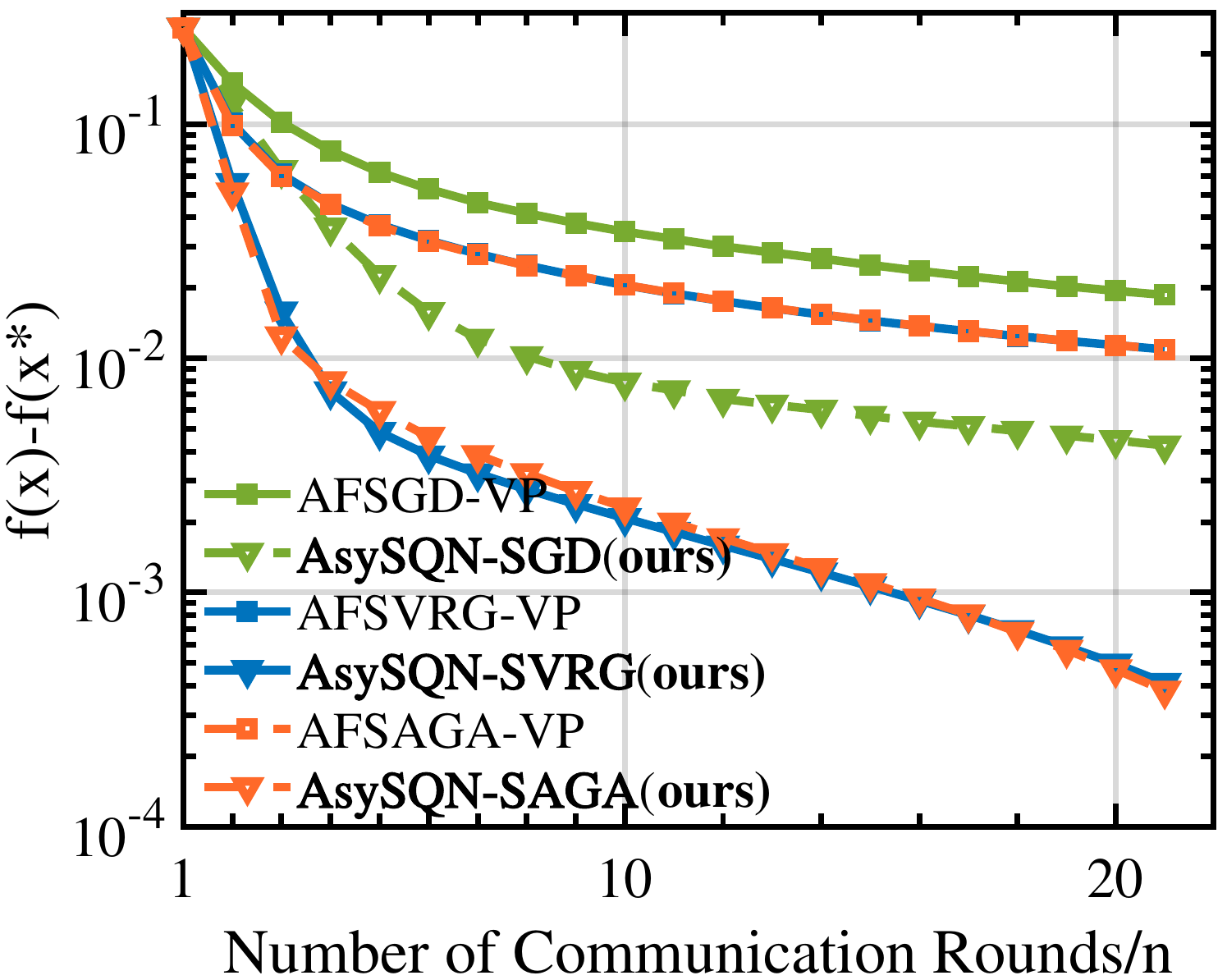}
			\caption{Results on $D_1$}
		\end{subfigure}
		\quad \begin{subfigure}{0.22\linewidth}
			\includegraphics[width=\linewidth]{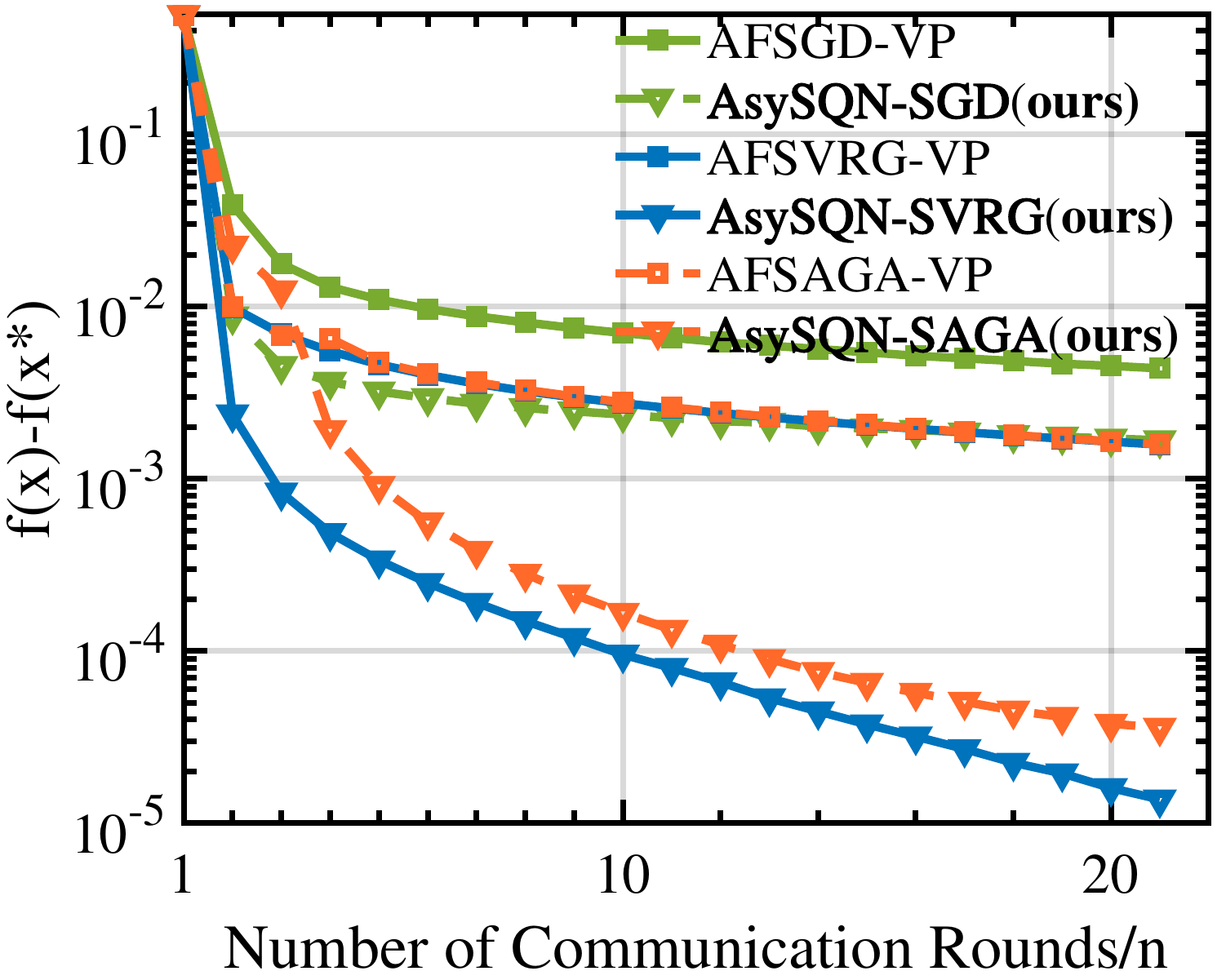}
			\caption{Results on $D_2$}
		\end{subfigure}
		\quad \begin{subfigure}{0.22\linewidth}
			\includegraphics[width=\linewidth]{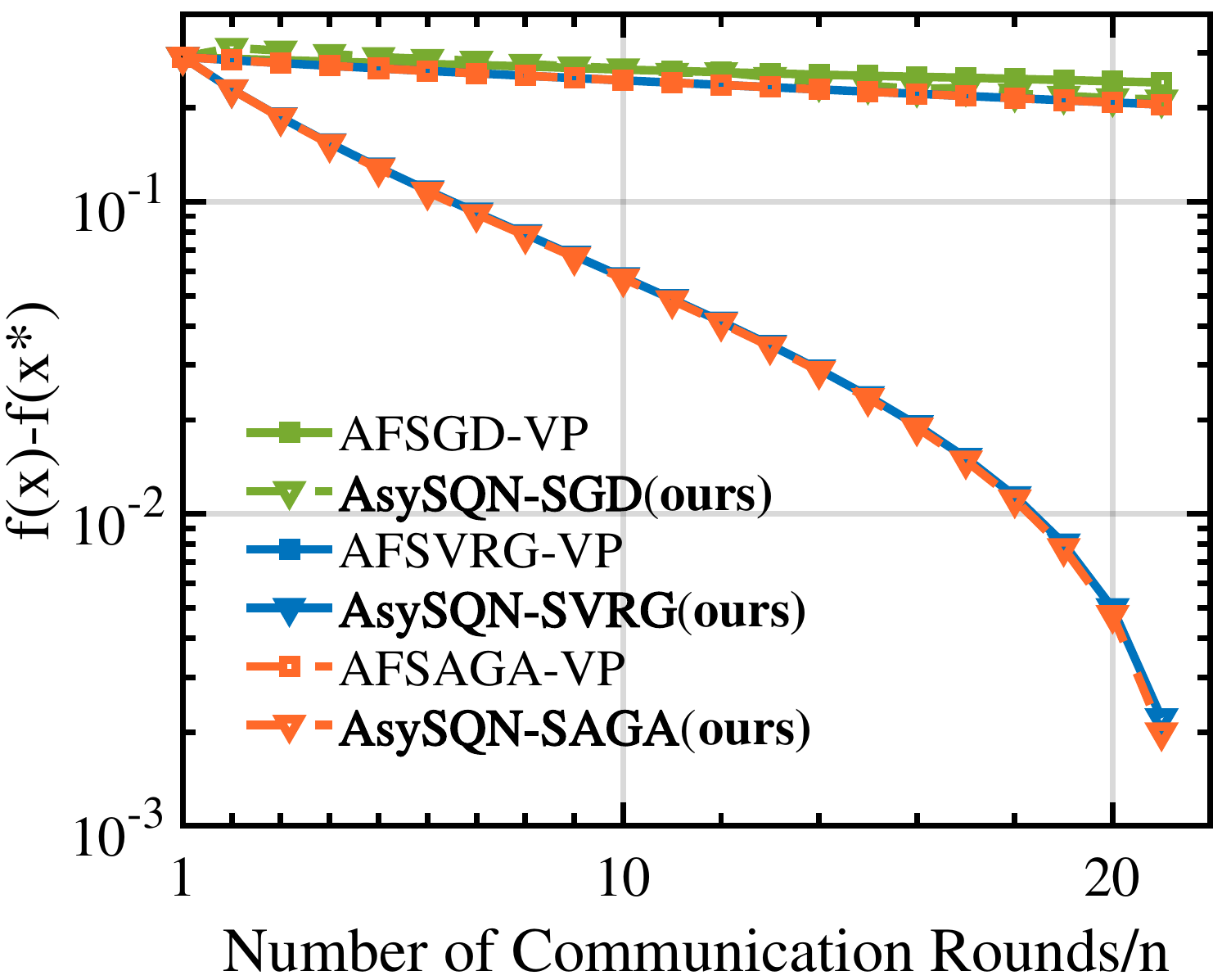}
			\caption{Results on $D_3$}
		\end{subfigure}%
		\quad \begin{subfigure}{0.22\linewidth}
			\includegraphics[width=\linewidth]{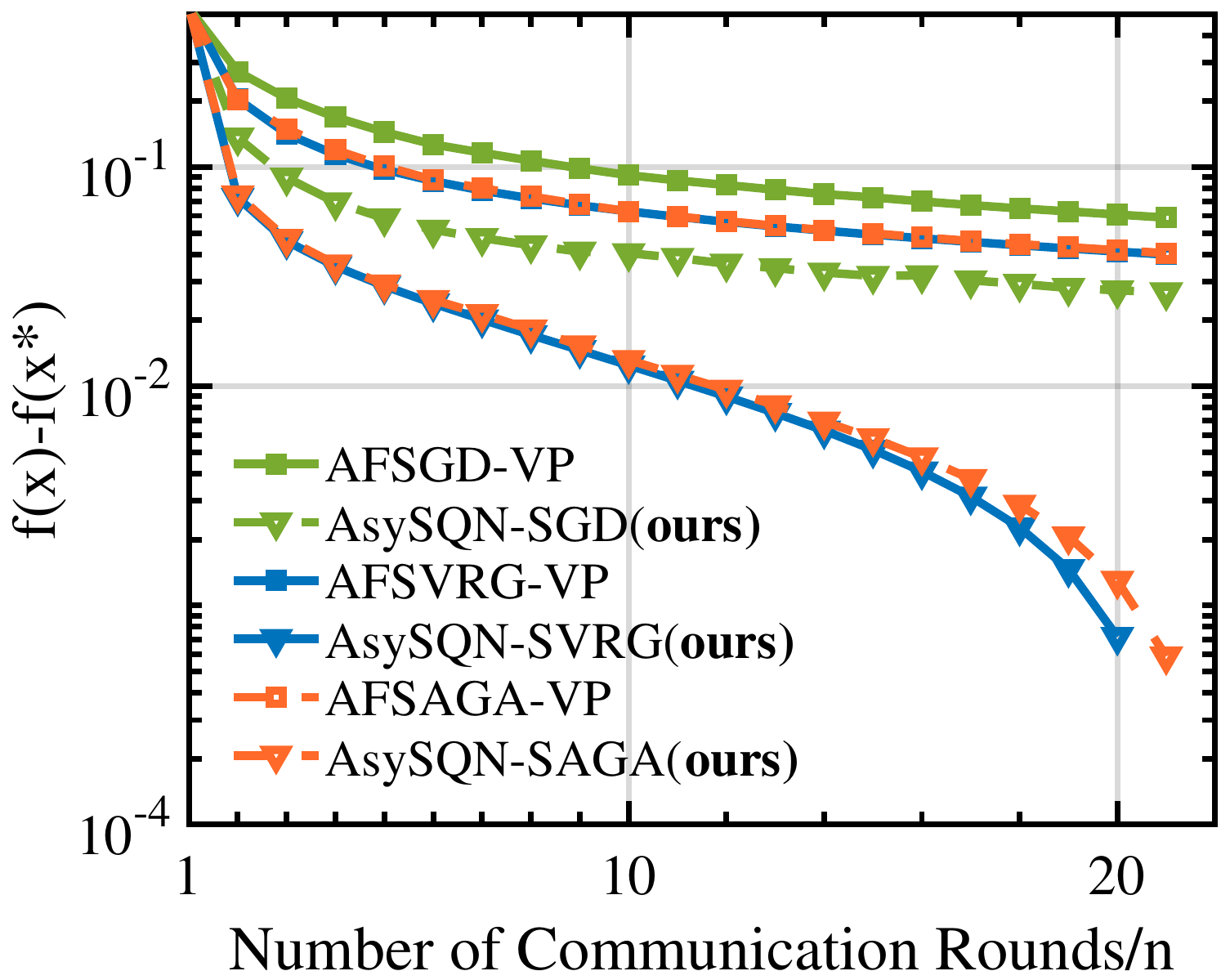}
			\caption{Results on $D_4$}
		\end{subfigure}%
		
		\begin{subfigure}{0.22\linewidth}
			\includegraphics[width=\linewidth]{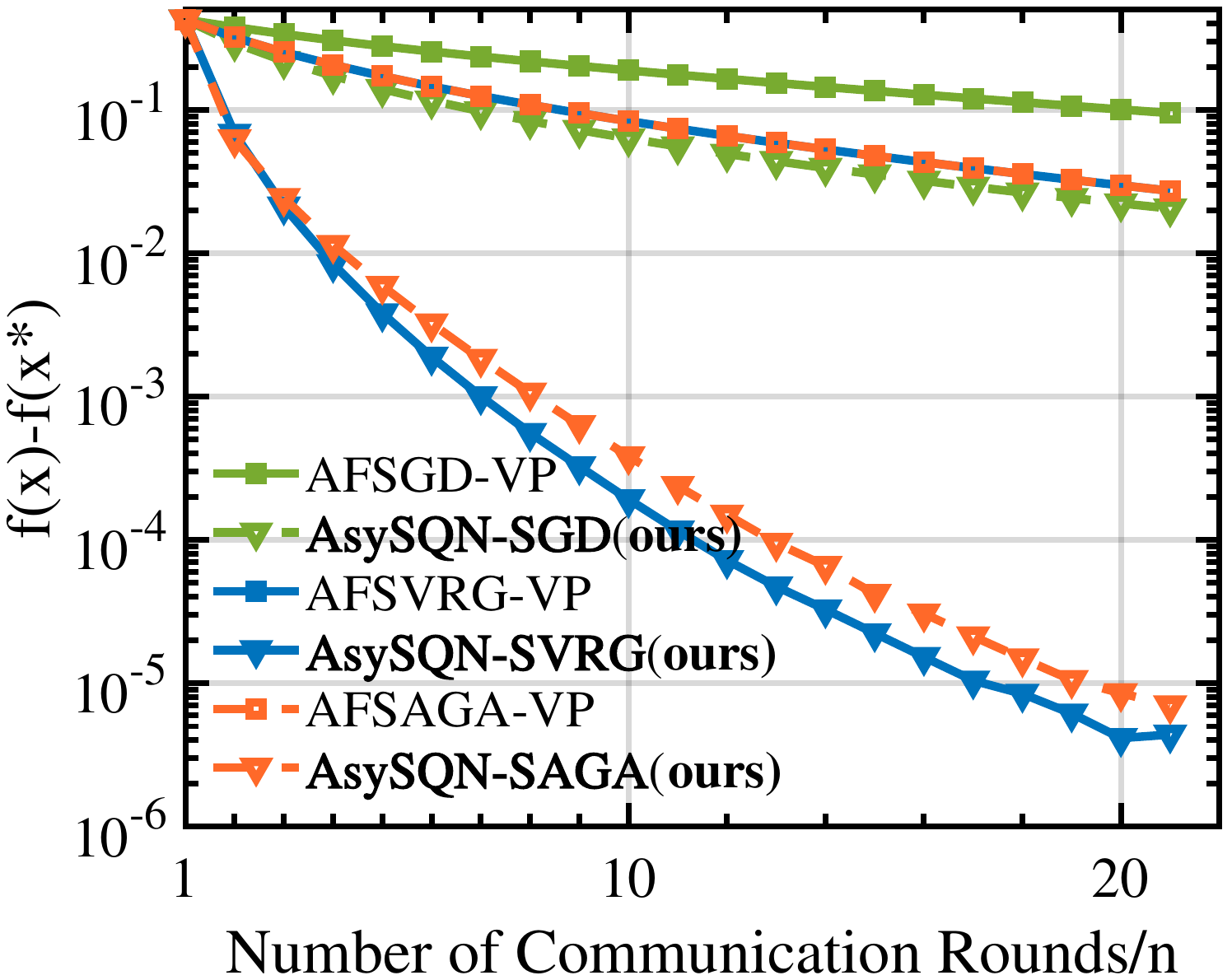}
			\caption{Results on $D_5$}
		\end{subfigure}%
		\quad 	\begin{subfigure}{0.22\linewidth}
			\includegraphics[width=\linewidth]{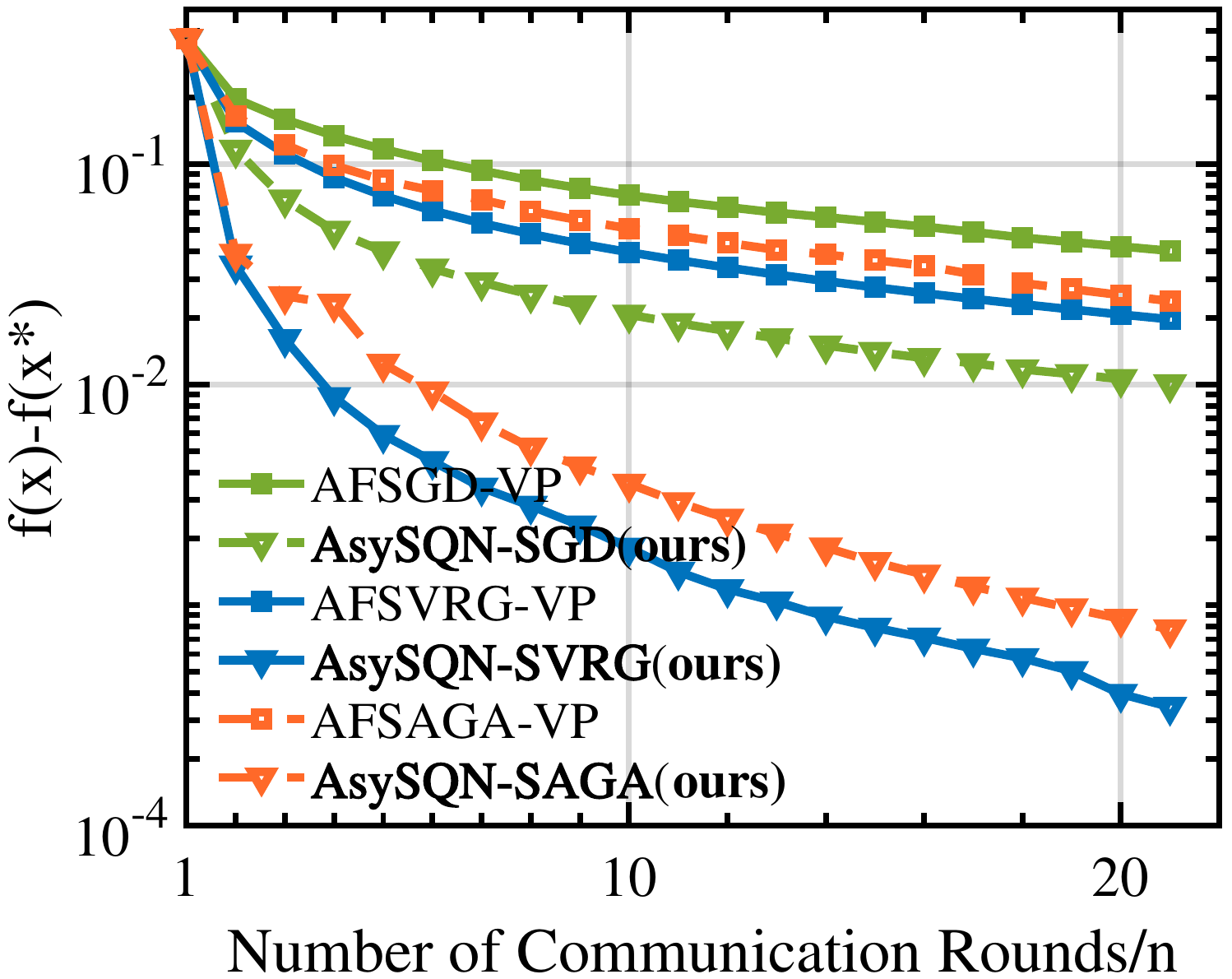}
			\caption{Results on $D_6$}
		\end{subfigure}%
		\quad	\ \begin{subfigure}{0.22\linewidth}
			\includegraphics[width=\linewidth]{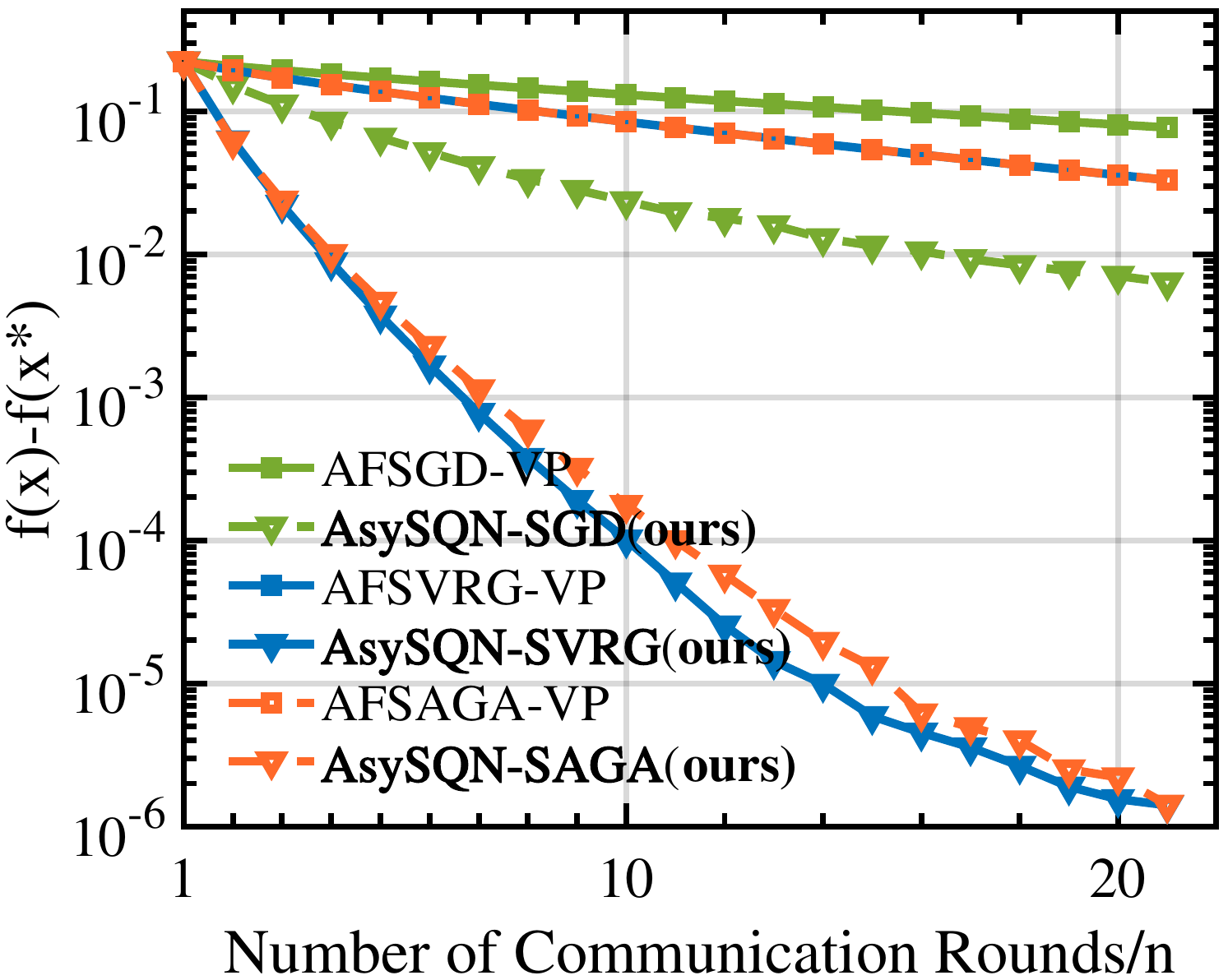}
			\caption{Results on $D_7$}
		\end{subfigure}%
		\quad \begin{subfigure}{0.22\linewidth}
			\includegraphics[width=\linewidth]{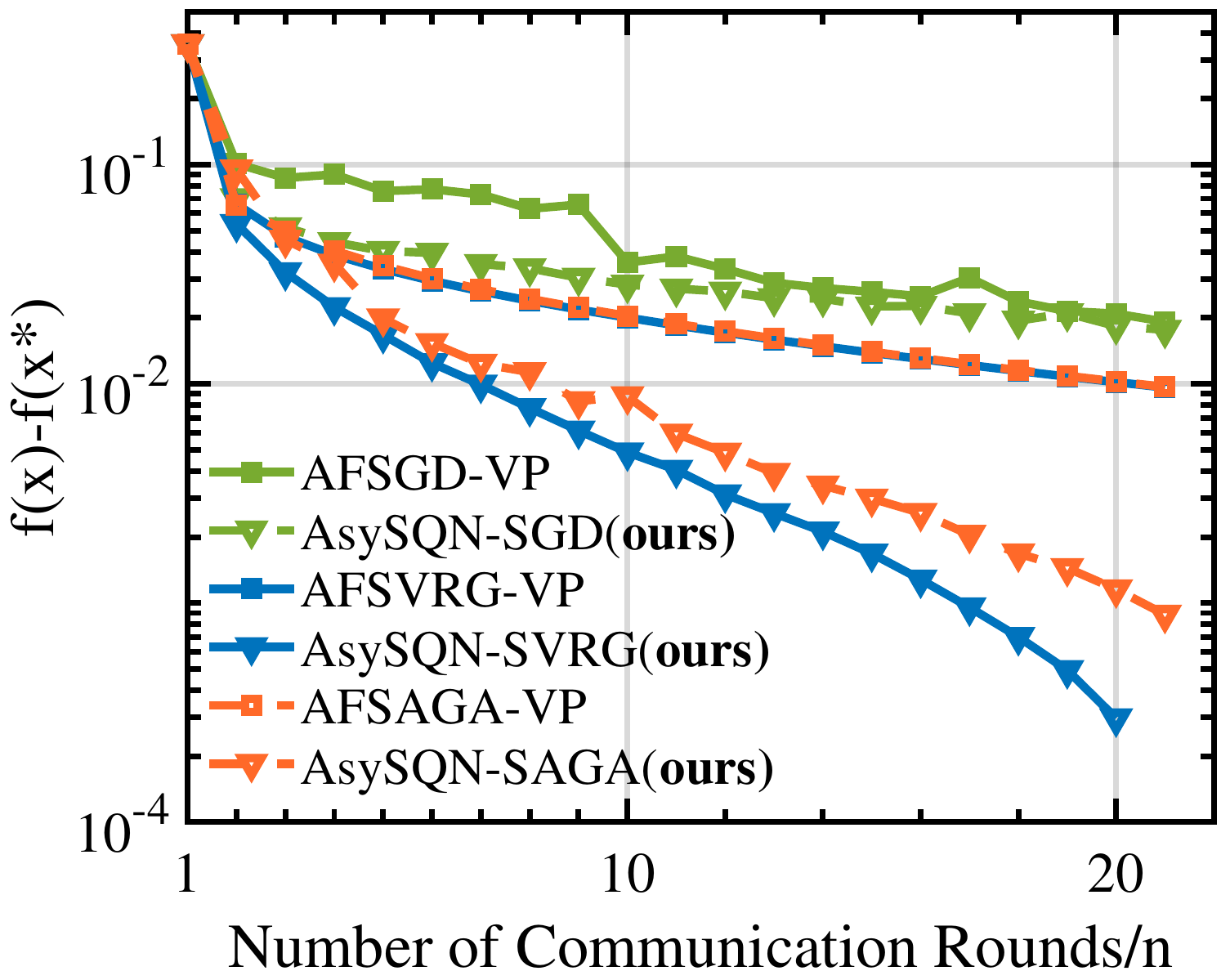}
			\caption{Results on $D_8$}
		\end{subfigure}%
		\caption{Sub-optimality v.s. the NCR on all datasets for solving $\mu$-strongly convex VFL problems.}
		\label{fig-epoch}
	\end{figure*}

	\noindent{\bf Problem for Evaluation:}
	In this paper, we use the popular $\ell_2$-norm regularized logistic regression problem for evaluation.
	\begin{equation}\label{P1}
		\min_{w \in \mathbb{R}^d} f(w):=\frac{1}{n} \sum_{i=1}^{n}  {\text{log}}(1+e^{-y_iw^{\mathrm{T}} x_i}) + \frac{\lambda}{2} \|w\|^2,
	\end{equation}
	where $\lambda$ is set as $1e^{-4}$ for all experiments.

{\begin{table*}[!t]
		\centering
		\caption{Results of CTI on different datasets, which are obtained during the training process of 21$n$ iterations (10 trials).}
		\label{table:communication}
		\setlength{\tabcolsep}{0.75mm}{{
			\scalebox{0.96}{	\small \begin{tabular}{@{}cccccccccc@{}}
					\toprule
					& & $D_1$($d_\ell=12$) & $D_2$($d_\ell=12$) & $D_3$($d_\ell=169,398$) & $D_4$($d_\ell=37$)
					& $D_5$($d_\ell=5,904$) & $D_6$($d_\ell=16$) & $D_7$($d_\ell=250$) & $D_8$($d_\ell=98$)
					\\ \midrule
					&CTI &1.057 & 1.059 & 89.265 & 1.187
					& 5.688 & 1.083 & 1.771 & 1.333 \\
					\bottomrule
		\end{tabular}}}}
	\end{table*}}
	
	\begin{figure*}[!t]
		\centering
		\ \begin{subfigure}{0.22\linewidth}
			\includegraphics[width=\linewidth]{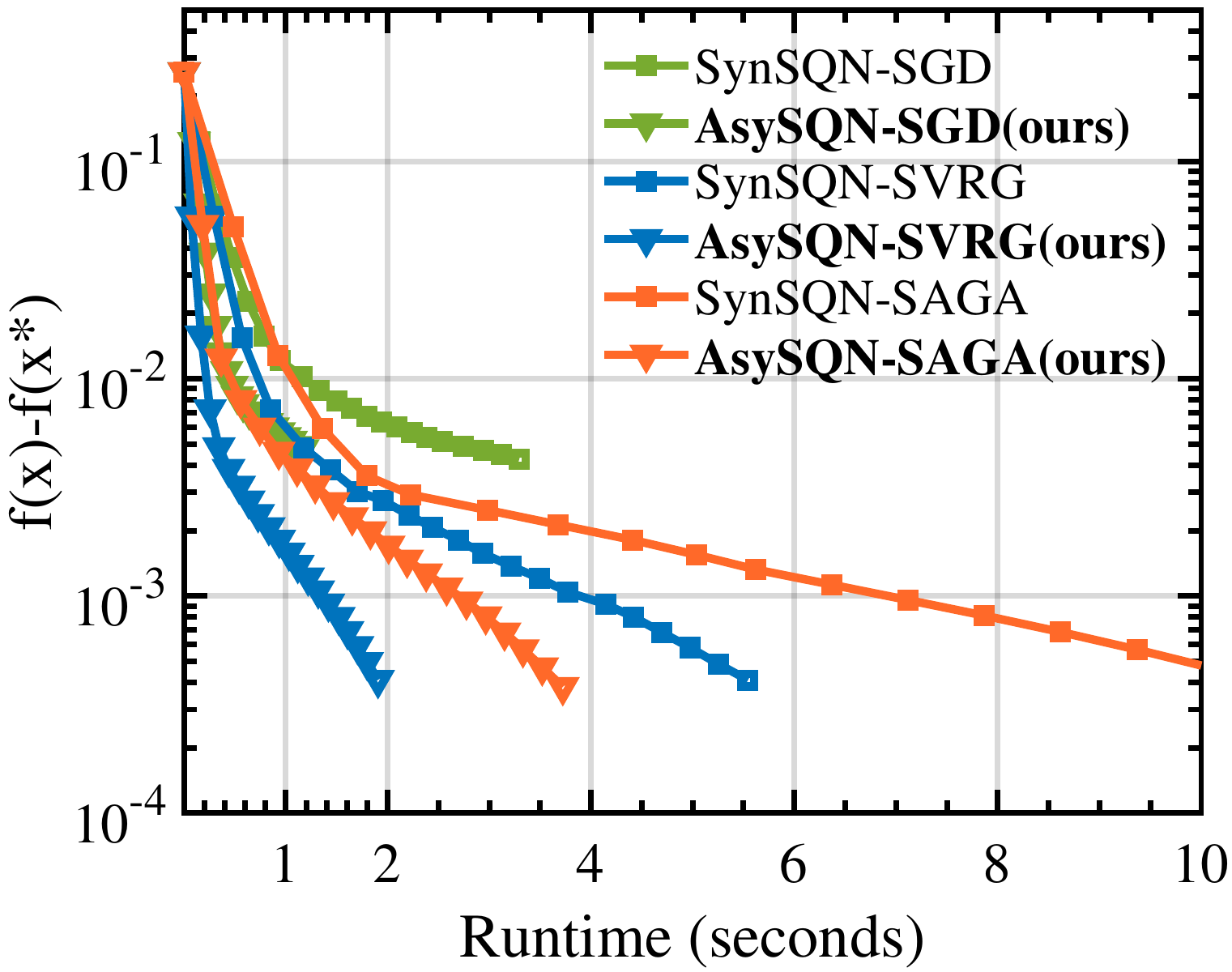}
			\caption{Results on $D_1$}
		\end{subfigure}
		\quad \begin{subfigure}{0.22\linewidth}
			\includegraphics[width=\linewidth]{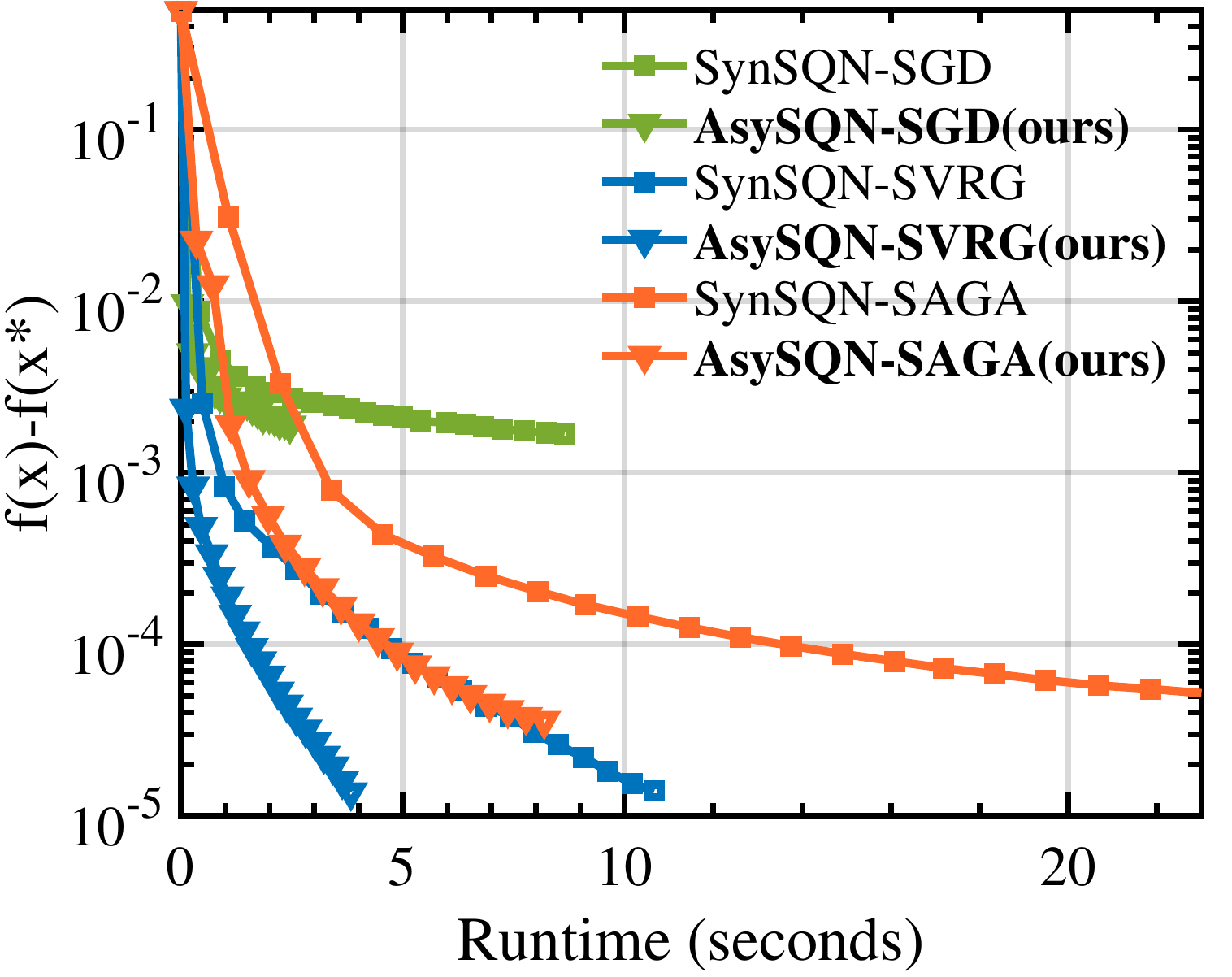}
			\caption{Results on $D_2$}
		\end{subfigure}
		\quad \begin{subfigure}{0.22\linewidth}
			\includegraphics[width=\linewidth]{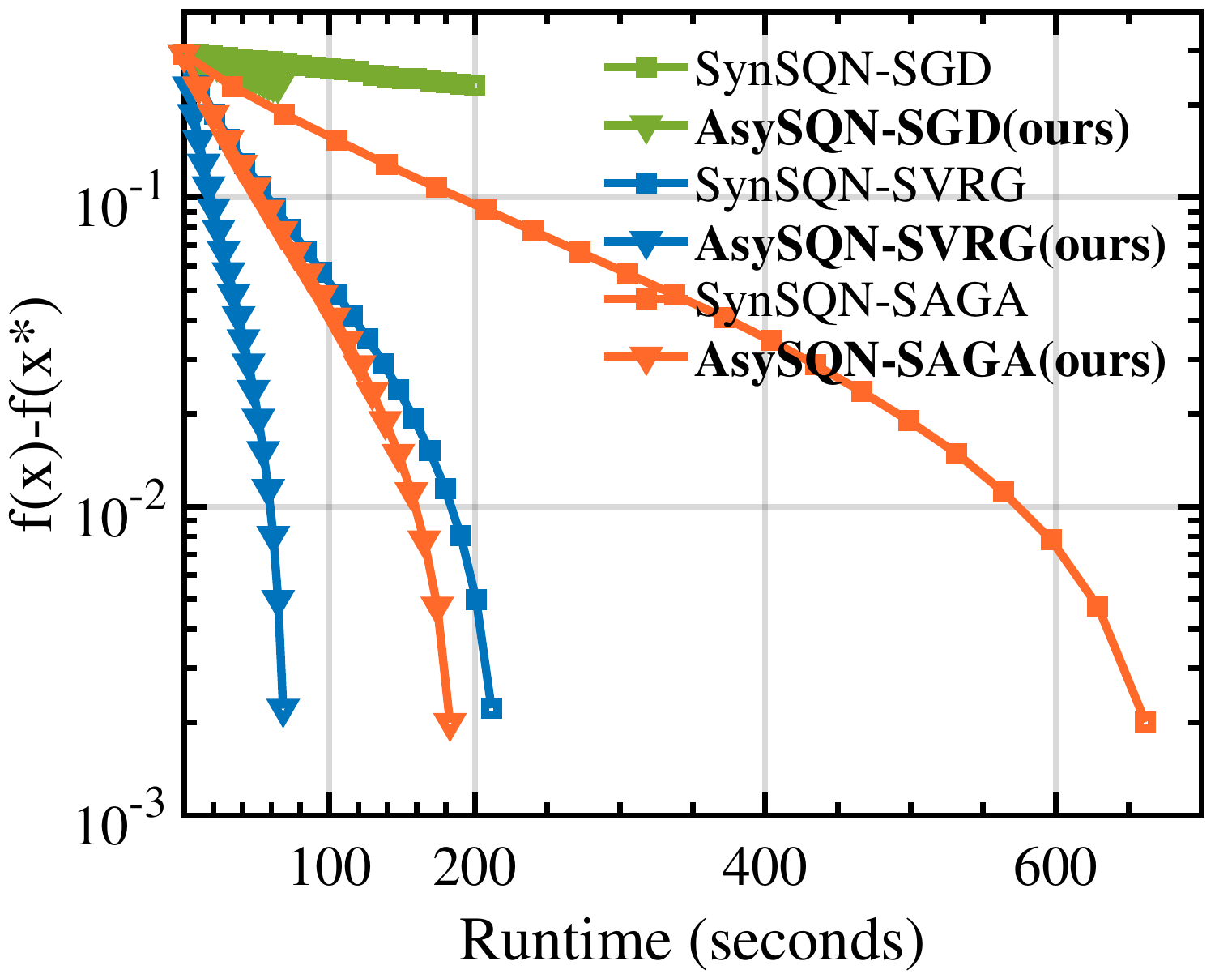}
			\caption{Results on $D_3$}
		\end{subfigure}%
		\quad \begin{subfigure}{0.22\linewidth}
			\includegraphics[width=\linewidth]{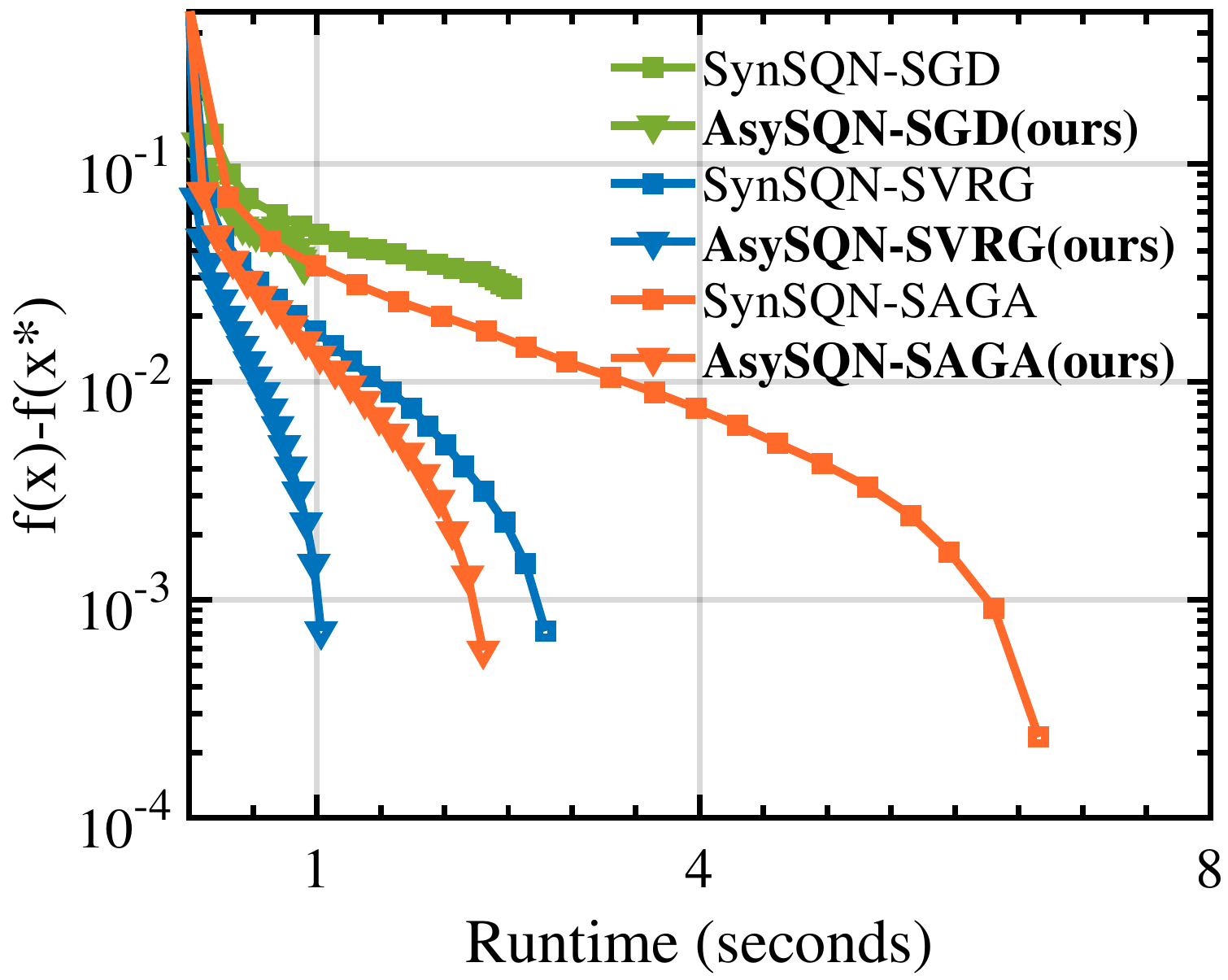}
			\caption{Results on $D_4$}
		\end{subfigure}%
		
		\begin{subfigure}{0.22\linewidth}
			\includegraphics[width=\linewidth]{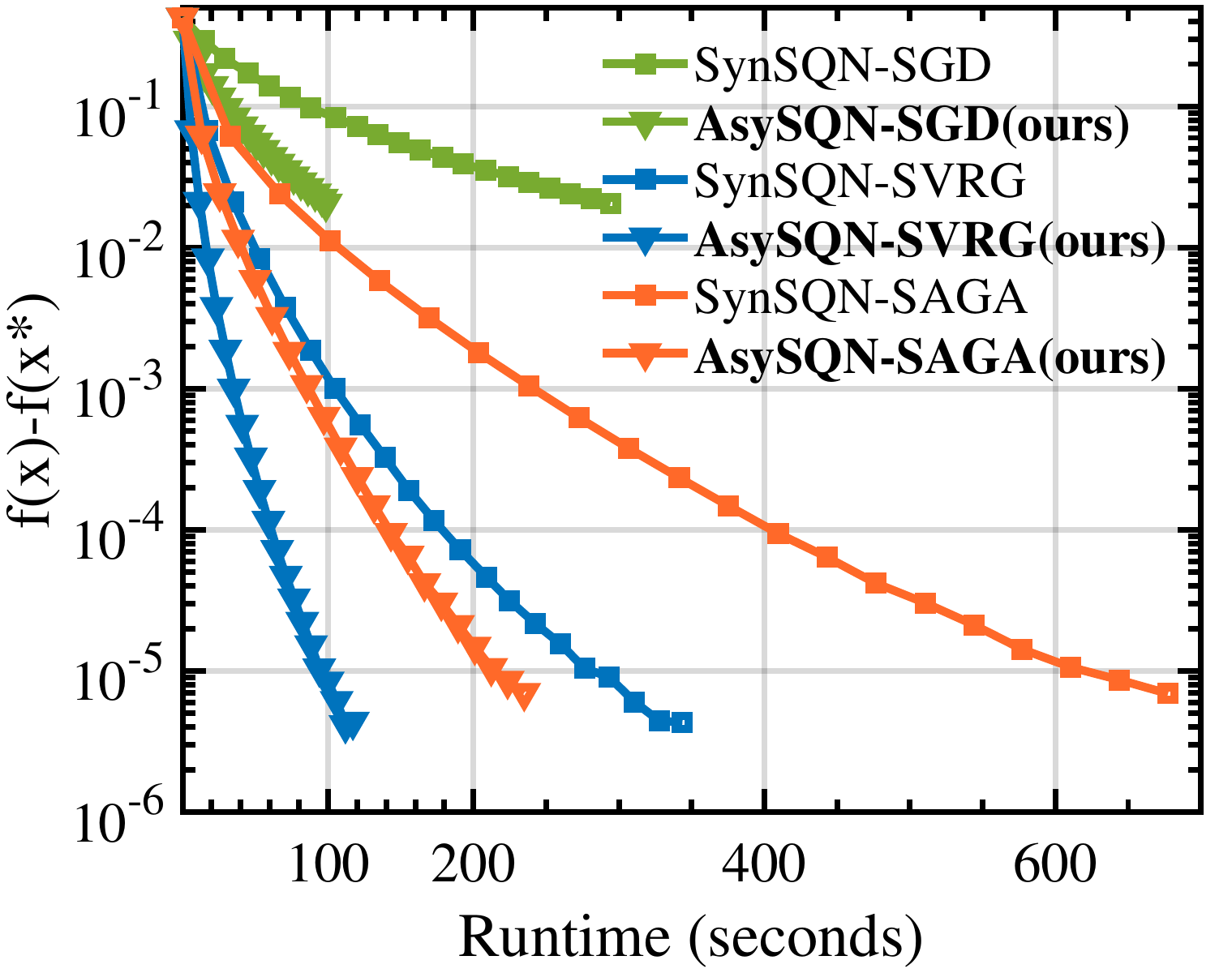}
			\caption{Results on $D_5$}
		\end{subfigure}%
		\quad 	\begin{subfigure}{0.22\linewidth}
			\includegraphics[width=\linewidth]{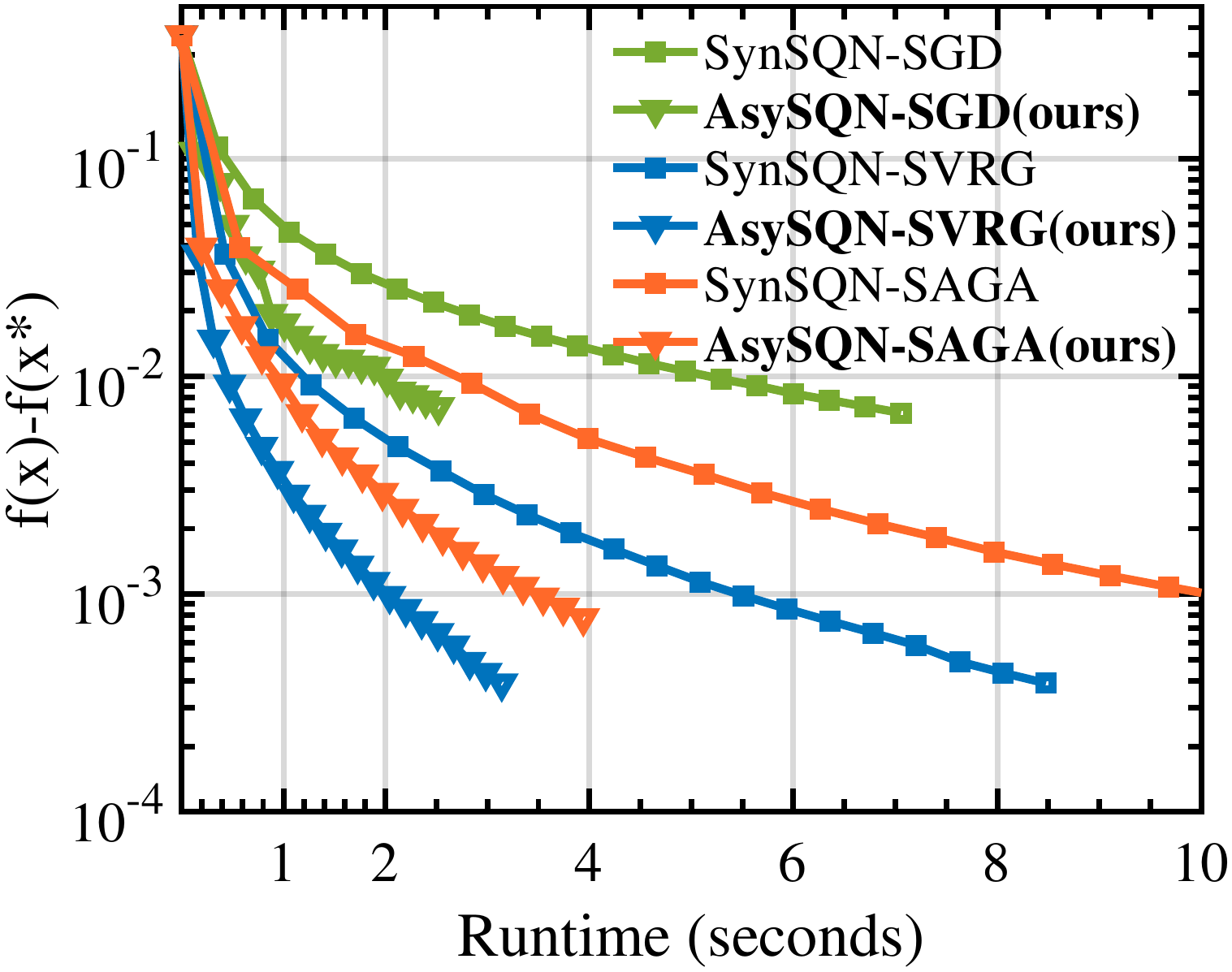}
			\caption{Results on $D_6$}
		\end{subfigure}%
		\quad	\ \begin{subfigure}{0.22\linewidth}
			\includegraphics[width=\linewidth]{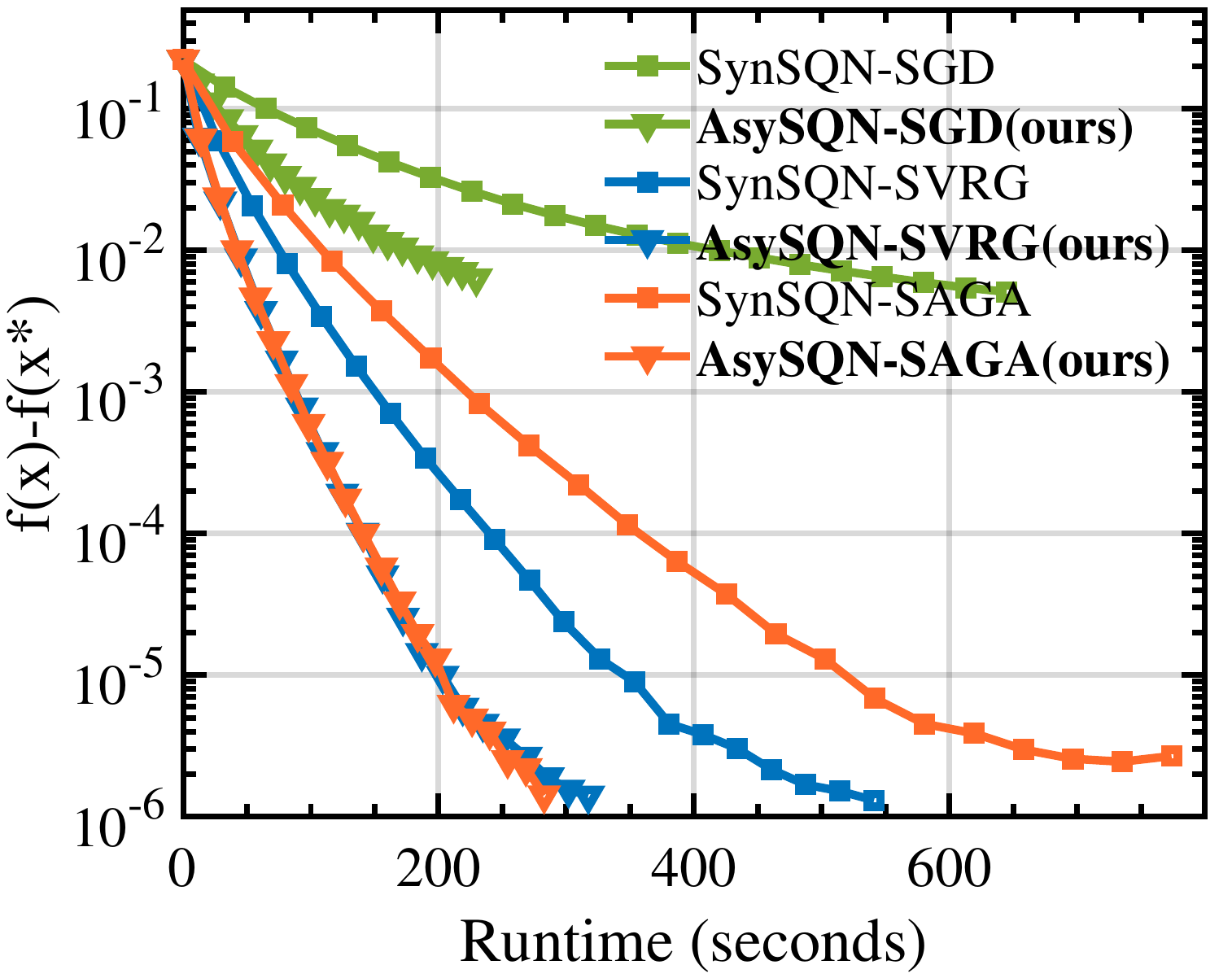}
			\caption{Results on $D_7$}
		\end{subfigure}%
		\quad	\ \begin{subfigure}{0.22\linewidth}
			\includegraphics[width=\linewidth]{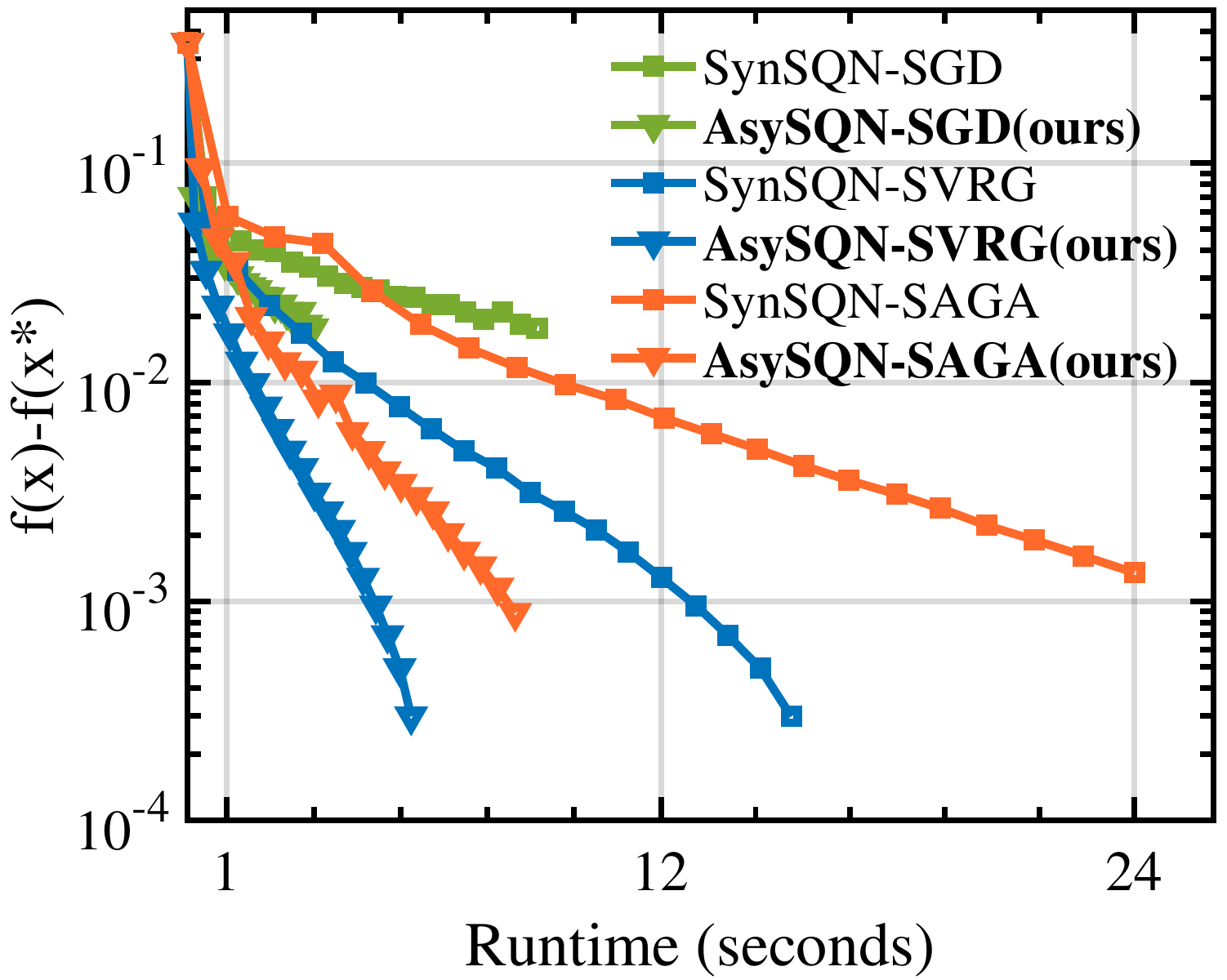}
			\caption{Results on $D_8$}
		\end{subfigure}%
		\caption{Sub-optimality v.s. training time on all datasets for solving $\mu$-strongly convex VFL problems.}
		\label{fig-trainingtime}
	\end{figure*}
	
	\begin{table*}[!t]
		\centering
		\caption{Accuracy of different algorithms to evaluate the losslessness  of our algorithms (10 trials).}
		\label{exp-lossless}
		\setlength{\tabcolsep}{1.4mm}{\small{\begin{tabular}{@{}cccccccccc@{}}
			\toprule
			&Algorithm& $D_1$(\%) & $D_2$(\%) & $D_3$(\%) & $D_4$(\%)
			& $D_5$(\%) & $D_6$(\%) & $D_7$(\%) & $D_8$(\%)
			\\ \midrule
			& NonF & 81.96$\pm$0.02 & 93.56$\pm$0.03 & 98.29$\pm$0.02 & 90.21$\pm$0.02
			& 96.02$\pm$0.03 & 85.03$\pm$0.02 & 87.43$\pm$0.04 & 87.01$\pm$0.04 \\
			&{\bf{ Ours}} & 81.96$\pm$0.03 & 93.56$\pm$0.06 & 98.29$\pm$0.03 & 90.21$\pm$0.03
			& 96.02$\pm$0.03 & 85.03$\pm$0.04 & 87.43$\pm$0.06 & 87.01$\pm$0.05 \\
			\bottomrule
		\end{tabular}}}
	\end{table*}

	\begin{table}
		\centering
		\caption{Improvements of CRU on all datasets.}
		\setlength{\tabcolsep}{1.3mm}{
			\begin{tabular}{@{}ccccccccc@{}}
				\toprule
				& $D_1$ & $D_2$ & $D_3$ &$D_4$ & $D_5$ & $D_6$ & $D_7$ & $D_8$ \\
				\midrule
				SGD-based  & 2.90 & 2.77 & 2.98 &2.84 & 2.82 & 2.80 & 2.75 &2.86\\
				SVRG-based & 2.96 & 2.80 & 3.01 &2.95 & 2.86. & 2.72 & 2.84 & 2.88 \\
				SAGA-based & 2.99 & 2.92 & 3.09 &2.96 & 2.91 & 2.79 & 2.74 & 2.92 \\ \bottomrule
		\end{tabular}}
		\label{CRU}
	\end{table}
	
	\begin{figure}
		\centering
		\begin{subfigure}{0.3\linewidth}
			\includegraphics[width=\linewidth]{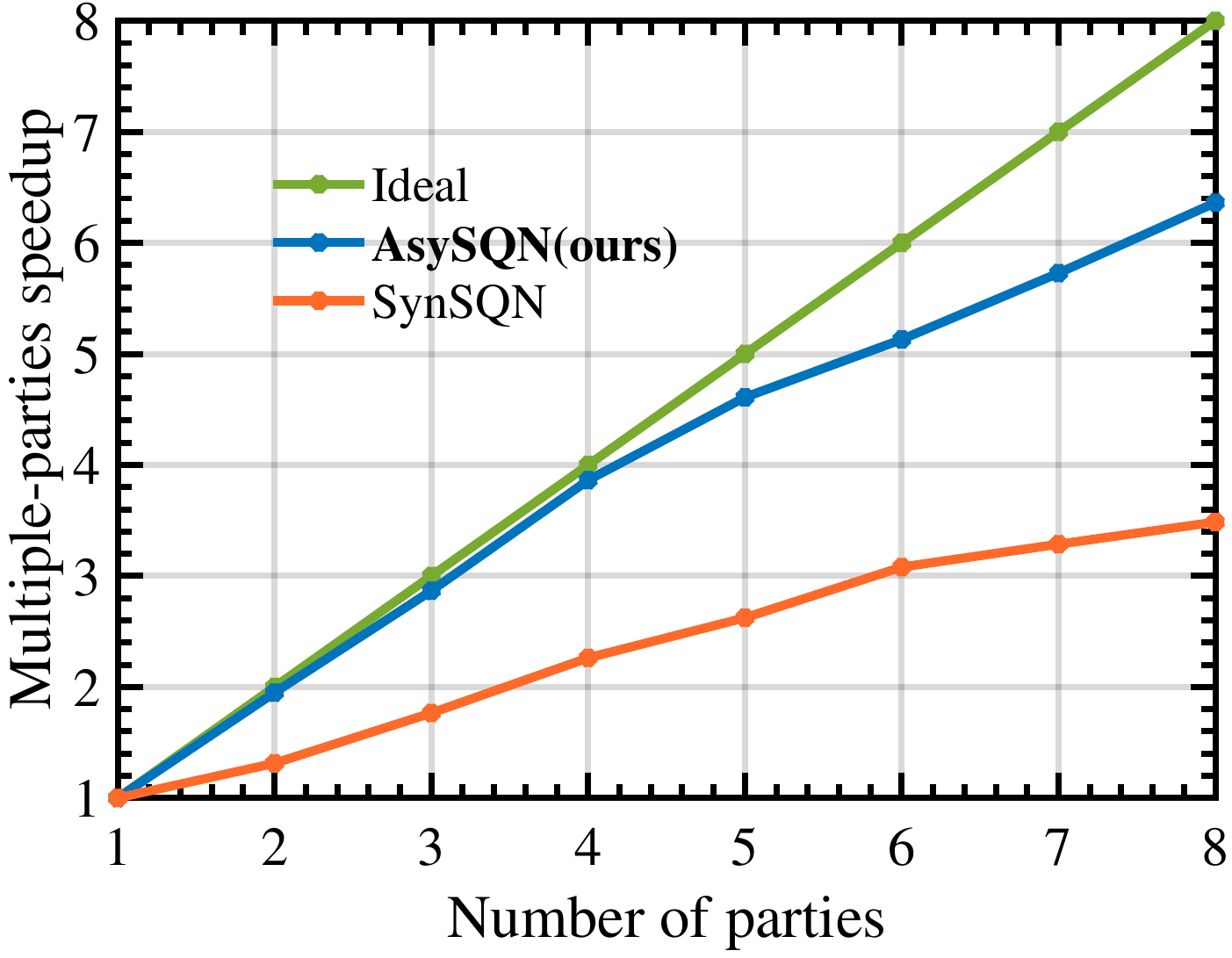}
			\caption{SAGA-based}
		\end{subfigure}
		\quad
		\begin{subfigure}{0.3\linewidth}
			\includegraphics[width=\linewidth]{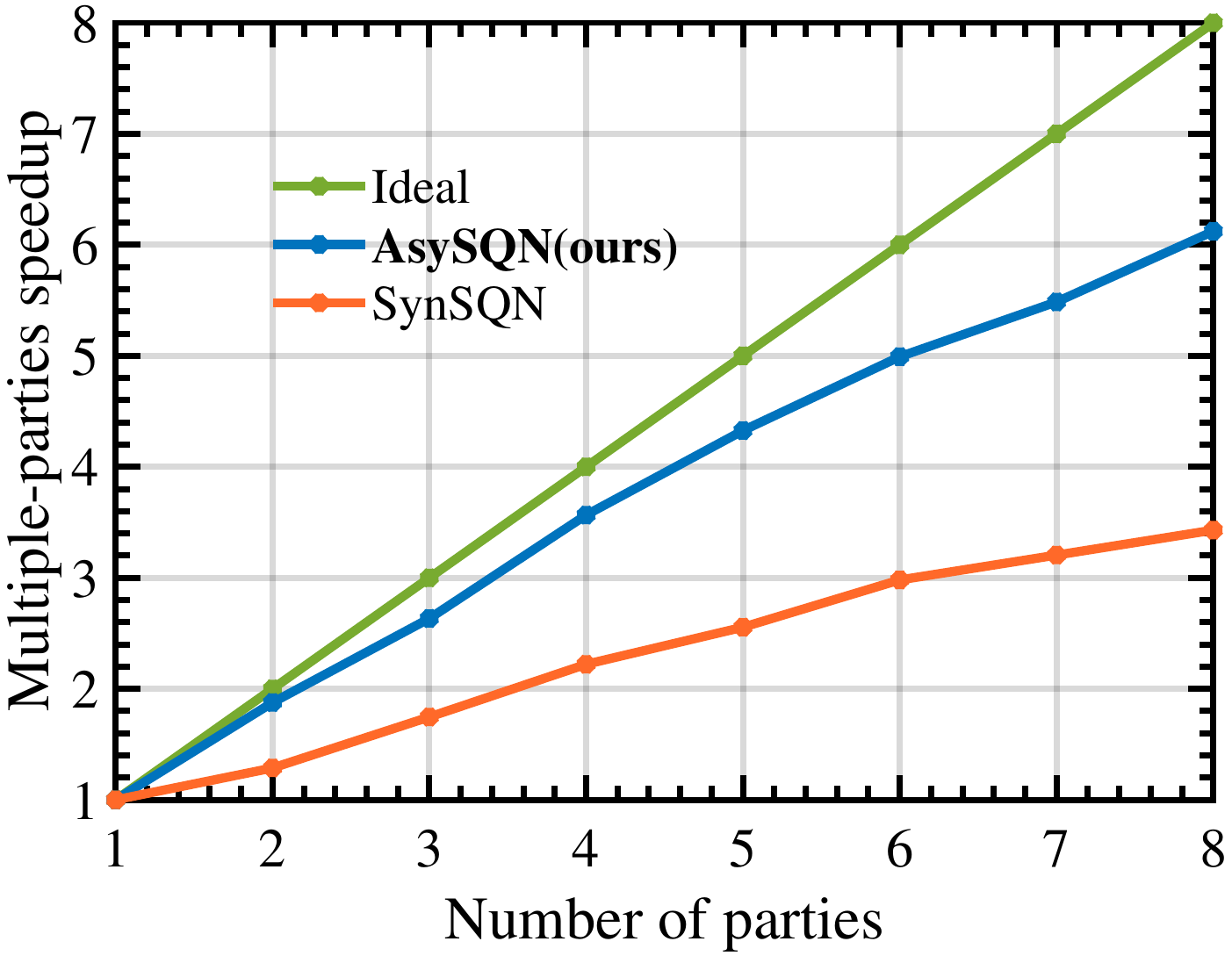}
			\caption{SVRG-based}
		\end{subfigure}
		\quad
		\begin{subfigure}{0.3\linewidth}
			\includegraphics[width=\linewidth]{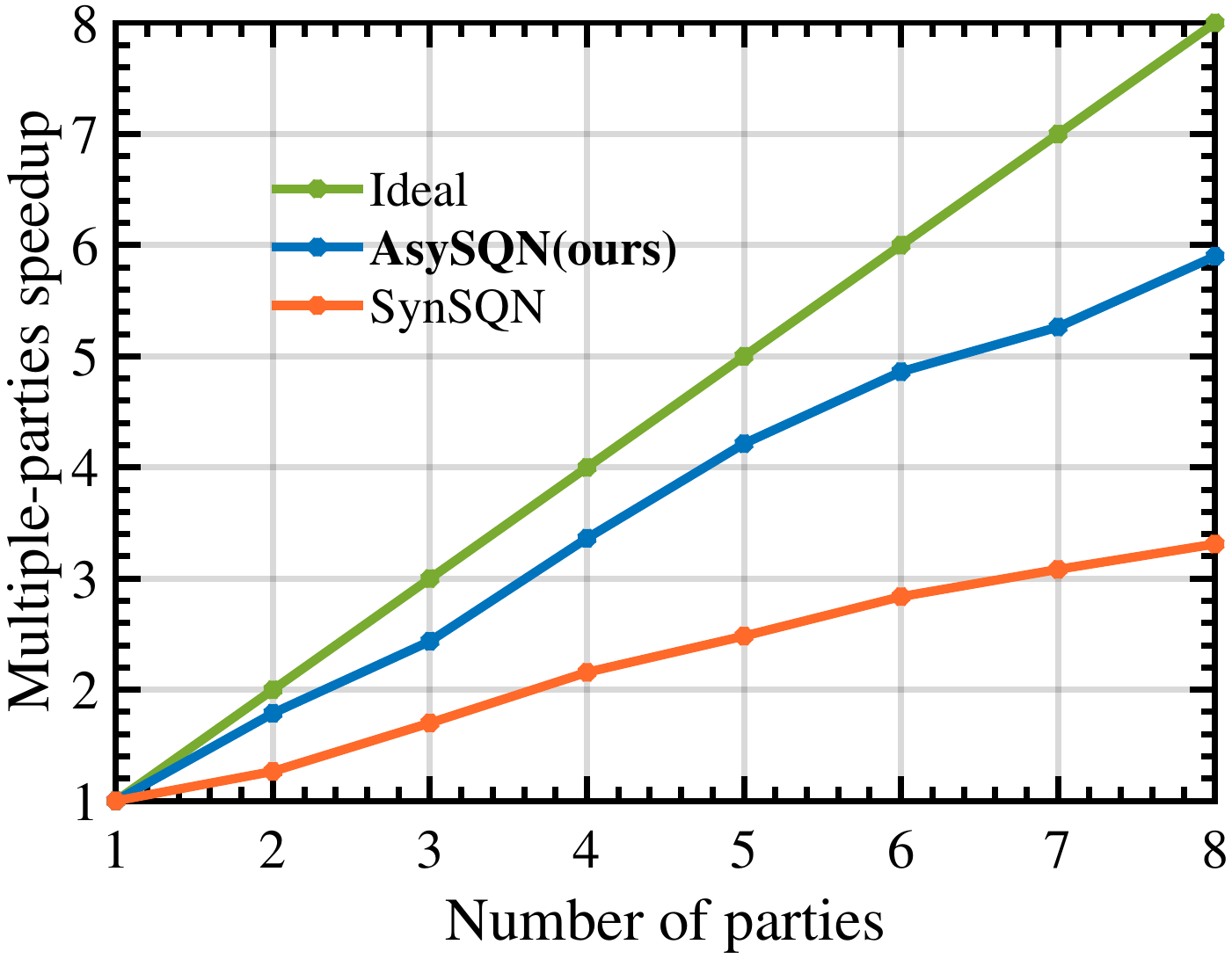}
			\caption{SGD-based}
		\end{subfigure}%
		\setlength{\abovecaptionskip}{0.2cm}
		\caption{Multiple-workers speedup scalability on $D_7$.}
		\label{Exp2}
	\end{figure}
	
	\noindent{\bf Datasets:}
	We use eight datasets for evaluation, which are summarized in Table~\ref{dataset}.
	Among them, $D_1$ (UCICreditCard) and $D_2$ (GiveMeSomeCredit) are from the Kaggle\footnote{https://www.kaggle.com/datasets}, and $D_3$ (news20), $D_4$(w8a), $D_5$ (rcv1), $D_6$ (a9a), $D_7$ (epsilon) and $D_8$ (mnist) are from the LIBSVM\footnote{https://www.csie.ntu.edu.tw/~cjlin/libsvmtools/datasets/}. Especially, $D_1$ and $D_2$ are the financial datasets, which are used to demonstrate the ability to address real applications.
	Following previous works, we apply one-hot encoding to categorical features of $D_1$ and $D_2$ , thus the number of features become 90 and 92, respectively, for $D_1$ and $D_2$.
	
	\subsection{Evaluation of Lower Communication Cost}
	We demonstrate that our proposed AsySQN-type algorithms have the lower communication costs by showing that AsySQN reduces the NCR and has a low per-round communication overhead.
	
	\noindent{\bf Reducing the Number of Communication Rounds:} To demonstrate that our AsySQN-type algorithms can significantly reduce the NCR, we compare them with the corresponding asynchronous stochastic first-order methods for VFL, e.g., compare AsySQN-SGD with AFSGD-VP \cite{gu2020Privacy}. The  experimental results are presented in Fig.~\ref{fig-epoch}. As depicted in Fig.~\ref{fig-epoch}, to achieve the same sub-optimality our AsySQN-type algorithms  needs much lower NCR than the corresponding first-order algorithms, which is consistent to our claim that AsySQN-type algorithms incorporated with approximate second-order information can dramatically reduce the number of communication rounds in practice.

	\noindent{\bf Low Per-Round Communication Overhead:} Note that, the quasi-Newton (QN)-based framework has already been proposed  in \cite{yang2019quasi}, which, however, is not communication-efficient due to large per-round communication overhead of transmitting the gradient difference. It is unnecessary to compare  AsySQN framework  with that QN-based because they have totally different structures. But it is necessary to demonstrate that our AsySQN is more communication-efficient than that transmitting the gradient difference. Thus, we compare the time spending of transmitting just the $\theta^\ell$ with that of transmitting the gradient (difference).  The results of communication time improvement (CTI) are presented in Table \ref{table:communication}, where
	\begin{equation}
		\text{CTI}= \frac{\text{$CT_1$ on transmitting a vector} \in \mathbb{R}^{d_\ell}}{\text{$CT_1$  on transmitting a scalar}}, \nonumber
	\end{equation}
	where $CT_1$ denotes communication time.  All experiments are implemented with $q=8$ parties and  the results are obtained in 10 trials.
	The results in Table \ref{table:communication} demonstrate that AsySQN indeed has a lower per-round communication overhead than transmitting the gradient. Specifically, when $d_\ell$ is small, the intrinsic communication time spending (\ie, time spending on transmitting nothing) dominates, thus the CRI is not significant. Importantly, when $d_\ell$ is significantly large (\emph{e.g.}, $D_5$, and even $D_3$) the CRI is very remarkable.

	\subsection{Evaluation of Better CRU}
	To demonstrate that our AsySQN-type algorithms have a better CRU, we compare them with the corresponding synchronous algorithms (SynSQN-type). We use CRUI to denote the CRU improvement of our asynchronous algorithms relative to the corresponding synchronous algorithms
	\begin{equation}\label{utilization}
		\text{CRUI}= \frac{\text{$CT_2$ of SynSQN-type algorithm}}
		{\text{$CT_2$ of AsySQN-type algorithm}}, \nonumber
	\end{equation}
	where $CT_2$ (computation time) means the overall training time subtracts the communication time during a fixed number of iterations, \ie, $21n$ for each datasets in our experiments. The results of CRUI are summarized in Table \ref{CRU}. As shown in  Table \ref{CRU}, our asynchronous algorithms have much better CRU than the corresponding synchronous algorithms for real-world VFL systems with parties owning unbalanced computation resources. Moreover, the more unbalanced computation resources the slowest and the fastest parties have the higher CRUI will be.
	\subsection{Evaluation of Training Efficiency and Scalability}
	To directly show the efficiency for training VFL models, we compared them with the corresponding synchronous ones and depict the loss v.s. training time curves in Fig.~\ref{fig-trainingtime}.
	
	We also consider the multiple-workers speedup scalability in terms of the number of parties and report the results in Fig.~\ref{Exp2}. Given $q$ parties, there is
	\[\text{$q$-workers speedup} =\frac{\text{TT for the serial computation}}{\text{TT of using $q$ parties}},\]
	where the training time (TT) is the time spending on reaching a certain precision of sub-optimality, e.g., $5e^{-5}$ for $D_6$. As shown in Fig.~\ref{Exp2}, asynchronous algorithms have a much better multiple-parties speedup scalability than synchronous ones and can achieve near-linear speedup.
	
	\subsection{Evaluation of Losslessness}
	To demonstrate the losslessness of our algorithms, we compare AsySQN-type algorithms with its non-federated (NonF) counterparts (the only difference to the AsySQN-type algorithms is that all data are integrated together for modeling). For datasets without testing data, we split the data set into $10$ parts, and use one of them for testing. Each comparison is repeated 10 times with $q=8$, and a same stop criterion, \emph{e.g.}, $5e^{-5}$ for $D_6$. As shown  in Table~\ref{exp-lossless}, the accuracy of our algorithms are the same with those of NonF algorithms, which demonstrate that our VFL algorithms are lossless.
	
	\section{Conclusion}
	In this paper, we proposed a novel AsySQN framework for the VFL applications to industry, where communication costs between different parties (\emph{e.g.}, different corporations, companies and organizations) are expensive and different parties owning unbalanced computation resources. Our AsySQN framework with slight per-round communication overhead utilizes approximate second-order information can dramatically reduces the number of communication rounds, and thus has lower communication cost. Moreover, AsySQN enables parties with unbalanced computation resources asynchronously update the model, which can achieve better computation resource utilization. Three SGD-type algorithms with different stochastic gradient estimators were also proposed under AsySQN, {\ie} AsySQN-SGD, -SVRG, SAGA, with theoretical guarantee for strongly convex problems.
\section{Acknowledgments}
Q.S. Zhang and C. Deng were supported in part by the National Natural Science Foundation of China under Grant 62071361, Key Research and Development Program of Shaanxi under Grant 2021ZDLGY01-03, and the Fundamental Research Funds for the Central Universities ZDRC2102.

	\bibliographystyle{ACM-Reference-Format}
	\bibliography{ref}
	
\newpage
\onecolumn
\appendix	
In the Appendix, we prove that given Algorithm 1 and Assumption 3, $H_k$ generated is positive semidefinite even that the infromation used for computing $H_k$ is stale. The proof of three theorems are presented in the next sections. Moreover, we also give the detailed convergence analyses of Theorems\ref{thm-sgdconvex} to \ref{thm-sagaconvex}. Given global iteration number $u$, $\xi(u)$ (or $\xi(u),\ell$) denotes its local iteration number on party $\psi(u)$ or $\ell$. And given local iteration number $k$, $\xi^{-1}(k)$ denotes its corresponding global iteration number. Note that, in the sequel, given global iteration number $u$, $H_{\xi(u)}$ denotes $H_{\xi(u)}^{\psi(u)}$ for notation abbreviation.
\section*{Appendix A: Proof of Theorem  \ref{thm-sgdconvex}}
Before proving Theorem \ref{thm-sgdconvex}, we first provide several basic inequalities in Lemma \ref{AsySPSAGA_lemma3}.
\begin{lemma} \label{AsySPSAGA_lemma3} For  AFSGD-VP, under Assumptions \ref{assum1}.1 and \ref{assum5},   we have that
\begin{eqnarray}\label{AsySPSAGA_lem3_0}
 \sum_{u\in K(t)} \| \nabla_{\mathcal{G}_{\psi(u)}} f(w_u) \|^2
 \geq \frac{1}{2} \sum_{u\in K(t)} \| \nabla_{\mathcal{G}_{\psi(u)}} f(w_t) \|^2 - \eta_1 \gamma^2 L^2  \sum_{u\in K(t)}   \sum_{v \in \{t,\ldots,u \}}\|  \widehat{v}^{\psi(v)}_v \|^2
\end{eqnarray}
\end{lemma}
\begin{proof} For any global iteration number $u \in K(t)$, we have that
\begin{eqnarray}\label{eqLem1_1}
&&  \| \nabla_{\mathcal{G}_{\psi(u)}} f(w_t) \|^2
\\ \nonumber &\leq& \left ( \| \nabla_{\mathcal{G}_{\psi(u)}} f(w_t) - \nabla_{\mathcal{G}_{\psi(u)}} f(w_u) \| +  \| \nabla_{\mathcal{G}_{\psi(u)}} f(w_u) \| \right)^2
\\ \nonumber &\leq& 2 \| \nabla_{\mathcal{G}_{\psi(u)}} f(w_t) - \nabla_{\mathcal{G}_{\psi(u)}} f(w_u) \|^2 +  2 \| \nabla_{\mathcal{G}_{\psi(u)}} f(w_u) \|^2
\\ \nonumber &\leq& 2  \| \nabla f(w_t) - \nabla f(w_u) \|^2
+  2 \| \nabla_{\mathcal{G}_{\psi(u)}} f(w_u) \|^2
\\ \nonumber &\stackrel{ (a) }{\leq} & 2 L^2 \| w_t - w_u \|^2
+  2 \| \nabla_{\mathcal{G}_{\psi(u)}} f(w_u) \|^2
\\ \nonumber &=&
2 L^2 \gamma^2 \| \sum_{v \in \{t,\ldots,u \}} \textbf{U}_{\psi(v)} H_{\xi(v)}\widehat{v}^{\psi(v)}_v  \|^2
+  2 \| \nabla_{\mathcal{G}_{\psi(u)}} f(w_u) \|^2
\\ \nonumber &\leq &
2 \eta_1 \sigma_2^2L^2 \gamma^2 \sum_{v \in \{t,\ldots,u \}} \|\widehat{v}^{\psi(v)}_v \|^2
+  2 \| \nabla_{\mathcal{G}_{\psi(u)}} f(w_u) \|^2
\end{eqnarray}
where the inequality (a) uses Assumption \ref{assum2}.1,  the last inequality uses Assumption \ref{assum3}.2
According to (\ref{eqLem1_1}), we have that
\begin{eqnarray}\label{eqLem1_2}
\| \nabla_{\mathcal{G}_{\psi(u)}} f(w_u) \|^2 \geq \frac{1}{2}\| \nabla_{\mathcal{G}_{\psi(u)}} f(w_t) \|^2
- \eta_1 \sigma_2^2L^2 \gamma^2 \sum_{v \in \{t,\ldots,u \}} \|  \widehat{v}^{\psi(v)}_v \|^2
\end{eqnarray}

Summing (\ref{eqLem1_2}) over all $u \in K(t)$, we obtain the conclusion.
This completes the proof.
\end{proof}
Based on the basic inequalities in Lemma \ref{AsySPSAGA_lemma3}, we provide the proof of Theorem \ref{thm-sgdconvex} in the following.
\newtheorem{refproof}{Proof}

\begin{proof}{\textbf{of Theorem \ref{thm-sgdconvex}:}}  For global iteration number $u \in K(t)$, we have that
\begin{eqnarray}\label{EqThm1_1}
&& \mathbb{E} f (w_{u+1})
\\ \nonumber
&\stackrel{ (a) }{\leq}&  \mathbb{E} \left ( f (w_{u}) + \langle \nabla f(w_{u}), w_{u+1}-w_{u}  \rangle + \frac{L_{\psi(u)}}{2} \|w_{u+1}-w_{u}   \|^2  \right )
\\ \nonumber
&=&  \mathbb{E} \left ( f (w_{u}) -  \gamma \langle \nabla f(w_{u}), H_{\xi(u)}\widehat{v}^{\psi(u)}_u  \rangle + \frac{L_{\psi(u)} \gamma^2}{2} \|  H_{\xi(u)}\widehat{v}^{\psi(u)}_u  \|^2  \right )
\\ \nonumber
&=&  \mathbb{E} \left ( f (w_{u}) -  \gamma \langle \nabla f(w_{u}),  H_{\xi(u)} \widehat{v}^{\psi(u)}_u - H_{\xi(u)} \nabla_{\mathcal{G}_{\psi(u)}} f_{\mathcal{I}_u} ({w}_u) + H_{\xi(u)}\nabla_{\mathcal{G}_{\psi(u)}} f_{\mathcal{I}_u} ({w}_u) \rangle \right )
\\ \nonumber
&& +\frac{L_{\psi(u)} \gamma^2}{2} \mathbb{E}\|  H_{\xi(u)}\widehat{v}^{\psi(u)}_u  \|^2
\\ \nonumber
&=&  \mathbb{E}  f (w_{u}) -  \gamma \mathbb{E} \langle \nabla f(w_{u}),  H_{\xi(u)}\nabla_{\mathcal{G}_{\psi(u)}} f ({w}_u) \rangle  + \gamma \mathbb{E} \langle \nabla f(w_{u}),   H_{\xi(u)} (\nabla_{\mathcal{G}_{\psi(u)}} f_{\mathcal{I}_u} ({w}_u) - \nabla_{\mathcal{G}_{\psi(u)}} f_{\mathcal{I}_u} (\widehat{w}_u))  \rangle
 \\ \nonumber
 && + \frac{L_{\psi(u)} \gamma^2}{2} \mathbb{E} \|  H_{\xi(u)}\widehat{v}^{\psi(u)}_u  \|^2
 \\ \nonumber
 &\stackrel{ (b) }{\leq}&  \mathbb{E}  f (w_{u}) -  \gamma \|H_{\xi(u)}\|\mathbb{E} \|  \nabla_{\mathcal{G}_{\psi(u)}} f ({w}_u) \|^2
 + \frac{\gamma\|H_{\xi(u)}\|}{2} \mathbb{E} \| \nabla_{\mathcal{G}_{\psi(u)}} f ({w}_u) \|^2
\\ \nonumber
&&+  \frac{\gamma\|H_{\xi(u)}\|}{2} \mathbb{E} \| \nabla_{\mathcal{G}_{\psi(u)}} f_{\mathcal{I}_u} ({w}_u) - \nabla_{\mathcal{G}_{\psi(u)}} f_{\mathcal{I}_u} (\widehat{w}_u) \|^2
  +  \frac{L_{\psi(u)} \gamma^2\|H_{\xi(u)}\|^2}{2} \mathbb{E} \|  \widehat{v}^{\psi(u)}_u  \|^2
 \\ \nonumber
 &\stackrel{ (c) }{\leq}&  \mathbb{E}  f (w_{u}) -  \frac{\gamma\sigma_1}{2} \mathbb{E} \|  \nabla_{\mathcal{G}_{\psi(u)}} f ({w}_u) \|^2  +  \frac{\gamma \sigma_2L^2}{2} \mathbb{E} \| {w}_u - \widehat{w}_u \|^2
 + \frac{L_{\psi(u)} \gamma^2\sigma_2^2}{2} \mathbb{E} \|  \widehat{v}^{\psi(u)}_u  \|^2
  \\ \nonumber
 &\stackrel{ (d) }{\leq}&  \mathbb{E}  f (w_{u}) -  \frac{\gamma\sigma_1}{2} \mathbb{E} \|  \nabla_{\mathcal{G}_{\psi(u)}} f ({w}_u) \|^2
 +  \frac{\tau_1 \sigma_2 L^2 \gamma^3}{2}  \sum_{u' \in D(u)} \mathbb{E} \|   \textbf{U}_{\psi(u')} H_{\xi(u')}\widehat{v}^{\psi(u')}_{u'} \|^2
 + \frac{L_{\psi(u)} \gamma^2\sigma_2^2}{2} \mathbb{E} \|  \widehat{v}^{\psi(u)}_u  \|^2
 \\ \nonumber
 &\leq&  \mathbb{E}  f (w_{u}) -  \frac{\gamma\sigma_1}{2} \mathbb{E} \|  \nabla_{\mathcal{G}_{\psi(u)}} f ({w}_u) \|^2
 +  \frac{\tau_1 \sigma_2^3L^2 \gamma^3}{2}  \sum_{u' \in D(u)} \mathbb{E} \|   \widehat{v}^{\psi(u')}_{u'} \|^2
 + \frac{L_{\psi(u)} \gamma^2\sigma_2^2}{2} \mathbb{E} \|  \widehat{v}^{\psi(u)}_u  \|^2
 \end{eqnarray}
where the  inequalities (a) follows from Assumption \ref{assum1}.2, (b) uses the fact of $\langle a,b \rangle\leq \frac{1}{2}(\|a\|^2+\|b\|^2)$ and $H_{\xi(u)}$ is positive semidefinite, (c) follows from Assumption \ref{assum3}.2 and
\begin{align}\label{001}
  \mathbb{E} \| \nabla_{\mathcal{G}_{\psi(u)}} f_{\mathcal{I}_u} ({w}_u) - \nabla_{\mathcal{G}_{\psi(u)}} f_{\mathcal{I}_u} (\widehat{w}_u) \|^2
  &=  \mathbb{E} \| \frac{1}{|\mathcal{I}_u|} \sum_{i_u\in \mathcal{I}_u}(\nabla_{\mathcal{G}_{\psi(u)}} f_{i_u} ({w}_u) - \nabla_{\mathcal{G}_{\psi(u)}} f_{i_u} (\widehat{w}_u)) \|^2
  \nonumber \\
  & \overset{(i)}=\frac{|\mathcal{I}_u|}{|\mathcal{I}_u|^2} \sum_{i_u\in \mathcal{I}_u}\mathbb{E}\| (\nabla_{\mathcal{G}_{\psi(u)}} f_{i_u} ({w}_u) - \nabla_{\mathcal{G}_{\psi(u)}} f_{i_u} (\widehat{w}_u)) \|^2
  \nonumber \\
  & \leq L^2\mathbb{E}\|w_u-\widehat{w}_u\|^,
\end{align}
where (i) uses $\|\sum_{i=1}^{n}a_i\|^2 = n\sum_{i=1}^{n}\mathbb{E}\|a_i\|^2$ , and  (ii) follows from Assumption~\ref{assum1}.2.

Summing  (\ref{EqThm1_1}) over all $ u \in K(t)$, we obtain
\begin{eqnarray}\label{EqThm1_2}
&& \mathbb{E} f (w_{t+|K(t)|}) - \mathbb{E}f (w_{t})
\\ \nonumber &\leq&
-  \frac{\gamma\sigma_1}{2} \sum_{u \in K(t)} \mathbb{E} \|  \nabla_{\mathcal{G}_{\psi(u)}} f ({w}_u) \|^2
+  \frac{\tau_1 \sigma_2^3 L^2 \gamma^3}{2}  \sum_{u \in K(t)}\sum_{u' \in D(u)} \mathbb{E} \|   \widehat{v}^{\psi(u')}_{u'} \|^2
\\ \nonumber &&
+ \frac{L_{\max} \sigma_2^2\gamma^2}{2} \sum_{u \in K(t)}\mathbb{E} \|  \widehat{v}^{\psi(u)}_u  \|^2
\\ \nonumber
&\stackrel{ (a) }{\leq} &
 -  \frac{\gamma\sigma_1}{2} \left ( \frac{1}{2} \sum_{u\in K(t)} \| \nabla_{\mathcal{G}_{\psi(u)}} f(w_t) \|^2 - \eta_1 \sigma_2^2\gamma^2 L^2  \sum_{u\in K(t)}   \sum_{v \in \{t,\ldots,u \}}\|\widehat{v}^{\psi(v)}_v \|^2 \right )
\\ \nonumber
&&  +  \frac{\tau_1 \sigma_2^3 L^2 \gamma^3}{2}  \sum_{u \in K(t)}\sum_{u' \in D(u)} \mathbb{E} \|   \widehat{v}^{\psi(u')}_{u'} \|^2
 + \frac{L_{\max} \sigma_2^2\gamma^2}{2} \sum_{u \in K(t)}\mathbb{E} \|  \widehat{v}^{\psi(u)}_u  \|^2
\\ \nonumber & = &
  -  \frac{\gamma\sigma_1}{4} \sum_{u\in K(t)} \| \nabla_{\mathcal{G}_{\psi(u)}} f(w_t) \|^2
  + \frac{\eta_1 \sigma_1\sigma_2^2 \gamma^3 L^2 }{2} \sum_{u\in K(t)}   \sum_{v \in \{t,\ldots,u \}}\|\widehat{v}^{\psi(v)}_v \|^2
\\ \nonumber
&&+  \frac{\tau_1 \sigma_2^3 L^2 \gamma^3}{2}  \sum_{u \in K(t)}\sum_{u' \in D(u)} \mathbb{E} \|   \widehat{v}^{\psi(u')}_{u'} \|^2
 + \frac{L_{\max} \sigma_2^2\gamma^2}{2} \sum_{u \in K(t)}\mathbb{E} \|  \widehat{v}^{\psi(u)}_u  \|^2
  \\ \nonumber
  & \leq &
  -  \frac{\gamma\sigma_1}{4}  \| \nabla f(w_t) \|^2 + \frac{\eta_1\sigma_1\sigma_2^2 \gamma^3 L^2 }{2} \sum_{u\in K(t)}   \sum_{v \in \{t,\ldots,u \}}\|\widehat{v}^{\psi(v)}_v \|^2
\\ \nonumber
&&+  \frac{\tau_1 \sigma_2^3 L^2 \gamma^3}{2}  \sum_{u \in K(t)}\sum_{u' \in D(u)} \mathbb{E} \|   \widehat{v}^{\psi(u')}_{u'} \|^2
 + \frac{L_{\max} \sigma_2^2\gamma^2}{2} \sum_{u \in K(t)}\mathbb{E} \|  \widehat{v}^{\psi(u)}_u  \|^2
 \\ \nonumber
 & \stackrel{ (b) }{\leq} &
 -  \frac{\gamma \mu \sigma_1}{2}  \left (  f(w_t)-f(w^*) \right )
   + \left (\frac{\sigma_1\sigma_2^2 \gamma^3 L^2 \eta_1^3 }{2}
   + \frac{\eta_1 \sigma_2^3 L^2 \gamma^3 \tau_1^2}{2}
   +  \frac{L_{\max} \eta_1\sigma_2^2\gamma^2}{2}  \right ) \frac{G}{b}
\\ \nonumber & = &
-  \frac{\gamma \mu \sigma_1}{2}  \left (  f(w_t)-f(w^*) \right )
+  \underbrace{\frac{\eta_1 \sigma_2^2\gamma^2 \left (  \sigma_1\gamma L^2 \eta_1^2
+  \sigma_2\gamma L^2 \tau_1^2  +  L_{\max}  \right ) G}{2b}}_{C}
\end{eqnarray}
where the  inequality (a) uses Lemma \ref{AsySPSAGA_lemma3}, the  inequality (b) uses Assumptions \ref{assum4} and \ref{assum5}, and that
\begin{align}\label{lemma00}
  & \mathbb{E} \|  \widehat{v}^{\psi(u)}_u  \|^2
  = \mathbb{E} \| \frac{1}{|\mathcal{I}_u|} \sum_{i_u\in \mathcal{I}_u}\nabla_{\mathcal{G}_{\psi(u)}} f_{i_u} (\widehat{w}_u) \|^2
 \overset{(i)} = \frac{1}{|\mathcal{I}_u|^2} \sum_{i_u\in \mathcal{I}_u}  \mathbb{E} \| \nabla_{\mathcal{G}_{\psi(u)}} f_{i_u} (\widehat{w}_u) \|^2
  \leq \frac{G}{|\mathcal{I}_u|}
 \overset{(ii)} = \frac{G}{b},
\end{align}
where (i) follows from that $\mathbb{E}\|\sum_{i=1}^{n}a_i\|^2 = \sum_{i=1}^{n}\mathbb{E}\|a_i\|^2$ for independent random variants and (ii) uses $|\mathcal{I}_u|=b$ and Assumption \ref{assum1}.3. According to (\ref{EqThm1_2}), we have that
\begin{eqnarray}\label{EqThm1_3}
&& \mathbb{E} f (w_{t+|K(t)|}) -f(w^*)\leq  \left (1-  \frac{\gamma \mu \sigma_1}{2} \right ) \left (  f(w_t)-f(w^*) \right ) +  C
\end{eqnarray}
Assume that $\cup_{\kappa \in P(t)}=\{0,1,\ldots,t\}$, applying  (\ref{EqThm1_3}), we have that
\begin{eqnarray}\label{EqThm1_4}
&& \mathbb{E} f (w_{t}) -f(w^*)
\\ \nonumber
& \leq & \left (1-  \frac{\gamma \mu \sigma_1}{2} \right )^{\upsilon(t)} \left (  f(w_0)-f(w^*) \right ) +  C \sum_{i=0}^{\upsilon(t)}\left (1-  \frac{\gamma \mu \sigma_1}{2} \right )^i
\\ \nonumber
& \leq & \left (1-  \frac{\gamma \mu \sigma_1}{2} \right )^{\upsilon(t)} \left (  f(w_0)-f(w^*) \right ) +  C \sum_{i=0}^{\infty}\left (1-  \frac{\gamma \mu \sigma_1}{2} \right )^i
\\ \nonumber
& = & \left (1-  \frac{\gamma \mu \sigma_1}{2} \right )^{\upsilon(t)} \left (  f(w_0)-f(w^*) \right ) +   \frac{{\gamma \eta_1 \sigma_2^2\left (  \sigma_1\gamma L^2\eta_1^2  +  \sigma_2 \gamma L^2 \tau_1^2  +  L_{\max}  \right ) G}}{ \mu \sigma_1 b}
\end{eqnarray}
Let $ \frac{{\gamma \eta_1 \sigma_2^2\left (  \sigma_1\gamma L^2\eta_1^2  +  \sigma_2 \gamma L^2 \tau_1^2  +  L_{\max}  \right ) G}}{ \mu \sigma_1 b} \leq \frac{\epsilon}{2}$, we have that $\gamma \leq  \frac{-   L_{\max}   + \sqrt{  L_{\max}^2  + \frac{{2 \mu \sigma_1 b \epsilon  (\sigma_1L^2 \eta_1^2  + \sigma_2\tau_1  L^2 \tau )}}{G\eta_1 \sigma_2^2}} }{2 L^2 (\sigma_1\eta_1^2  + \sigma_2\tau_1^2 )}$.

Let $\left (1-  \frac{\gamma \mu \sigma_1}{2} \right )^{\upsilon(t)} \left (  f(w_0)-f(w^*) \right ) \leq \frac{\epsilon}{2}$, we have that
\begin{eqnarray}\label{EqThm1_5}
\log\left ( \frac{2 \left (  f(w_0)-f(w^*) \right )}{\epsilon} \right ) \leq \upsilon(t) \log \left ( \frac{1}{1-  \frac{\gamma \mu \sigma_1}{2} } \right )
\end{eqnarray}
Because $\log \left ( \frac{1}{\rho} \right ) \geq 1 -\rho$ for $0<\rho \leq 1$, we have that
\begin{eqnarray}\label{EqThm1_5}
\upsilon(t)
&\geq& \frac{2}{\gamma \mu \sigma_1} \log\left ( \frac{2 \left (  f(w_0)-f(w^*) \right )}{\epsilon} \right )
\\ \nonumber
&\geq& \frac{2}{ \mu \sigma_1} \frac{2 L^2 (\sigma_1\eta_1^2  + \sigma_2^2\tau_1^2)}{-   L_{\max}   + \sqrt{  L_{\max}^2  + \frac{{2 \mu \sigma_1 b\epsilon  (\sigma_1L^2 \eta_1^2  + \sigma_2  L^2 \tau_1^2 )}}{G\eta_1 \sigma_2^2}} }  \log\left ( \frac{2 \left (  f(w_0)-f(w^*) \right )}{\epsilon} \right )
\end{eqnarray}
This completes the proof.
\end{proof}

\section*{Appendix B: Proof of Theorem  \ref{thm-svrgconvex}}

Before proving Theorem \ref{thm-svrgconvex}, we first provide  Lemma \ref{AsySCGD+_lemma1} to provides an upper bound to $ \mathbb{E}  \left \|   \widehat{v}^{\ell}_u \right \|^2$.
\begin{lemma} \label{AsySCGD+_lemma1} For  AsySQN-SVRG, under Assumptions \ref{assum1}, \ref{assum2}.1, \ref{assum4} and \ref{assum5},  let $u \in K(t)$,  we have that
\begin{eqnarray}
\label{AsySCGD+_lem3_0_1} && \mathbb{E}  \left \|   \widehat{v}^{\ell}_u \right \|^2
\\ \nonumber &\leq &   \frac{16 L^2}{\mu}\mathbb{E} (f(w^s_t)-f(w^*))  +  \frac{8 L^2}{\mu} \mathbb{E} (f(w^s)-f(w^*))
\\ \nonumber  &  &+ 4 L^2 \gamma^2 \sigma_2^2\eta_1 \sum_{v \in \{t,\ldots,u \}} \mathbb{E} \left  \|   \widehat{v}^{\psi(v)}_v\right \|^2   + 2 \tau_1 L^2 \gamma^2 \sigma_2^2\mathbb{E}  \sum_{u' \in D(u)}   \left \|    \widehat{v}^{\psi(u')}_{u'}  \right  \|^2
\end{eqnarray}
\end{lemma}
\begin{proof} Define ${v}^{\ell}_u= \nabla_{\mathcal{G}_\ell} f_{\mathcal{I}_u} ({w}_u^s) - \nabla_{\mathcal{G}_\ell} f_{\mathcal{I}_u} (w^s)
  +  \nabla_{\mathcal{G}_\ell} f(w^s)$. We have that
$
\label{AsySCGD+_lem3_0_2}  \mathbb{E}  \left \|   \widehat{v}^{ \ell }_u \right \|^2 =\mathbb{E}  \left \|   \widehat{v}^{ \ell }_u - {v}^{\ell}_u + {v}^{\ell}_u \right \|^2 \leq 2 \mathbb{E}   \left \|   \widehat{v}^{ \ell }_u - {v}^{\ell}_u \right  \|^2 +   2 \mathbb{E} \left  \| {v}^{\ell}_u \right \|^2
$.
Firstly, we give the upper bound to $\mathbb{E} \left  \| {v}^{\ell}_u \right \|^2$ as follows.
\begin{align}\label{Lemma2_1}
& \mathbb{E}\left\|v_{u}^{\ell}\right\|^{2}
\nonumber \\
=& \mathbb{E}\left\|\nabla_{\mathcal{G}_{\ell}} f_{\mathcal{I}_u}\left(w_{u}^{s}\right)-\nabla_{\mathcal{G}_{\ell}} f_{\mathcal{I}_u}\left(w^{s}\right)+\nabla_{\mathcal{G}_{\ell}} f\left(w^{s}\right)\right\|^{2}
\nonumber \\
=& \mathbb{E}\left\|\nabla_{\mathcal{G}_{\ell}} f_{\mathcal{I}_u}\left(w_{u}^{s}\right)-\nabla_{\mathcal{G}_{\ell}} f_{\mathcal{I}_u}\left(w^{*}\right)-\nabla_{\mathcal{G}_{\ell}} f_{\mathcal{I}_u}\left(w^{s}\right)+\nabla_{\mathcal{G}_{\ell}} f_{\mathcal{I}_u}\left(w^{*}\right)+\nabla_{\mathcal{G}_{\ell}} f\left(w^{s}\right)\right\|^{2}
\nonumber \\
\leq & 2 \mathbb{E}\left\|\nabla_{\mathcal{G}_{\ell}} f_{\mathcal{I}_u}\left(w_{u}^{s}\right)-\nabla_{\mathcal{G}_{\ell}} f_{\mathcal{I}_u}\left(w^{*}\right)\right\|^{2}+2 \mathbb{E}\left\|\nabla_{\mathcal{G}_{\ell}} f_{\mathcal{I}_u}\left(w^{s}\right)-\nabla_{\mathcal{G}_{\ell}} f_{\mathcal{I}_u}\left(w^{*}\right)-\nabla_{\mathcal{G}_{\ell}} f\left(w^{s}\right)+\nabla_{\mathcal{G}_{\ell}} f\left(w^{*}\right)\right\|^{2}
\nonumber \\
\overset{(i)}\leq & 2 \mathbb{E}\left\|\nabla_{\mathcal{G}_{\ell}} f_{\mathcal{I}_u}\left(w_{u}^{s}\right)-\nabla_{\mathcal{G}_{\ell}} f_{\mathcal{I}_u}\left(w^{*}\right)\right\|^{2}+2 \mathbb{E}\left\|\nabla_{\mathcal{G}_{\ell}} f_{\mathcal{I}_u}\left(w^{s}\right)-\nabla_{\mathcal{G}_{\ell}} f_{\mathcal{I}_u}\left(w^{*}\right)\right\|^{2}
\nonumber \\
= & 2 \mathbb{E}\|\frac{1}{|\mathcal{I}_u|}\sum_{i_u\in \mathcal{I}_u} (\nabla_{\mathcal{G}_{\ell}} f_{i_u}\left(w_{u}^{s}\right)-\nabla_{\mathcal{G}_{\ell}} f_{i_u}\left(w^{*}\right))\|^{2}
+2 \mathbb{E}\|\frac{1}{|\mathcal{I}_u|}\sum_{i_u\in \mathcal{I}_u}(\nabla_{\mathcal{G}_{\ell}} f_{i_u}\left(w^{s}\right)-\nabla_{\mathcal{G}_{\ell}} f_{i_u}\left(w^{*}\right))\|^{2}
\nonumber \\
\overset{(ii)}\leq & 2 L^{2} \mathbb{E}\left\|w_{u}^{s}-w^{*}\right\|^{2}+2 L^{2} \mathbb{E}\left\|w^{s}-w^{*}\right\|^{2}
\nonumber \\
= & 2 L^{2} \mathbb{E}\left\|w_{u}^{s}-w_t^s + w_t^s-w^{*}\right\|^{2}+2 L^{2} \mathbb{E}\left\|w^{s}-w^{*}\right\|^{2}
\nonumber \\
\leq & 4L^2 \mathbb{E}\|w_{u}^{s}-w_t^s\|^2 + 4L^2\mathbb{E}\| w_t^s-w^{*}\|^{2}+2 L^{2} \mathbb{E}\left\|w^{s}-w^{*}\right\|^{2}
\nonumber \\
\overset{(iii)}=& 4L^2\gamma^2 \mathbb{E}\|\sum_{v\in\{t,\ldots,u\}}U_{\psi(u)}H_{\xi(v)}\widehat{v}_v^{\psi(v)}\|^2 + 4L^2\mathbb{E}\| w_t^s-w^{*}\|^{2}+2 L^{2} \mathbb{E}\left\|w^{s}-w^{*}\right\|^{2}
\nonumber \\
\overset{(iv)}\leq & \frac{8L^2}{\mu} \mathbb{E}(f(w_t^s)-f(w^*)) + \frac{4L^2}{\mu}\mathbb{E}(f(w^s)-f(w^*)) + 4L^2\gamma^2\sigma_2^2\eta_1\sum_{v\in\{t,\ldots,u\}}\mathbb{E}\|\widehat{v}_v^{\psi(v)}\|^2
\end{align}
where the (i) follows from $\mathbb{E} \left \| x- \mathbb{E} x\right \|^2 \leq \mathbb{E} \left \| x\right \|^2$, (ii)  follows from the proof of Eq.~(\ref{001}), (iii) uses Definition 2, and the inequality (iv) uses Assumption \ref{assum1}.

Next, we give the upper bound to $\mathbb{E}   \left \|   \widehat{v}^{ \ell }_u - {v}^{\ell}_u \right  \|^2$ as follows.
\begin{eqnarray}\label{Lemma2_2}
  \mathbb{E}   \left \|   \widehat{v}^{ \ell }_u - {v}^{\ell}_u \right  \|^2
 &=&\mathbb{E}   \left \| \frac{1}{|{\mathcal{I}_u}|} \sum_{i_u\in {\mathcal{I}_u}}(\nabla_{\mathcal{G}_\ell} f_{i_u} (\widehat{w}_u^s)  - \nabla_{\mathcal{G}_\ell} f_{i_u} ({w}_u^s)) \right  \|^2
 \nonumber \\
 &\overset{i}\leq& \frac{1}{|{\mathcal{I}_u}|} \sum_{i_u\in {\mathcal{I}_u}} \mathbb{E}   \left \| (\nabla_{\mathcal{G}_\ell} f_{i_u} (\widehat{w}_u^s)  - \nabla_{\mathcal{G}_\ell} f_{i_u} ({w}_u^s)) \right  \|^2
\\ \nonumber
&\leq& L^2 \mathbb{E}   \left \| \widehat{w}_u^s  - {w}_u^s \right  \|^2
\\ \nonumber
&\overset{ii}=& L^2 \gamma^2 \mathbb{E}   \left \|  \sum_{u' \in D(u)}    \textbf{U}_{\psi(u')} H_{u'}\widehat{v}^{\psi(u')}_{u'}  \right  \|^2
\\ \nonumber
&\overset{iii}\leq & \tau_1 L^2\sigma_2^2 \gamma^2 \mathbb{E}  \sum_{u' \in D(u)}   \left \|    \widehat{v}^{\psi(u')}_{u'}  \right  \|^2,
\end{eqnarray}
where (i) follows from $\|\sum_{i=1}^{n}a_i\|^2 \leq n\sum_{i=1}^{n}\mathbb{E}\|a_i\|^2$, (ii) uses Definition 2 and (iii) follows from Assumption \ref{assum4}.
Combining  (\ref{Lemma2_1}) and (\ref{Lemma2_2}), we have that
\begin{eqnarray}\label{Lemma2_3}
 && \mathbb{E}  \left \|   \widehat{v}^{ \ell }_u \right \|^2
\\ \nonumber  &\leq &    2 \mathbb{E}   \left \|   \widehat{v}^{ \ell }_u - {v}^{\ell}_u \right  \|^2 +   2 \mathbb{E} \left  \| {v}^{\ell}_u \right \|^2
\\ \nonumber &\leq &   \frac{16 L^2}{\mu}\mathbb{E} (f(w^s_t)-f(w^*))  +  \frac{8 L^2}{\mu} \mathbb{E} (f(w^s)-f(w^*))
\\ \nonumber  &  &+ 4 L^2 \gamma^2 \sigma_2^2\eta_1 \sum_{v \in \{t,\ldots,u \}} \mathbb{E} \left  \|   \widehat{v}^{\psi(v)}_v\right \|^2   + 2 \tau_1 L^2 \gamma^2 \sigma_2^2\mathbb{E}  \sum_{u' \in D(u)}   \left \|    \widehat{v}^{\psi(u')}_{u'}  \right  \|^2
\end{eqnarray}
This completes the proof.
\end{proof}
Based on the basic inequalities in Lemma \ref{AsySPSAGA_lemma3}, we provide the proof of Theorem \ref{thm-svrgconvex} in the following.

\begin{proof}{\textbf{of Theorem \ref{thm-svrgconvex}:}}
Similar to (\ref{EqThm1_1}),  for $u \in K(t)$ at $s$-th outer loop, we have that
\begin{eqnarray}\label{EqThm2_1}
&& \mathbb{E} f (w_{u+1}^s)
\\ \nonumber
&\stackrel{ (a) }{\leq}&  \mathbb{E} \left ( f (w_{u}^s) + \langle \nabla f(w_{u}^s), w_{u+1}^s-w_{u}^s  \rangle + \frac{L_{\psi(u)}}{2} \|w_{u+1}^s-w_{u}^s   \|^2  \right )
\\ \nonumber
&=&  \mathbb{E} \left ( f (w_{u}^s) -  \gamma \langle \nabla f(w_{u}^s),  H_{\xi(u)}\widehat{v}^{\psi(u)}_u  \rangle + \frac{L_{\psi(u)} \gamma^2}{2} \| H_{\xi(u)}\widehat{v}^{\psi(u)}_u  \|^2  \right )
\\ \nonumber
&\stackrel{ (b) }{=}&  \mathbb{E}f (w_{u}^s) -  \gamma \mathbb{E} \langle \nabla f(w_{u}^s),  H_{\xi(u)} \nabla_{\mathcal{G}_{\psi(u)}} f_{\mathcal{I}_u}(\widehat{w}_{u}^s)  \rangle + \frac{L_{\psi(u)} \gamma^2\|H_{\xi(u)}\|^2}{2} \mathbb{E} \|  \widehat{v}^{\psi(u)}_u  \|^2
\\ \nonumber
&=&  \mathbb{E} f (w_{u}^s) -  \gamma \mathbb{E} \langle \nabla f(w_{u}^s), H_{\xi(u)}\nabla_{\mathcal{G}_{\psi(u)}} f_{\mathcal{I}_u}(\widehat{w}_{u}^s)- H_{\xi(u)}\nabla_{\mathcal{G}_{\psi(u)}} f_{\mathcal{I}_u} ({w}_u^s)
+ H_{\xi(u)}\nabla_{\mathcal{G}_{\psi(u)}} f_{\mathcal{I}_u} ({w}_u^s) \rangle
\\ \nonumber
 && + \frac{L_{\psi(u)} \gamma^2\|H_{\xi(u)}\|^2}{2} \mathbb{E} \|  \widehat{v}^{\psi(u)}_u  \|^2
\\ \nonumber
&=&  \mathbb{E}  f (w_{u}^s) -  \gamma \mathbb{E} \langle \nabla f(w_{u}^s),  H_{\xi(u)}\nabla_{\mathcal{G}_{\psi(u)}} f ({w}_u^s) \rangle  + \gamma \mathbb{E} \langle \nabla f(w_{u}^s),   H_{\xi(u)}( \nabla_{\mathcal{G}_{\psi(u)}} f_{\mathcal{I}_u} ({w}_u^s) - \nabla_{\mathcal{G}_{\psi(u)}} f_{\mathcal{I}_u} (\widehat{w}_u^s) ) \rangle
 \\ \nonumber
 && + \frac{L_{\psi(u)} \gamma^2\|H_{\xi(u)}\|^2}{2} \mathbb{E} \|  \widehat{v}^{\psi(u)}_u  \|^2
 \\ \nonumber
 &\stackrel{ (c) }{\leq}&  \mathbb{E}  f (w_{u}^s) -  \frac{\gamma\sigma_1}{2} \mathbb{E} \|  \nabla_{\mathcal{G}_{\psi(u)}} f ({w}_u^s) \|^2  +  \frac{\tau_1 \sigma_2^3 L^2 \gamma^3}{2}  \sum_{u' \in D(u)} \mathbb{E} \|   \widehat{v}^{\psi(u')}_{u'} \|^2
 + \frac{L_{\psi(u)} \gamma^2\sigma_2^2}{2} \mathbb{E} \|  \widehat{v}^{\psi(u)}_u  \|^2
 \end{eqnarray}
where the  inequalities (a)  use Assumption \ref{assum2}.1, the equality (b) uses Assumption \ref{assum3}.2, the inequality (c) follows the proof in (\ref{EqThm1_1}).

Summing  (\ref{EqThm2_1}) over all $ u \in K(t)$, we obtain
\begin{eqnarray}\label{EqThm2_2}
&& \mathbb{E} f (w_{t+|K(t)|}^s) - \mathbb{E}f (w_{t}^s)
\\ \nonumber
&\leq&   -  \frac{\gamma\sigma_1}{2} \sum_{u \in K(t)} \mathbb{E} \|  \nabla_{\mathcal{G}_{\psi(u)}} f ({w}_u^s) \|^2  +  \frac{\tau_1 \sigma_2^3 L^2 \gamma^3}{2}  \sum_{u \in K(t)}\sum_{u' \in D(u)} \mathbb{E} \|   \widehat{v}^{\psi(u')}_{u'} \|^2
\\ \nonumber
&& + \frac{L_{\max} \sigma_2^2\gamma^2}{2} \sum_{u \in K(t)}\mathbb{E} \|  \widehat{v}^{\psi(u)}_u  \|^2
\\ \nonumber
&\stackrel{ (a) }{\leq} &   -  \frac{\gamma\sigma_1}{2} \left ( \frac{1}{2} \sum_{u\in K(t)} \| \nabla_{\mathcal{G}_{\psi(u)}} f(w_t^s) \|^2 - \eta_1 \gamma^2 L^2 \sigma_2^2 \sum_{u\in K(t)}   \sum_{v \in \{t,\ldots,u \}}\|\widehat{v}^{\psi(v)}_v \|^2 \right )
\\ \nonumber
&&  +  \frac{\tau_1 \sigma_2^3 L^2 \gamma^3}{2}  \sum_{u \in K(t)}\sum_{u' \in D(u)} \mathbb{E} \|   \widehat{v}^{\psi(u')}_{u'} \|^2
 + \frac{L_{\max} \sigma_2^2\gamma^2}{2} \sum_{u \in K(t)}\mathbb{E} \|  \widehat{v}^{\psi(u)}_u  \|^2
 \\ \nonumber
 & = &   -  \frac{\gamma\sigma_1}{4} \sum_{u\in K(t)} \| \nabla_{\mathcal{G}_{\psi(u)}} f(w_t^s) \|^2 + \frac{\eta_1 \sigma_1\sigma_2^2 \gamma^3 L^2 }{2} \sum_{u\in K(t)}   \sum_{v \in \{t,\ldots,u \}}\|\widehat{v}^{\psi(v)}_v \|^2
\\ \nonumber
&&  +  \frac{\tau_1 \sigma_2^3 L^2 \gamma^3}{2}  \sum_{u \in K(t)}\sum_{u' \in D(u)} \mathbb{E} \|   \widehat{v}^{\psi(u')}_{u'} \|^2
 + \frac{L_{\max} \sigma_2^2\gamma^2}{2} \sum_{u \in K(t)}\mathbb{E} \|  \widehat{v}^{\psi(u)}_u  \|^2
 \\ \nonumber
 & \stackrel{ (b) }{\leq} &   -  \frac{\gamma\sigma_1}{4} \| \nabla f(w_t^s) \|^2 +  \frac{\tau_1 \sigma_2^3 L^2 \gamma^3}{2}  \sum_{u \in K(t)}\sum_{u' \in D(u)} \mathbb{E} \|   \widehat{v}^{\psi(u')}_{u'} \|^2
\\ \nonumber
&&   + \left ( \frac{\eta_1\sigma_1\sigma_2^2 \gamma^3 L^2 \eta_1 }{2}+ \frac{L_{\max} \sigma_2^2\gamma^2}{2} \right ) \sum_{u \in K(t)}\mathbb{E} \|  \widehat{v}^{\psi(u)}_u  \|^2
\\ \nonumber
& \stackrel{ (c) }{\leq} &   -  \frac{\gamma \mu\sigma_1}{2} \mathbb{E} (f(w^s_t)-f(w^*))
+  \frac{\tau_1 \sigma_2^3 L^2 \gamma^3}{2}  \sum_{u \in K(t)}\sum_{u' \in D(u)} \mathbb{E} \|   \widehat{v}^{\psi(u')}_{u'} \|^2 + \underbrace{\left (\sigma_1\gamma L^2 \eta_1^2 + L_{\max}  \right )\frac{\gamma^2\sigma_2^2}{2}}_{C}
\\ \nonumber
&&   \cdot \sum_{u \in K(t)}\left (   \frac{16 L^2}{\mu}\mathbb{E} (f(w^s_t)-f(w^*))  +  \frac{8 L^2}{\mu} \mathbb{E} (f(w^s)-f(w^*)) \right .
\\ \nonumber
&  & \left . + 4 L^2 \gamma^2 \sigma_2^2\eta_1 \sum_{v \in \{t,\ldots,u \}} \mathbb{E} \left  \|   \widehat{v}^{\psi(v)}_v\right \|^2   + 2 \tau_1 L^2 \gamma^2 \sigma_2^2\mathbb{E}  \sum_{u' \in D(u)}   \left \|    \widehat{v}^{\psi(u')}_{u'}  \right  \|^2  \right )
\end{eqnarray}
where the  inequality (a) uses Lemma \ref{AsySPSAGA_lemma3}, the  inequality (b) uses Assumption  \ref{assum5}, the inequality (c) uses Lemma \ref{AsySCGD+_lemma1}.

Let $e^s_t=\mathbb{E} (f(w^s_t)-f(w^*))$ and $e^s=\mathbb{E} (f(w^s)-f(w^*))$, we have
\begin{eqnarray}\label{EqThm2_3}
&& e_{t+|K(t)|}^s
\\ \nonumber
& \stackrel{ (a) }{\leq}  & \left ( 1 - \frac{\gamma \mu\sigma_1}{2} +  \frac{16 L^2 \eta_1  C}{\mu}  \right ) e^s_t +  \frac{8 L^2 \eta_1  C}{\mu} e^s + 4 C    L^2 \gamma^2 \sigma_2^2\eta_1  \sum_{u \in K(t)}   \sum_{v \in \{t,\ldots,u \}} \mathbb{E} \left  \|   \widehat{v}^{\psi(v)}_v\right \|^2
\\ \nonumber
&&  +  \left ( \frac{\tau_1 \sigma_2^3 L^2 \gamma^3}{2} + 2 \tau_1 L^2 \gamma^2 \sigma_2^2C \right ) \sum_{u \in K(t)}\sum_{u' \in D(u)} \mathbb{E} \|   \widehat{v}^{\psi(u')}_{u'} \|^2
\\ \nonumber
& \stackrel{ (b) }{\leq}  & \left ( 1 - \frac{\gamma \mu\sigma_1}{2} +  \frac{16 L^2 \eta_1 C}{\mu}  \right ) e^s_t +  \frac{8 L^2 \eta_1 C}{\mu} e^s
\\ \nonumber
&&  + \gamma^3 \left (  \left ( \frac{  \sigma_2 }{2} +     \frac{2C}{\gamma} \right )  \tau_1^2   + 4 \frac{C}{\gamma}     \eta_1^2  \right ) \eta_1  L^2 \sigma_2^2 \frac{9G}{b}
\end{eqnarray}
where the   inequality (a) uses (\ref{EqThm2_2}), and the inequality (b) uses Assumptions \ref{assum4}, \ref{assum5} and follows the proof of Eq.~\ref{lemma00}.
We carefully choose $\gamma$ such that $1> \frac{\gamma \mu}{2} -  \frac{16 L^2 \eta_1  C}{\mu}\stackrel{\rm def}{=}\rho>0$. Assume that $\cup_{\kappa \in P(t)}=\{0,1,\ldots,t\}$, applying  (\ref{EqThm2_3}), we have that
\begin{eqnarray}\label{EqThm2_4}
&& e_{t}^s
\\ \nonumber
& \leq & \left ( 1 - \rho \right )^{\upsilon(t)}  e^s  +   \left ( \frac{8 L^2 \eta_1 C}{\mu} e^s
 + \gamma^3 \left (  \left ( \frac{ \sigma_2 }{2} +     \frac{2C}{\gamma} \right )  \tau_1^2   + 4 \frac{C}{\gamma}     \eta_1^2  \right ) \eta_1 q  L^2 \sigma_2^2 \frac{9G}{b} \right ) \sum_{i=0}^{{\upsilon(t)}}\left (1-  \rho \right )^i
\\ \nonumber
& \leq & \left ( 1 - \rho \right )^{\upsilon(t)}  e^s  +   \left ( \frac{8 L^2 \eta_1 C}{\mu} e^s
 + \gamma^3 \left (  \left ( \frac{ \sigma_2   }{2} +     \frac{2C}{\gamma} \right )   \tau_1^2    + 4 \frac{C}{\gamma}     \eta_1^2  \right ) \eta_1 q  L^2 \sigma_2^2 \frac{9G}{b} \right ) \frac{1}{\rho}
 \\ \nonumber
 & = & \left ( \left ( 1 - \rho \right )^{\upsilon(t)} + \frac{8 L^2 \eta_1 C}{\rho \mu} \right ) e^s  +
 \gamma^3 \left (  \left ( \frac{ \sigma_2  }{2} +     \frac{2C}{\gamma} \right )  \tau_1^2   + 4 \frac{C}{\gamma}     \eta_1^2  \right )   \frac{9\eta_1  L^2 \sigma_2^2 G}{b\rho}
\end{eqnarray}
Thus, to achieve the accuracy $\epsilon$ of Problem~(\ref{P}) for AsySQN-SVRG, \emph{i.e.}, $\mathbb{E} f (w_{S}) -f(w^*) \leq \epsilon$,  we can carefully choose $\gamma$ such that
\begin{eqnarray}
\frac{8 L^2 \eta_1 C}{\rho \mu} &\leq& 0.5
\\  \gamma^3 \left (  \left ( \frac{ \sigma_2  }{2} +     \frac{2C}{\gamma} \right )   \tau_1^2   + 4 \frac{C}{\gamma}     \eta_1^2  \right )   \frac{9\eta_1  L^2 \sigma_2^2 G}{b\rho} &\leq& \frac{\epsilon}{8}
\end{eqnarray}
and let $\left ( 1 - \rho \right )^{\upsilon(t)} \leq 0.25$, i.e., $\upsilon(t) \geq \frac{\log 0.25}{\log (1 - \rho)}$, we have that
\begin{eqnarray}\label{EqThm2_5}
e^{s+1} \leq 0.75 e^{s} + \frac{\epsilon}{8}
\end{eqnarray}
Recursively apply (\ref{EqThm2_5}), we have that
\begin{eqnarray}
e^{S} \leq (0.75)^S e^{0}+ \frac{\epsilon}{2}
\end{eqnarray}
Finally, the outer loop  number $S$ should satisfy the  condition of $S \geq \frac{\log \frac{2 e^0}{\epsilon }}{\log \frac{4}{3}} $.
This completes the proof.
\end{proof}

\section*{Appendix C: Proof of Theorem  \ref{thm-sagaconvex}}
Before proving Theorem \ref{thm-sagaconvex}, we first provide  Lemma \ref{lemma3} to provides an upper bound to $ \mathbb{E}  \left \|   \widehat{v}^{\ell}_u \right \|^2$.
\begin{lemma} \label{AsySGHT_lemma2} For AsySQN-SAGA,  we have that
\begin{align}\label{lem4_0}
 & \mathbb{E} \left \|  \alpha_{\I_u}^{u,\ell} -  \nabla_{\mathcal{G}_\ell} f_{\I_u}(w^*) \right \|^2
   \\ \nonumber
   &\leq
\frac{L^2}{nb} \sum_{u'=1}^{\xi(u,\ell)-1} \left ( 1 -\frac{1}{n} \right )^{\xi(u,\ell)-u'-1} \sigma(w_{{\xi^{-1}(u',\ell)}})
 +\frac{ L^2 }{b}\left ( 1 -\frac{1}{n} \right )^{\xi(u,\ell)} \sigma(w_0)
\\  &
\mathbb{E}\left \|  {\alpha}_{\I_u}^{u,\ell} - \widehat{\alpha}_{\I_u}^{u,\ell} \right \|^2
\\ \nonumber
& \leq \frac{\tau_1 L^2 \sigma_2^2 \gamma^2}{nb} \sum_{u'=1}^{\xi(u,\ell)-1}  \sum_{\widetilde{u} \in D(\xi^{-1}(u',\ell))} \left ( 1 -\frac{1}{n} \right )^{\xi(u,\ell)-u'-1}
\mathbb{E}  \left \|       \widehat{v}^{\psi(\widetilde{u})}_{\widetilde{u}} \right \|^2
\end{align}
where $\sigma(w_u) = \mathbb{E} \| w_u - w^* \|^2$.
\end{lemma}
\begin{proof} Firstly, we have that
\begin{align}\label{lem4_1}
& \mathbb{E} \left \| \alpha_{\mathcal{I}_u}^{u,\ell} -  \nabla_{\mathcal{G}_\ell} f_{\mathcal{I}_u}(w^*)  \right \|^2
\\ \nonumber
=& \mathbb{E}\left \|\frac{1}{|\I_u|} \sum_{i_u\in \I_u} (\alpha_{i_u}^{u,\ell} -  \nabla_{\mathcal{G}_\ell} f_{i_u}(w^*))  \right \|^2
\\ \nonumber
\overset{(i)}=& \frac{1}{|\I_u|^2}\sum_{i_u\in \I_u}\mathbb{E} \left \| (\alpha_{i_u}^{u,\ell} -  \nabla_{\mathcal{G}_\ell} f_{i_u}(w^*))  \right \|^2
\\ \nonumber
=&  \frac{1}{n|\I_u|^2} \sum_{ |\I_u|}\sum_{i=1}^n \mathbb{E} \left \|  \alpha_{i}^{u,\ell} -  \nabla_{\mathcal{G}_\ell} f_{i}(w^*) \right \|^2
\\  \nonumber
 \overset{(ii)}= &\frac{1}{nb} \sum_{i=1}^n   \mathbb{E} \sum_{u'=0}^{\xi(u,\ell) -1} \mathbf{1}_{ \{ \textbf{u}_{i}^u =u' \}} \left \|   \nabla_{\mathcal{G}_\ell} f_i(w_{\xi^{-1}(u',\ell)}) -  \nabla_{\mathcal{G}_\ell} f_i(w^*) \right \|^2
\\   \nonumber
= &\frac{1}{nb} \sum_{u'=0}^{\xi(u,\ell)-1} \sum_{i=1}^n  \mathbb{E}  \mathbf{1}_{ \{ \textbf{u}_{i}^u =u' \}} \left \|  \nabla_{\mathcal{G}_\ell} f_i(w_{\xi^{-1}(u',\ell)}) -  \nabla_{\mathcal{G}_\ell} f_i(w^*) \right \|^2
\end{align}
where  $\textbf{u}_{i}^u$ denote the last iterate  to update the $\widehat{\alpha}_i^{u,\ell}$, (i) follows from that $\i_u\in\I_u$ are sampled with replacement, (ii) follows from $|\I|=b$. We consider the two cases $u'>0$ and $u'=0$ as following.

For $u'>0$,  we have that
\begin{align}\label{lem4_2}
& \mathbb{E} \left ( \mathbf{1}_{ \{ \textbf{u}_{i}^u =u' \}} \left \|   \nabla_{\mathcal{G}_\ell} f_i(w_{{\xi^{-1}(u',\ell)}}) -  \nabla_{\mathcal{G}_\ell} f_i(w^*) \right \|^2 \right )
\\    \stackrel{ (a) }{\leq}  & \nonumber \mathbb{E} \left ( \mathbf{1}_{ \{ i_{u'} = i \}} \mathbf{1}_{ \{ i_v \neq i, \forall v \ s.t. \ u'+1 \leq v \leq \xi(u,\ell) -1 \}}  \right . \left .
 \left \|   \nabla_{\mathcal{G}_\ell} f_i(w_{{\xi^{-1}(u',\ell)}}) -  \nabla_{\mathcal{G}_\ell} f_i(w^*) \right \|^2 \right )
\\    \stackrel{ (b) }{\leq}  & \nonumber  P{ \{ i_{u'} = i \}}  P { \{ i_v \neq i, \forall v \ s.t. \ u'+1 \leq v \leq \xi(u,\ell) -1 \}} \cdot
 \mathbb{E}   \left \|   \nabla_{\mathcal{G}_\ell} f_i(w_{{\xi^{-1}(u',\ell)}}) -  \nabla_{\mathcal{G}_\ell} f_i(w^*) \right \|^2
\\    \stackrel{ (c) }{\leq}  & \nonumber \frac{1}{n} \left ( 1 -\frac{1}{n} \right )^{\xi(u,\ell)-u'-1}  \mathbb{E}  \left \|  \nabla_{\mathcal{G}_\ell} f_i(w_{{\xi^{-1}(u',\ell)}}) -  \nabla_{\mathcal{G}_\ell} f_i(w^*) \right \|^2
\end{align}
where the inequality (a) uses the fact $i_{u'}$ and $i_v$ are independent for $v \neq u'$, the inequality (b) uses the fact that $P{ \{ i_u = i \}} = \frac{1}{n}$ and $P { \{ i_v \neq i\} } =1-(1-\frac{1}{n})^{b}\leq 1-\frac{1}{n}$.

For $u'=0$, we have that
\begin{align}\label{lem4_3}
& \mathbb{E} \left ( \mathbf{1}_{ \{ \textbf{u}_{i}^u =0 \}}\left \|   \nabla_{\mathcal{G}_\ell} f_i(w_{0}) -  \nabla_{\mathcal{G}_\ell} f_i(w^*) \right \|^2   \right )
\\    \leq & \nonumber \mathbb{E} \left ( \mathbf{1}_{ \{ i_v \neq i, \forall v \ s.t. \ 0 \leq v \leq \xi(u,\ell)-1 \}}
 \left \|   \nabla_{\mathcal{G}_\ell} f_i(w_{0}) -  \nabla_{\mathcal{G}_\ell} f_i(w^*) \right \|^2 \right )
\\    \leq & \nonumber   P { \{ i_v \neq i, \forall v \ s.t. \ 0 \leq v \leq \xi(u,\ell) -1 \}}
  \mathbb{E}  \left \|   \nabla_{\mathcal{G}_\ell} f_i(w_{0}) -  \nabla_{\mathcal{G}_\ell} f_i(w^*) \right \|^2
\\    \leq & \nonumber  \left ( 1 -\frac{1}{n} \right )^{\xi(u,\ell)}  \mathbb{E}  \left \|   \nabla_{\mathcal{G}_\ell} f_i(w_{0}) -  \nabla_{\mathcal{G}_\ell} f_i(w^*) \right \|^2
\end{align}

Substituting (\ref{lem4_2}) and (\ref{lem4_3}) into (\ref{lem4_1}), we have that
\begin{align}\label{lem4_4}
& \mathbb{E} \left \| \alpha_{\I_u}^{u,\ell} -  \nabla_{\mathcal{G}_\ell} f_{\I_u}(w^*)  \right \|^2
\\    = & \nonumber
\frac{1}{nb} \sum_{u'=0}^{\xi(u,\ell) -1} \sum_{i=1}^n  \mathbb{E}  \mathbf{1}_{ \{ \textbf{u}_{i}^u =u' \}} \left \|  \nabla_{\mathcal{G}_\ell} f_i(w_{{\xi^{-1}(u',\ell)}}) -  \nabla_{\mathcal{G}_\ell} f_i(w^*) \right \|^2
\\   \stackrel{ (a) }{\leq} & \nonumber
\frac{1}{nb} \sum_{u'=1}^{\xi(u,\ell)-1} \sum_{i=1}^n \frac{1}{n} \left ( 1 -\frac{1}{n} \right )^{\xi(u,\ell)-u'-1} \mathbb{E}  \left \|  \nabla_{\mathcal{G}_\ell} f_i(w_{{\xi^{-1}(u',\ell)}}) -  \nabla_{\mathcal{G}_\ell} f_i(w^*) \right \|^2
\\ \nonumber &
+  \frac{1}{nb} \sum_{u'=1}^{\xi(u,\ell)-1}   \sum_{i=1}^n   \left ( 1 -\frac{1}{n} \right )^{\xi(u,\ell)-1}  \mathbb{E}  \left \|  \nabla_{\mathcal{G}_\ell} f_i(w_{0}) -  \nabla_{\mathcal{G}_\ell} f_i(w^*) \right \|^2
\\    \stackrel{ (b) }{\leq} & \nonumber
\frac{L^2}{nb} \sum_{u'=1}^{\xi(u,\ell)-1} \left ( 1 -\frac{1}{n} \right )^{\xi(u,\ell)-u'-1} \sigma(w_{{\xi^{-1}(u',\ell)}})
 +\frac{ L^2 }{b}\left ( 1 -\frac{1}{n} \right )^{\xi(u,\ell)} \sigma(w_0)
\end{align}
where the  inequality (a) uses (\ref{lem4_2}) and (\ref{lem4_3}),  the  inequality (b) uses Assumption \ref{assum1}.

Similarly, we have that
\begin{align}\label{lem4_5}
&  \mathbb{E}\left \|  {\alpha}_{\I_u}^{u,\ell} - \widehat{\alpha}_{\I_u}^{u,\ell} \right \|^2
\\ \nonumber   =
& \frac{1}{nb} \sum_{u'=0}^{\xi(u,\ell)-1}  \sum_{i=1}^n \mathbb{E} \mathbf{1}_{ \{ \textbf{u}_{i}^u =u'\} }   \left \| {\alpha}_{i}^{u',\ell} - \widehat{\alpha}_{i}^{u',\ell} \right \|^2
\\    \stackrel{ (a) }{\leq} & \nonumber
\frac{1}{nb} \sum_{u'=1}^{\xi(u,\ell)-1} \sum_{i=1}^n \left ( \frac{1}{n} \left ( 1 -\frac{1}{n} \right )^{\xi(u,\ell)-u'-1}  \mathbb{E}  \left \| {\alpha}_{i}^{u',\ell} - \widehat{\alpha}_{i}^{u',\ell} \right \|^2 +  \frac{1}{n}  \sum_{i=1}^n   \left ( 1 -\frac{1}{n} \right )^{\xi(u,\ell)-1}  \mathbb{E}   \left \| {\alpha}_{i}^{0,\ell} - \widehat{\alpha}_{i}^{0,\ell} \right \|^2 \right )
\\   \stackrel{ (b) }{=} & \nonumber
\frac{1}{nb} \sum_{u'=1}^{\xi(u,\ell)-1} \sum_{i=1}^n \left ( \frac{1}{n} \left ( 1 -\frac{1}{n} \right )^{\xi(u,\ell)-u'-1}  \mathbb{E}  \left \| {\alpha}_{i}^{u',\ell} - \widehat{\alpha}_{i}^{u',\ell} \right \|^2 \right )
\\   = & \nonumber
\frac{1}{nb} \sum_{u'=1}^{\xi(u,\ell)-1} \left ( 1 -\frac{1}{n} \right )^{\xi(u,\ell)-u'-1}
 \mathbb{E}  \left \|  \nabla_{\mathcal{G}_\ell} f_i(\widehat{w}_{\xi^{-1}(u',\ell)}) -  \nabla_{\mathcal{G}_\ell} f_i(w_{\xi^{-1}(u',\ell)}) \right \|^2
\\    \stackrel{ (c) }{\leq} & \nonumber
\frac{L^2}{nb} \sum_{u'=1}^{\xi(u,\ell)-1} \left ( 1 -\frac{1}{n} \right )^{\xi(u,\ell)-u'-1} \mathbb{E}  \left \|  \widehat{w}_{\xi^{-1}(u',\ell)} - w_{\xi^{-1}(u',\ell)} \right \|^2
\\    = & \nonumber
\frac{L^2 \gamma^2}{nb} \sum_{u'=1}^{\xi(u,\ell)-1} \left ( 1 -\frac{1}{n} \right )^{\xi(u,\ell)-u'-1}
  \mathbb{E}  \left \|   \sum_{\widetilde{u} \in D(\xi^{-1}(u',\ell))}    \textbf{U}_{\psi(\widetilde{u})} H_{\psi(u)} \widehat{v}^{\psi(\widetilde{u})}_{\widetilde{u}} \right \|^2
\\    \stackrel{ (d) }{\leq} & \nonumber
 \frac{\tau_1 L^2 \sigma_2^2\gamma^2}{nb} \sum_{u'=1}^{\xi(u,\ell)-1}  \sum_{\widetilde{u} \in D(\xi^{-1}(u',\ell))} \left ( 1 -\frac{1}{n} \right )^{\xi(u,\ell)-u'-1} \cdot
\mathbb{E}  \left \|       \widehat{v}^{\psi(\widetilde{u})}_{\widetilde{u}} \right \|^2
\end{align}
where the inequality (a) can be obtained similar to (\ref{lem4_4}), the equality (b) uses the fact of ${\alpha}_{i}^{0,\ell} = \widehat{\alpha}_{i}^{0,\ell}$, the inequality (c) uses Assumption \ref{assum2}.1, and the  inequality (d) uses Assumption \ref{assum4}.
 This completes the proof.
\end{proof}

\begin{lemma} \label{The3lemma2} Given a global time counter $u$, we let $\{ \overline{u}_0,\overline{u}_1, \ldots,\overline{u}_{\upsilon(u)-1}\}$ be the  all start time counters for the global time counters from 0 to $u$. For AsySQN-SAGA,   we have that
\begin{align}\label{lem4_0}
& \mathbb{E} \left  \| {v}^{\ell}_{u} \right \|^2
\\ \nonumber \leq&
\frac{4 \eta_1L^2}{nb} \sum_{k'=1}^{\upsilon(u)} \left ( 1 -\frac{1}{n} \right )^{\upsilon(u)-k'} \sigma(w_{\overline{u}_{k'}})
+ \frac{8 \sigma_2^2L^2}{b}  \gamma^2 \eta_1^2 G
 + \frac{2 L^2}{b} \left ( 1 -\frac{1}{n} \right )^{\upsilon(u)}  \sigma(w_0)
 + \frac{4 L^2}{b} \sigma(  w_{\varphi(u)})
\end{align}
where $\sigma(w_u) = \mathbb{E} \| w_u - w^* \|^2$.
\end{lemma}
\begin{proof}
We have that
\begin{align}\label{Lemma3_1}
 & \mathbb{E} \left  \| {v}^{\ell}_{u} \right \|^2
\\   = & \nonumber
\mathbb{E} \left \| \nabla_{\mathcal{G}_\ell} f_{\I_u} (w_u) - \alpha_{\I_u}^{u,\ell} +  \frac{1}{n} \sum_{i=1}^n \alpha_{i}^{\ell} \right \|^2
\\   = & \nonumber
\mathbb{E} \left \| \nabla_{\mathcal{G}_\ell} f_{\I_u} (w_u) - \nabla_{\mathcal{G}_\ell} f_{\I_u}(w^*) - \alpha_{\I_u}^{u,\ell} + \nabla_{\mathcal{G}_\ell} f_{\I_u}(w^*) + \frac{1}{n} \sum_{i=1}^n \alpha_{i}^{u,\ell} - \nabla_{\mathcal{G}_\ell} f(w^*) + \nabla_{\mathcal{G}_\ell} f(w^*) \right \|^2
\\   \stackrel{ (a) }{\leq} &  \nonumber
2 \mathbb{E} \left \|  \nabla_{\mathcal{G}_\ell} f_{\I_u}(w^*) - \alpha_{\I_u}^{u,\ell} + \frac{1}{n} \sum_{i=1}^n \alpha_{i}^{u,\ell} - \nabla_{\mathcal{G}_\ell} f(w^*) \right \|^2
+ 2 \mathbb{E} \left \| \nabla_{\mathcal{G}_\ell} f_{\I_u} (w_u) - \nabla_{\mathcal{G}_\ell} f_{\I_u}(w^*) \right \|^2
\\  \stackrel{ (b) }{\leq} &  \nonumber 2
\mathbb{E} \left \| \alpha_{\I_u}^{u,\ell} - \nabla_{\mathcal{G}_\ell} f_{\I_t}(w^*) \right \|^2
+ 2 \mathbb{E} \left \| \nabla_{\mathcal{G}_\ell} f_{\I_u} (w_u) - \nabla_{\mathcal{G}_\ell} f_{\I_u}(w^*) \right \|^2
\\   \stackrel{ (c) }{\leq} &  \nonumber
\frac{2L^2}{nb} \sum_{u'=1}^{\xi(u,\ell)-1} \left ( 1 -\frac{1}{n} \right )^{\xi(u,\ell)-u'-1} \mathbb{E} \| w_{{\xi^{-1}(u',\ell)}} - w^* \|^2
 + \frac{2 L^2}{b} \left ( 1 -\frac{1}{n} \right )^{\xi(u,\ell)}  \sigma(w_0)
+ \frac{2 L^2}{b} \mathbb{E} \| w_u - w^* \|^2
\\  = &  \nonumber
\frac{2L^2}{nb} \sum_{u'=1}^{\xi(u,\ell)-1} \left ( 1 -\frac{1}{n} \right )^{\xi(u,\ell)-u'-1}
 \mathbb{E} \| w_{{\xi^{-1}(u',\ell)}} - w_{\varphi({\xi^{-1}(u',\ell)})}+ w_{\varphi({\xi^{-1}(u',\ell)})} - w^* \|^2
\\  \nonumber &
+ \frac{2 L^2}{b}\left ( 1 -\frac{1}{n} \right )^{\xi(u,\ell)}  \sigma(w_0)
 + \frac{2 L^2}{b} \mathbb{E} \| w_u - w_{\varphi(u)}+ w_{\varphi(u)} - w^* \|^2
\\   \stackrel{ (d) }{\leq} &  \nonumber
\frac{2L^2}{nb} \sum_{u'=1}^{\xi(u,\ell)-1} \left ( 1 -\frac{1}{n} \right )^{\xi(u,\ell)-u'-1}
 \mathbb{E} \left ( 2\| w_{{\xi^{-1}(u',\ell)}} - w_{\varphi({\xi^{-1}(u',\ell)})} \|^2
 + 2\|  w_{\varphi({\xi^{-1}(u',\ell)})} - w^* \|^2 \right )
\\  \nonumber  &
+ \frac{2 L^2}{b} \left ( 1 -\frac{1}{n} \right )^{\upsilon(u)}  \sigma(w_0)
+ \frac{4 L^2}{b} \mathbb{E} \|  w_{\varphi(u)} - w^* \|^2
 + \frac{4 \sigma_2^2L^2}{b} \gamma^2 \mathbb{E} \left  \|  \sum_{v \in \{{\varphi(u)},\ldots,u \}} \textbf{U}_{\psi(v)} \widehat{v}^{\psi(v)}_v  \right \|^2
\\   \stackrel{ (e) }{\leq} &  \nonumber
\frac{2L^2}{nb} \sum_{u'=1}^{\xi(u,\ell)-1} \left ( 1 -\frac{1}{n} \right )^{\xi(u,\ell)-u'-1} \mathbb{E} \left ( 2 \eta_1 \sigma_2^2 \gamma^2   \sum_{v \in \{{\varphi({\xi^{-1}(u',\ell)})},\ldots,{{\xi^{-1}(u',\ell)}} \}} \| \widehat{v}^{\psi(v)}_v   \|^2  + 2\|  w_{\varphi({\xi^{-1}(u',\ell)})} - w^* \|^2 \right )
\\  \nonumber &
+ \frac{2 L^2}{b}\left ( 1 -\frac{1}{n} \right )^{\upsilon(u)}  \sigma(w_0)
 +  \frac{4 L^2}{b}  \mathbb{E} \|  w_{\varphi(u)} - w^* \|^2
 +  \frac{4 \sigma_2^2L^2}{b} \gamma^2 \eta_1 \sum_{v \in \{{\varphi(u)},\ldots,u \}} \mathbb{E} \left  \|   \widehat{v}^{\psi(v)}_v  \right \|^2
\\    \stackrel{ (f) }{\leq} &  \nonumber
\frac{2 L^2}{nb} \sum_{u'=1}^{\xi(u,\ell)-1} \left ( 1 -\frac{1}{n} \right )^{\xi(u,\ell)-u'-1}
\mathbb{E} \left ( 2 \eta_1^2 \sigma_2^2\gamma^2 G + 2\|  w_{\varphi({\xi^{-1}(u',\ell)})} - w^* \|^2 \right )
\\  \nonumber  &
+ \frac{2 L^2}{b}\left ( 1 -\frac{1}{n} \right )^{\upsilon(u)}  \sigma(w_0)
  + \frac{4 L^2}{b} \mathbb{E} \|  w_{\varphi(u)} - w^* \|^2
  + \frac{4 \sigma_2^2L^2}{b} \gamma^2 \eta_1^2 G
\\    \stackrel{ (g) }{\leq} &  \nonumber
\frac{4 L^2}{nb} \sum_{u'=1}^{\xi(u,\ell)-1} \left ( 1 -\frac{1}{n} \right )^{\xi(u,\ell)-u'-1} \mathbb{E}  \|  w_{\varphi({\xi^{-1}(u',\ell)})} - w^* \|^2
\\  \nonumber &
+ \frac{2 L^2}{b} \left ( 1 -\frac{1}{n} \right )^{\upsilon(u)}  \sigma(w_0)
 + \frac{4 L^2}{b} \mathbb{E} \|  w_{\varphi(u)} - w^* \|^2
  + \frac{8 \sigma_2^2 L^2}{b} \gamma^2 \eta_1^2 G
\\    \stackrel{ (h) }{\leq} &  \nonumber
\frac{4 \eta_1L^2}{nb} \sum_{k'=1}^{\upsilon(u)} \left ( 1 -\frac{1}{n} \right )^{\upsilon(u)-k'} \sigma(w_{\overline{u}_{k'}})
+ \frac{8 \sigma_2^2L^2}{b}  \gamma^2 \eta_1^2 G
 + \frac{2 L^2}{b} \left ( 1 -\frac{1}{n} \right )^{\upsilon(u)}  \sigma(w_0)
 + \frac{4 L^2}{b} \sigma(  w_{\varphi(u)})
\end{align}
where the  inequality (a) uses $\| \sum_{i=1}^n a_i \|^2 \leq n \sum_{i=1}^n \| a_i \|^2 $, the  inequality (b) follows from $\mathbb{E} \left \| x- \mathbb{E} x\right \|^2 \leq \mathbb{E} \left \| x\right \|^2$,  the  inequality (c) uses Lemma \ref{AsySGHT_lemma2},  (d) uses $\| \sum_{i=1}^n a_i \|^2 \leq n \sum_{i=1}^n \| a_i \|^2 $ and Assumption  \ref{assum3}.2, the inequality (e) uses Assumptions \ref{assum3}.2 and \ref{assum5}, the inequality (f) uses Assumption \ref{assum2}.3, and the  inequality (g) uses the fact $ \sum_{u'=1}^{\xi(u,\ell)-1} \left ( 1 -\frac{1}{n} \right )^{\xi(u,\ell)-u'-1} \leq n$.
%
%
This
completes the proof.
\end{proof}

\begin{lemma} \label{lemma3} For  AsySQN-SAGA, under Assumptions \ref{assum1} to \ref{assum5},  let $u \in K(t)$,  we have that
\begin{align}
\label{lemma3_0_1}  \mathbb{E} \left  \| \widehat{v}^{\ell}_u \right \|^2
\leq &  6 \eta_1 L^2 \gamma^2 \sum_{u' \in D(u)} \mathbb{E}  \left \|       \widehat{v}^{\psi(u')}_{u'} \right \|^2 + \frac{12 \tau_1 L^2 \gamma^2}{l} \sum_{u'=1}^{\xi(u,\ell)-1}  \cdot
\\ & \nonumber \sum_{\widetilde{u} \in D(\xi^{-1}(u',\ell))} \left ( 1 -\frac{1}{n} \right )^{\xi(u,\ell)-u'-1} \mathbb{E}  \left \|       \widehat{v}^{\psi(\widetilde{u})}_{\widetilde{u}} \right \|^2
 +  2 \mathbb{E} \left \|  v_{u}^\ell  \right \|^2
\end{align}
\end{lemma}
\begin{proof} Define ${v}^{\ell}_u= \nabla_{\mathcal{G}_\ell} f_{\I} (w_u) - \alpha_{\I}^{\ell} +  \frac{1}{n} \sum_{i=1}^n \alpha_i^{\ell}$. We have that
$
\label{AsySCGD+_lem3_0_2}  \mathbb{E}  \left \|   \widehat{v}^{ \ell }_u \right \|^2 =\mathbb{E}  \left \|   \widehat{v}^{ \ell }_u - {v}^{\ell}_u + {v}^{\ell}_u \right \|^2 \leq 2 \mathbb{E}   \left \|   \widehat{v}^{ \ell }_u - {v}^{\ell}_u \right  \|^2 +   2 \mathbb{E} \left  \| {v}^{\ell}_u \right \|^2
$.

Next, we give the upper bound to $\mathbb{E}   \left \|   \widehat{v}^{ \ell }_u - {v}^{\ell}_u \right  \|^2$ as follows.
Next, we have that
\begin{align}\label{AsySGHT_lem4_11}
&   \mathbb{E} \left \| \widehat{v}_{u}^\ell -  v_{u}^\ell  \right \|^2
\\    = & \nonumber
\mathbb{E} \left \| \nabla_{\mathcal{G}_\ell}  f_{\I_u}(\widehat{w}_u)- \widehat{\alpha}_{\I_u}^{u,\ell}  + \frac{1}{n} \sum_{i=1}^n \widehat{\alpha}_{i}^{u,\ell}  - \nabla_{\mathcal{G}_\ell} f_{\I_u}(w_t)   + \alpha_{\I_u}^{u,\ell} - \frac{1}{n} \sum_{i=1}^n \alpha_{i}^{u,\ell} \right \|^2
\\    \stackrel{ (a) }{\leq} & \nonumber
3  \mathbb{E} \underbrace{\left \| \nabla_{\mathcal{G}_\ell}  f_{\I_u}(\widehat{w}_u) -\nabla_{\mathcal{G}_\ell}  f_{\I_u}(w_u) \right \|^2 }_{Q_1} + 3\mathbb{E} \underbrace{\left \| {\alpha}_{\I_u}^{u,\ell} - \widehat{\alpha}_{\I_u}^{u,\ell} \right \|^2 }_{Q_2}
 + 3\mathbb{E} \underbrace{\left \|  \frac{1}{n} \sum_{i=1}^n \alpha_{i}^{u,\ell} - \frac{1}{n} \sum_{i=1}^n \widehat{\alpha}_{i}^{u,\ell}  \right \|^2}_{Q_3}
\end{align}
where the  inequality (a) uses $\| \sum_{i=1}^n a_i \|^2 \leq n \sum_{i=1}^n \| a_i \|^2 $.
We will give the upper bounds for the expectations  of $Q_1$, $Q_2$ and $Q_3$  respectively.
\begin{align}\label{AsySGHT_lem4_2}
& \nonumber \mathbb{E} Q_1 = \mathbb{E} \left \| \nabla_{\mathcal{G}_\ell}  f_{\I_u}(\widehat{w}_u) -\nabla_{\mathcal{G}_\ell}  f_{\I_u}(w_u) \right \|^2
\\ \nonumber =&
\mathbb{E} \left \|\frac{1}{|\I|}\sum_{i\in\I}(\nabla_{\mathcal{G}_\ell}  f_{i_u}(\widehat{w}_u) -\nabla_{\mathcal{G}_\ell}  f_{i_u}(w_u) )\right \|^2
\\ \nonumber =&
\frac{1}{|\I|^2}\sum_{i\in\I}\mathbb{E} \left \|\nabla_{\mathcal{G}_\ell}  f_{i_u}(\widehat{w}_u) -\nabla_{\mathcal{G}_\ell}  f_{i_u}(w_u) \right \|^2
\\   \nonumber \leq &
 \frac{L^2}{b} \mathbb{E} \left \| \widehat{w}_u - {w}_u \right \|^2 = \frac{L^2}{b} \gamma^2 \mathbb{E} \left \|  \sum_{u' \in D(u)}    \textbf{U}_{\psi(u')} H_{\psi(u')}\widehat{v}^{\psi(u')}_{u'} \right \|^2
\\    \leq & \frac{\tau_1 L^2 \sigma_2^2 \gamma^2}{b} \sum_{u' \in D(u)} \mathbb{E}  \left \|       \widehat{v}^{\psi(u')}_{u'} \right \|^2
\end{align}
where the first inequality uses Assumption \ref{assum2}.1, the second inequality uses $\| \sum_{i=1}^n a_i \|^2 \leq n \sum_{i=1}^n \| a_i \|^2 $ and Assumption \ref{assum3}.2.
\begin{align}\label{AsySGHT_lem4_3}
& \mathbb{E} Q_2 =   \mathbb{E}\left \|  {\alpha}_{i_u}^{u,\ell} - \widehat{\alpha}_{i_u}^{u,\ell} \right \|^2
\\ \nonumber   \leq &
\frac{\tau_1 L^2 \sigma_2^2 \gamma^2}{nb} \sum_{u'=1}^{\xi(u,\ell)-1}  \sum_{\widetilde{u} \in D(\xi^{-1}(u',\ell))}
 \left ( 1 -\frac{1}{n} \right )^{\xi(u,\ell)-u'-1} \mathbb{E}  \left \|       \widehat{v}^{\psi(\widetilde{u})}_{\widetilde{u}} \right \|^2
\end{align}
where the inequality uses Lemma \ref{AsySGHT_lemma2}.
\begin{align}\label{AsySGHT_lem4_4}
& \mathbb{E} Q_3 =    \mathbb{E} \left \|  \frac{1}{n} \sum_{i=1}^n \alpha_{i}^{u,\ell} - \frac{1}{n} \sum_{i=1}^n \widehat{\alpha}_{i}^{u,\ell}   \right \|^2
\\    \leq & \nonumber  \frac{1}{n} \sum_{i=1}^n \mathbb{E} \left \|  \alpha_{i}^{u,\ell}  - \widehat{\alpha}_{i}^{u,\ell}   \right \|^2
\\    \leq & \nonumber
\frac{\tau_1 L^2 \sigma_2^2 \gamma^2}{nb} \sum_{u'=1}^{\xi(u,\ell)-1}  \sum_{\widetilde{u} \in D(\xi^{-1}(u',\ell))}
 \left ( 1 -\frac{1}{n} \right )^{\xi(u,\ell)-u'-1} \mathbb{E}  \left \|       \widehat{v}^{\psi(\widetilde{u})}_{\widetilde{u}} \right \|^2
\end{align}
where  the first inequality uses $\| \sum_{i=1}^n a_i \|^2 \leq n \sum_{i=1}^n \| a_i \|^2 $, the second inequality uses (\ref{AsySGHT_lem4_3}).

\begin{align}\label{AsySGHT_lem4_1}
&  \mathbb{E} \left \| \widehat{v}_{u}^\ell  \right \|^2 \leq  2 \mathbb{E} \left \| \widehat{v}_{u}^\ell  - v^\ell_{u} \right \|^2 +  2 \mathbb{E} \left \|  v_{u}^\ell  \right \|^2
\\    \leq & \nonumber    6  \mathbb{E} {Q_1} + 6\mathbb{E} {Q_2}  + 6\mathbb{E} {Q_3} + 2 \mathbb{E} \left \|  v_{u}^\ell  \right \|^2
\\    \leq & \nonumber \frac{6 \tau_1 L^2 \sigma_2^2 \gamma^2}{b}  \sum_{u' \in D(u)} \mathbb{E}  \left \|       \widehat{v}^{\psi(u')}_{u'} \right \|^2  +  2 \mathbb{E} \left \|  v_{u}^\ell  \right \|^2
\\    & \nonumber + \frac{12 \tau_1 L^2 \sigma_2^2 \gamma^2}{nb} \sum_{u'=1}^{\xi(u,\ell)-1}  \sum_{\widetilde{u} \in D(\xi^{-1}(u',\ell))}
  \left ( 1 -\frac{1}{n} \right )^{\xi(u,\ell)-u'-1} \mathbb{E}  \left \|       \widehat{v}^{\psi(\widetilde{u})}_{\widetilde{u}} \right \|^2
\end{align}
where the second inequality uses Lemma \ref{AsySGHT_lemma2}.
This completes the proof.
\end{proof}

Based on the basic inequalities in Lemma \ref{AsySPSAGA_lemma3}, we provide the proof of Theorem \ref{thm-sagaconvex} in the following.

\begin{proof}
Similar to (\ref{EqThm2_2}),  we have that
\begin{align}\label{EqThm4_1}
 & \mathbb{E} f (w_{t+|K(t)|}) - \mathbb{E}f (w_{t})
 \\ \nonumber \stackrel{ (a) }{\leq} &
 -  \frac{\gamma\sigma_1}{4} \| \nabla f(w_t^s) \|^2 +  \frac{\tau_1 \sigma_2^3 L^2 \gamma^3}{2}  \sum_{u \in K(t)}\sum_{u' \in D(u)} \mathbb{E} \|   \widehat{v}^{\psi(u')}_{u'} \|^2
\\ \nonumber
&   + \left ( \frac{\eta_1\sigma_1\sigma_2^2 \gamma^3 L^2 \eta_1 }{2}+ \frac{L_{\max} \sigma_2^2\gamma^2}{2} \right ) \sum_{u \in K(t)}\mathbb{E} \|  \widehat{v}^{\psi(u)}_u  \|^2
 \\ \nonumber  \stackrel{ (b) }{\leq} &
 -  \frac{\gamma\sigma_1}{4} \| \nabla f(w_t) \|^2 +  \frac{\tau_1 L^2 \sigma_2^3 \gamma^3}{2}  \sum_{u \in K(t)}\sum_{u' \in D(u)} \mathbb{E} \|   \widehat{v}^{\psi(u')}_{u'} \|^2
\\ \nonumber &
+ \left ( \frac{\sigma_1\sigma_2^2\eta_1^2 \gamma^3 L^2  }{2}+ \frac{L_{\max} \sigma_2^2\gamma^2}{2} \right ) \sum_{u \in K(t)}
\left ( \frac{6 \tau_1 \sigma_2^2L^2}{b} \gamma^2 \sum_{u' \in D(u)}  \mathbb{E}  \left \|       \widehat{v}^{\psi(u')}_{u'} \right \|^2  \right .
\\ \nonumber &
 \left . + \frac{12 \tau_1 L^2 \sigma_2^2\gamma^2}{nb} \sum_{u'=1}^{\xi(u,\ell)-1}  \sum_{\widetilde{u} \in D(\xi^{-1}(u',\ell))} \left ( 1 -\frac{1}{n} \right )^{\xi(u,\ell)-u'-1} \mathbb{E}  \left \|       \widehat{v}^{\psi(\widetilde{u})}_{\widetilde{u}} \right \|^2
  +  2 \mathbb{E} \left \|  v_{u}^{\psi({u})}  \right \|^2 \right )
\\ \nonumber  \stackrel{ (c) }{\leq} &
 -  \frac{\gamma\sigma_1}{4} \| \nabla f(w_t) \|^2 +  \frac{\tau_1 L^2 \sigma_2^3\gamma^3}{2b}  \eta_1\tau_1G
 + \left ( \gamma L^2 \sigma_1\eta_1^2 + L_{\max} \right )3 \eta_1 L^2 \sigma_2^4\gamma^4 \tau_1^2 \frac{G}{b}
\\ \nonumber &  + \left ( \gamma L^2 \sigma_1\eta_1^2 + L_{\max} \right ) 6  \eta_1  \sigma_2^4L^2 \tau_1^2 \frac{G}{b} \gamma^4
+ \left ( \gamma L^2  \sigma_1 \eta_1^2 + L_{\max} \right )  \sigma_2^2\gamma^2 \sum_{u \in K(t)}     \mathbb{E} \left \|  v_{u}^{\psi({u})}  \right \|^2
 \\ \nonumber  = &
  -  \frac{\gamma \sigma_1}{4} \| \nabla f(w_t) \|^2  + \left ( \gamma L^2 \sigma_1\eta_1^2 + L_{\max} \right ) \sigma_2^2\gamma^2 \sum_{u \in K(t)}     \mathbb{E} \left \|  v_{u}^{\psi({u})}  \right \|^2
\\ \nonumber &
 +  \left ( \frac{1 }{2}   + 9\gamma\sigma_2\left ( \gamma L^2 \sigma_1\eta_1^2 + L_{\max} \right )  \right ) \sigma_2^3\gamma^3 L^2 \eta_1 \tau_1^2 \frac{G}{b}
 \\ \nonumber  \stackrel{ (d) }{\leq} &
 -   -  \frac{\gamma \sigma_1}{4} \| \nabla f(w_t) \|^2  + \left ( \gamma L^2 \sigma_1\eta_1^2 + L_{\max} \right ) \sigma_2^2\gamma^2\sum_{u \in K(t)}
 \left (    \frac{4 \eta_1L^2}{nb} \sum_{k'=1}^{\upsilon(u)} \left ( 1 -\frac{1}{n} \right )^{\upsilon(u)-k'} \sigma(w_{\overline{u}_{k'}})\right .
\\   \nonumber & \left . + \frac{2 L^2}{b} \left ( 1 -\frac{1}{n} \right )^{\upsilon(u)}  \sigma(w_0)
 + \frac{4 L^2}{b} \sigma(  w_{\varphi(u)})  + \frac{8 L^2 \sigma_2^2}{b}  \gamma^2 \eta_1^2 G \right )
\\ \nonumber &
 +  \left ( \frac{1 }{2}   + 9\gamma\sigma_2\left ( \gamma L^2 \sigma_1\eta_1^2 + L_{\max} \right )  \right ) \sigma_2^3\gamma^3 L^2 \eta_1 \tau_1^2 \frac{G}{b}
 \\ \nonumber \stackrel{ (e) }{\leq} &
 -  \frac{\gamma \sigma_1\mu}{4} e(w_t)-  \frac{\gamma \sigma_1\mu^2}{4} \sigma(w_t)
\\ \nonumber &
+ \left ( \gamma L^2 \sigma_1\eta_1^2 + L_{\max} \right ) \sigma_2^2\gamma^2\eta_1
\left (\frac{ 2 L^2}{b}  \left ( 1 -\frac{1}{n} \right )^{\upsilon(t)}  \sigma (w_0)
 + \frac{4 L^2 }{b}\sigma(w_{t}) \right )
\\ \nonumber &
\left ( \frac{\tau_1^2 }{2}   + (9\tau_1^2+8\eta_1^2)\sigma_2\gamma\left ( \gamma L^2 \sigma_1\eta_1^2 + L_{\max} \right )  \right ) \sigma_2^3\gamma^3 L^2 \eta_1\frac{G}{b}
\\ \nonumber &
 + 4\left ( \gamma L^2 \sigma_1\eta_1^2 + L_{\max} \right )    \frac{L^2 \eta_1^2\sigma_2^2\gamma^2}{nb }
 \sum_{k'=1}^{\upsilon(t)} \left ( 1 -\frac{1}{n} \right )^{\upsilon(t)-k'} \sigma(w_{\overline{u}_{k'}})
  \end{align}
where the  inequalities (a)  use (\ref{EqThm2_2}), the equality (b) uses Lemma \ref{lemma3}, the inequality (c) uses Assumptions \ref{assum2}.3 and \ref{assum5}, the inequality (d) uses Lemma \ref{The3lemma2}, the inequality (e) uses Assumption  \ref{assum5}.

According to (\ref{EqThm4_1}), we have that
\begin{align}\label{EqThm4_2}
  e (w_{t+|K(t)|})
 \leq & \left ( 1 -  \frac{\gamma\sigma_1 \mu}{4} \right ) e(w_t) + c_1  \left (   \left ( 1 -\frac{1}{n} \right )^{\upsilon(t)}  \sigma(w_0)
 + 2 \sigma(w_{t}) \right )
\\ \nonumber & +  c_0
 + c_2 \sum_{k'=1}^{\upsilon(t)} \left ( 1 -\frac{1}{n} \right )^{\upsilon(t)-k'} \sigma(w_{\overline{u}_{k'}})  -  \frac{\gamma\sigma_1 \mu^2}{4} \sigma(w_t)
   \\ \nonumber  = & \left ( 1 -  \frac{\gamma \mu\sigma_1}{4} \right ) e(w_t)  +\left (  -  \frac{\gamma \sigma_1 \mu^2}{4} +2 c_1+c_2 \right )  \sigma(w_t)
\\ \nonumber &    + c_1    \left ( 1 -\frac{1}{n} \right )^{\upsilon(t)}  \sigma(w_0)
 + c_2  \sum_{k'=1}^{\upsilon(t)-1} \left ( 1 -\frac{1}{n} \right )^{\upsilon(t)-k'} \sigma(w_{\overline{u}_{k'}}) +  c_0
  \end{align}
where ${c_0} =  \left ( \frac{\tau_1^2 }{2}   + (9\tau_1^2+8\eta_1^2)\sigma_2\gamma\left ( \gamma L^2 \sigma_1\eta_1^2 + L_{\max} \right )  \right ) \sigma_2^3\gamma^3 L^2 \eta_1\frac{G}{b}$,  $c_1 = \left ( \gamma L^2 \sigma_1\eta_1^2 + L_{\max} \right ) \sigma_2^2\gamma^2\eta_1\frac{2L^2}{b}$,  ${c_2}  = 4\left ( \gamma L^2 \sigma_1\eta_1^2 + L_{\max} \right )    \frac{L^2 \eta_1^2\sigma_2^2\gamma^2}{nb } $, $\{ \overline{u}_0,\overline{u}_1, \ldots,\overline{u}_{\upsilon(u)-1}\}$ are the  all start time counters for the global time counters from 0 to $u$.

We define the Lyapunov function as $\mathcal{L}_t=\sum_{k=0}^{\upsilon(t)} \rho^{\upsilon(t)-k} e(w_{\overline{u}_{k}})$ where $\rho \in (1 -\frac{1}{n},1)$, we have that
 \begin{align}\label{EqThm4_3}
 & \mathcal{L}_{t+|K(t)|}
\\   = & \nonumber
\rho^{\upsilon(t)+1} e(w_0) +  \sum_{k=0}^{\upsilon(t)} \rho^{\upsilon(t)-k} e(w_{\overline{u}_{k+1}})
\\   \stackrel{ (a) }{\leq} & \nonumber
\rho^{\upsilon(t)+1} e(w_0) +   \sum_{k=0}^{\upsilon(t)} \rho^{\upsilon(t)-k} \left [  \left ( 1 -  \frac{\gamma \sigma_1\mu}{4} \right ) e(w_{\overline{u}_{k}})
 +\left (  -  \frac{\gamma \sigma_1\mu^2}{4} +2 c_1+c_2 \right )  \sigma(w_{\overline{u}_{k}})
 \right .
\\ \nonumber &
\left .
+ c_1    \left ( 1 -\frac{1}{n} \right )^{k}  \sigma(w_0)
 + c_2  \sum_{k'=1}^{k-1} \left ( 1 -\frac{1}{n} \right )^{k-k'} \sigma(w_{\overline{u}_{k'}}) +  c_0 \right ]
\\ \label{EqThm4_4}   = &  \rho^{\upsilon(t)+1} e(w_0) +  \left ( 1 -  \frac{\gamma \sigma_1\mu}{4} \right )\mathcal{L}_{t} +  \sum_{k=0}^{\upsilon(t)} \rho^{\upsilon(t)-k}
\left [  \left (  -  \frac{\gamma\sigma_1 \mu^2}{4} +2 c_1+c_2 \right )  \sigma(w_{\overline{u}_{k}})
\right .
\\ \nonumber & \left .
+ c_1    \left ( 1 -\frac{1}{n} \right )^{k}  \sigma(w_0)
+ c_2  \sum_{k'=1}^{k-1} \left ( 1 -\frac{1}{n} \right )^{k-k'} \sigma(w_{\overline{u}_{k'}})    \right ]+\sum_{k=0}^{\upsilon(t)} \rho^{\upsilon(t)-k}  c_0
\\   \stackrel{ (b) }{\leq} & \nonumber
\rho^{\upsilon(t)+1} e(w_0) +  \left ( 1 -  \frac{\gamma \sigma_1\mu}{4} \right )\mathcal{L}_{t}
+    \left (  -  \frac{\gamma \sigma_1\mu^2}{4} +2 c_1+c_2 \right )  \sigma(w_{\overline{u}_{\upsilon(t)}}) +  \frac{c_0}{1-\rho}
\\   \stackrel{ (c) }{\leq} & \nonumber
\rho^{\upsilon(t)+1} e(w_0) +  \left ( 1 -  \frac{\gamma \sigma_1\mu}{4} \right )\mathcal{L}_{t}
 -    \left ( \frac{\gamma\sigma_1 \mu^2}{4} -2 c_1-c_2 \right )\frac{2}{L}  e(w_{\overline{u}_{\upsilon(t)}}) +  \frac{c_0}{1-\rho}
 \end{align}
where the  inequality (a) uses (\ref{EqThm4_2}), the  inequality (b) holds by appropriately choosing $\gamma$ such that the terms related to  $\sigma(w_{\overline{u}_{k}})$ ($k=0,\cdots,\upsilon(t)-1$) are negative,  because the signs related to  the lowest  orders of $\sigma(w_{\overline{u}_{k}})$ ($k=0,\cdots,\upsilon(t)-1$)  are negative. In the following, we give the detailed analysis of choosing $\gamma$ such that the terms related to $\sigma(w_{\overline{u}_{k}})$ ($k=0,\cdots,\upsilon(t)-1$)  are negative. We first consider $k=0$. Assume that $ C(\sigma(w_{0}))$ is the coefficient term of $\sigma(w_{0}))$ in  (\ref{EqThm4_4}), we have that
 \begin{align}\label{AsySGHT_theorem2_2.1.1}
& C(\sigma(w_{0}))
\\   = & \nonumber   \rho^{\upsilon(t)}\left ( -  \frac{\gamma \sigma_1\mu^2}{4} +2 c_1+c_2 \right )+  c_1   \sum_{k=0}^{\upsilon(t)} \rho^{\upsilon(t)-k}  \left ( 1 -\frac{1}{n} \right )^{k}
\\   = & \nonumber   \rho^{\upsilon(t)}\left ( -  \frac{\gamma \sigma_1\mu^2}{4} +2 c_1+c_2+  c_1   \sum_{k=0}^{\upsilon(t)} \left (\frac{ 1 -\frac{1}{n}}{\rho} \right )^{k}  \right )
\\  \leq & \nonumber  \rho^{\upsilon(t)}\left ( -  \frac{\gamma \sigma_1 \mu^2}{4} +2 c_1+c_2+  c_1   \frac{1}{1- \frac{ 1 -\frac{1}{n}}{\rho}} \right )
\\   = & \nonumber  \rho^{\upsilon(t)}\left ( -  \frac{\gamma \sigma_1 \mu^2}{4} +c_2+  c_1 \left ( 2+ \frac{1}{1- \frac{ 1 -\frac{1}{n}}{\rho}} \right ) \right )
 \end{align}
Based on (\ref{AsySGHT_theorem2_2.1.1}), we can carefully choose $\gamma$ such that $ -  \frac{\gamma \sigma_1 \mu^2}{4} +c_2+  c_1 \left ( 2+ \frac{1}{1- \frac{ 1 -\frac{1}{n}}{\rho}} \right )    \leq 0$.

 Assume that $C(\sigma(w_{\overline{u}_{k}}))$ is the coefficient term of $\sigma(w_{\overline{u}_{k}})$ ($k=1,\cdots,\upsilon(t)-1$) in the big square brackets of (\ref{EqThm4_4}), we have that
  \begin{align}\label{AsySGHT_theorem2_2.1.2}
 & C(\sigma(w_{\overline{u}_{k}}))
\\   = & \nonumber \rho^{\upsilon(t)-k}  \left (  -  \frac{\gamma \sigma_1 \mu^2}{4} +2 c_1+c_2 \right ) + c_2  \sum_{v=k+1}^{\upsilon(t)-1} \left ( 1 -\frac{1}{n} \right )^{v-k} \rho^{\upsilon(t)-v}
\\   = & \nonumber \rho^{\upsilon(t)-k} \left ( -  \frac{\gamma \sigma_1 \mu^2}{4} +2 c_1+c_2  + c_2  \sum_{v=k+1}^{\upsilon(t)-1} \left ( 1 -\frac{1}{n} \right )^{v-k} \rho^{k-v} \right )
\\   = & \nonumber \rho^{\upsilon(t)-k} \left ( -  \frac{\gamma \sigma_1 \mu^2}{4} +2 c_1+c_2  + c_2  \sum_{v=k+1}^{\upsilon(t)-1} \left ( \frac{1 -\frac{1}{n}}{\rho}\right )^{v-k}  \right )
\\   \leq & \nonumber \rho^{\upsilon(t)-k} \left ( -  \frac{\gamma \sigma_1 \mu^2}{4} +2 c_1+c_2 \left ( 1+   \frac{1}{1- \frac{ 1 -\frac{1}{n}}{\rho}}  \right ) \right )
 \end{align}
Based on (\ref{AsySGHT_theorem2_2.1.2}), we can carefully choose $\gamma$ such that $-  \frac{\gamma \sigma_1 \mu^2}{4} +2 c_1+c_2 \left ( 1+   \frac{1}{1- \frac{ 1 -\frac{1}{n}}{\rho}}  \right ) \leq 0$.

 Thus, based on (\ref{EqThm4_3}), we have that
\begin{align}\label{EqThm4_7}
& \left ( \frac{\gamma \sigma_1 \mu^2}{4} -2 c_1-c_2 \right )\frac{2}{L}  e(w_{\overline{u}_{k}})
 \\   \leq & \nonumber  \left ( \frac{\gamma \sigma_1 \mu^2}{4} -2 c_1-c_2 \right )\frac{2}{L}  e(w_{\overline{u}_{k}}) + \mathcal{L}_{t+|K(t)|}
 \\   \stackrel{ (a) }{\leq} & \nonumber \rho^{\upsilon(t)+1} e(w_0) +  \left ( 1 -  \frac{\gamma \sigma_1 \mu}{4} \right )\mathcal{L}_{t}  +  \frac{c_0}{1-\rho}
\\   \stackrel{ (b) }{\leq} & \nonumber \left ( 1 -  \frac{\gamma \sigma_1 \mu}{4} \right )^{\upsilon(t)+1} \mathcal{L}_{0} +  \rho^{\upsilon(t)+1}e(w_0)  \sum_{k=0}^{\upsilon(t)+1} \left (\frac{1 -  \frac{\gamma \sigma_1 \mu}{4}}{\rho} \right )^k
\\ \nonumber & + \frac{c_0}{1-\rho}\sum_{k=0}^{\upsilon(t)} \left ( 1 -  \frac{\gamma \sigma_1 \mu}{4} \right )^{k}
\\   \leq & \nonumber \left ( 1 -  \frac{\gamma  \sigma_1 \mu}{4} \right )^{\upsilon(t)+1} e(w_0)   +  \rho^{\upsilon(t)+1}e(w_0)  \frac{1}{1-\frac{1 -  \frac{\gamma \sigma_1 \mu}{4}}{\rho} } + \frac{c_0}{1-\rho} \frac{4}{\gamma \sigma_1 \mu }
\\   \stackrel{ (c) }{\leq} & \nonumber  \frac{{2\rho- 1 +  \frac{\gamma \sigma_1 \mu}{4} }}{\rho- 1 +  \frac{\gamma \sigma_1 \mu}{4} } \rho^{\upsilon(t)+1}e(w_0)    + \frac{c_0}{1-\rho} \frac{4}{\gamma \sigma_1 \mu }
\end{align}
where the inequality (a) follows from  (\ref{EqThm4_3}), the inequality (b) holds by using the inequality (\ref{EqThm4_3}) recursively, the inequality (c) uses the fact that $ 1 -  \frac{\gamma \sigma_1 \mu}{4} < \rho$.

According to (\ref{EqThm4_7}), we have that
\begin{align}\label{EqThm4_8}
 \nonumber  e(w_{\overline{u}_{k}}) \leq & \frac{{2\rho- 1 +  \frac{\gamma \sigma_1 \mu}{4} }}{(\rho- 1 +  \frac{\gamma \sigma_1 \mu}{4} )\left ( \frac{\gamma \sigma_1 \mu^2}{4} -2 c_1-c_2 \right )} \rho^{\upsilon(t)+1}e(w_0)
\\  & + \frac{4 c_0}{ \gamma \sigma_1 \mu (1-\rho) \left ( \frac{\gamma \sigma_1 \mu^2}{4} -2 c_1-c_2 \right )}
\end{align}

Thus, to achieve the accuracy $\epsilon$ of (\ref{formulation1}) for AFSAGA-VP, \emph{i.e.}, $\mathbb{E} f (w_{\overline{u}_{k}}) -f(w^*) \leq \epsilon$,  we can carefully choose $\gamma$ such that
\begin{eqnarray}
\frac{4 c_0}{ \gamma \sigma_1 \mu (1-\rho) \left ( \frac{\gamma \sigma_1 \mu^2}{4} -2 c_1-c_2 \right )} &\leq& \frac{\epsilon}{2}
\\
0<1 -  \frac{\gamma \sigma_1 \mu}{4} & <&1
\\ -  \frac{\gamma \sigma_1 \mu^2}{4} +2 c_1+c_2 \left ( 1+   \frac{1}{1- \frac{ 1 -\frac{1}{n}}{\rho}}  \right ) &\leq& 0
\\  -  \frac{\gamma \sigma_1 \mu^2}{4} +c_2+  c_1 \left ( 2+ \frac{1}{1- \frac{ 1 -\frac{1}{n}}{\rho}} \right )   &\leq& 0
\end{eqnarray}
and let $\frac{{2\rho- 1 +  \frac{\gamma \sigma_1 \mu}{4} }}{(\rho- 1 +  \frac{\gamma \sigma_1 \mu}{4} )\left ( \frac{\gamma \sigma_1 \mu^2}{4} -2 c_1-c_2 \right )} \rho^{\upsilon(t)+1}e(w_0)  \leq \frac{\epsilon}{2}$, we have that
\begin{eqnarray}\label{EqThm2_5}
\upsilon(t)  \geq \frac{\log \frac{2 \left ({{2\rho- 1 +  \frac{\gamma \sigma_1 \mu}{4} }} \right ) e(w_0) }{\epsilon {(\rho- 1 +  \frac{\gamma \sigma_1 \mu}{4} )\left ( \frac{\gamma \sigma_1 \mu^2}{4} -2 c_1-c_2 \right )} }}{\log \frac{1}{\rho}}
\end{eqnarray}
This completes the proof.
\end{proof}

\end{document}